\documentclass[journal,onecolumn]{IEEEtran}

\usepackage[utf8]{inputenc} % allow utf-8 input
\usepackage[T1]{fontenc}    % use 8-bit T1 fonts
\usepackage{hyperref}       % hyperlinks
\usepackage{url}            % simple URL typesetting
\usepackage{booktabs}       % professional-quality tables
\usepackage{amsfonts}       % blackboard math symbols
\usepackage{nicefrac}       % compact symbols for 1/2, etc.
\usepackage{microtype}      % microtypography
\usepackage{xcolor}         % colors
\usepackage{tikz}
\usepackage{amsmath}
\usepackage{subcaption}
\usepackage{cleveref}
\usepackage{amsthm}
\usepackage{algpseudocode}
\usepackage{algorithm}
\usepackage{multirow}

\usepackage{enumitem}
\setitemize{noitemsep,topsep=0pt,parsep=0pt,partopsep=0pt}
\graphicspath{{./images/}}
\usepackage{subfiles}
\usepackage{tikz}
\usepackage{amsmath}
\usepackage{amsfonts}
\usepackage{amsthm}
\usepackage{cleveref}
\usepackage{subcaption}
\usepackage{algpseudocode}
\usepackage{algorithm}
\usepackage{multirow}

\newcommand{\bi}{\begin{itemize}}
\newcommand{\ei}{\end{itemize}}
\newcommand{\bal}{\begin{align}}
\newcommand{\eal}{\end{align}}

\newcommand{\EE}{\mathbb{E}}
\newcommand{\PP}{{P}}
\newcommand{\RR}{\mathbb{R}}
\newcommand{\QQ}{{Q}}

\newcommand{\cA}{\mathcal{A}}
\newcommand{\cX}{\mathcal{X}}
\newcommand{\cY}{\mathcal{Y}}
\newcommand{\cZ}{\mathcal{Z}}

\newcommand{\cN}{\mathcal{N}}

\newcommand{\cV}{\mathcal{V}}

\newcommand{\cP}{\mathcal{P}}
\newcommand{\cD}{\mathcal{D}}

\newcommand{\cG}{\mathcal{G}}
\newcommand{\cF}{\mathcal{F}}

\newcommand{\cH}{\mathcal{H}}

\newcommand{\cU}{\mathcal{U}}

\newcommand{\eps}{\epsilon}

\newcommand{\OT}{OT}
\newcommand{\Rad}{\mathfrak{R}}

\DeclareMathOperator{\supp}{supp}

\def\<{\langle}
\def\>{\rangle}
\newtheorem{thm}{Theorem}[section]
\newtheorem{theorem}{\textbf{Theorem}}
\newtheorem{lemma}{\textbf{Lemma}}

\newtheorem{col}{\textbf{Corollary}}

\newtheorem{definition}{\textbf{Definition}}

\newtheorem{remark}{\textbf{Remark}}

\renewcommand{\l}{\left}
\renewcommand{\r}{\right}

\begin{document}
\title{Training Generative Models from Privatized Data
via~Entropic~Optimal~Transport}
\author{Daria~Reshetova,
%,~\IEEEmembership{Member,~IEEE,}
        Wei-Ning~Chen,
        %~\IEEEmembership{Member,~IEEE,}
        and~Ayfer~\"Ozg\"ur%~\IEEEmembership{Member,~IEEE}% <-this % stops a space
\thanks{Department of Electrical Engineering,
  Stanford University,
  Stanford, CA 94205, USA, e-mail: \{resh,aozgur,wnchen\}@stanford.edu\\
% \textcolor{blue}{A conference version of the paper was submitted to AISTATS 2024. This journal version provides several new results regarding the generalization capabilities of Entropic GANs that were not part of the AISTATS 2024 submission. In particular,  the conference version states Theorem III.1, (ii) for $p\in\{1,2\}$ only while the current version provides a more general result for any $p\geq 1$. Theorem III.2 is also a new result not present in the conference version; the conference version only proves Corollary 3.}
}}
  % <-this % stops a space
% \thanks{J. Doe and J. Doe are with Anonymous University.}% <-this % stops a space
% \thanks{Manuscript received April 19, 2005; revised August 26, 2015.}}
\maketitle

\begin{abstract}
Local differential privacy is a powerful method for privacy-preserving data collection. In this paper, we develop a  framework for training Generative Adversarial Networks (GANs) on differentially privatized data. We show that entropic regularization of optimal transport -- a popular regularization method in the literature that has often been leveraged for its computational benefits -- enables the generator to learn the raw (unprivatized) data distribution even though it only has access to privatized samples. We prove that at the same time this leads to fast statistical convergence at the parametric rate. This shows that entropic regularization of optimal transport uniquely enables the mitigation of both the effects of privatization noise and the curse of dimensionality in statistical convergence. We provide experimental evidence to support the efficacy of our framework in practice.  
%In the arena of machine learning, privacy-sensitive domains often limit model deployment due to data access restrictions. iming to circumvent these limitations, we propose an adaptation of the Wasserstein Generative Adversarial Network (GAN) framework to fit the locally differentially private (LDP) setting. At the moment, LDP is esteemed as the leading technique for maintaining privacy during data collection. Our proposed method integrates the denoising attributes of entropic Wasserstein distance with the additive noise privacy-preserving mechanisms. This combination uniquely enables the mitigation of both the regularization bias and the effects of noise, thereby enhancing the overall efficacy of the model. We analyse the proposed method and provide sample complexity results and experimental evidence to support its efficacy.
\end{abstract}

\section{Introduction}
Local differential privacy (LDP) \cite{dwork2006calibrating, kasiviswanathan2011can} has emerged as a popular criterion to provide privacy guarantees on individuals' personal data and has been recently deployed by major technology organizations for privacy-preserving data collection from peripheral devices. In this framework, the user data is locally randomized (e.g. by the addition of noise) before it is transferred to the data curator, so the privacy guarantee does not rely on a trusted centralized server. Mathematically provable guarantees on the randomization mechanism ensure that any adversary that gets access to the privatized data will be unable to learn too much about the user's personal information. This directly alleviates many of the systematic privacy and security challenges associated with traditional data collection. % including those around transparency, data misuse, control, compliance with regulatory strictures, breaches, processing, and release.
Learning from privatized data, however, requires rethinking machine learning methods to extract accurate and useful population-level models from the privatized (noisy) data. 

In this paper, we consider the problem of training generative models from locally privatized user data.  In recent years, deep-learning-based generative models, known as Generative Adversarial Networks (GANs), have become a popular framework for learning data distributions and sampling, and have achieved impressive results in various domains \cite{pan2023drag,booker2023population,chan2022efficient}. GANs aim to learn
a mapping $G(\cdot)$, called the generator,  which comes from a set of functions ${\cG\subseteq\{G:\cZ\to\cX\}}$ usually modeled as a neural network, and maps a latent random variable $Z\in \cZ$ with some known simple distribution to a random variable  $G(Z)\in\cX$, with distribution $\PP_{G(Z)}$ that is close to the target probability measure $P_X$. For example, by using the popular $p$-Wasserstein distance as a discrepancy measure between the generated and target distributions the GAN optimization problem becomes,
\begin{align}\label{eq:GAN}\min_{G\in\cG}W_p^p\left(\PP_{G(Z)},\PP_X\right).
\end{align}
In practice, the target distribution $P_X$ is represented by its samples  $\{X_i\}_{i=1}^n\sim P_X^{\otimes n}$ and the optimization problem is solved by replacing  $P_X$ in \eqref{eq:GAN} with the empirical distribution $\QQ_X^n$ of the samples, i.e. \begin{align}\label{eq:GANemp}
	\min_{G\in\cG}W_p^p\left(\PP_{G(Z)},\QQ_X^n\right).
\end{align} %For example, if $X_i$ represents images taken by users, $G$ represents a generative model for such images.

How can we use the GAN framework above to learn a generative model for $P_X$ when we have only access to samples $\{Y_i=M(X_i)\}_{i=1}^n$ privatized by an LDP mechanism $M:\mathcal{X}\rightarrow\mathcal{Y}$? For example, $Y_i$ can represent a privatized image obtained from $X_i$ by adding sufficient Gaussian or Laplace noise independently to each pixel. Simply replacing the target distribution $P_X$ in \eqref{eq:GAN} with the empirical distribution $\QQ_Y^n$ of the privatized samples, 
\begin{align}\label{eq:GAN2}\min_{G\in\cG}W_p^p\left(\PP_{G(Z)},\QQ_Y^n\right),
\end{align}
will result in a generative model for $P_Y=M\#P_X$, the push-forward distribution of $P_X$ through the privatization mechanism $M$, rather than the original distribution $P_X$. In other words, we will learn to generate the \emph{privatized} data (e.g., noisy images) instead of learning to generate the original (raw) data. 

In this paper, we show that a simple but non-intuitive modification of the objective in \eqref{eq:GAN2} -- the addition of an entropic regularization term -- allows one to provably learn the original distribution $P_X$ from the privatized samples $Y_i$ under de-facto privatization mechanisms such as the Laplace or Gaussian mechanism, i.e. $\PP_{G_n(Z)}\rightarrow\PP_X$ for the minimizer $G_n$ of the entropically regularized version of \eqref{eq:GAN2} provided that the generator class $\mathcal{G}$ is expressive enough to generate $P_X$. More generally, we show that the original distribution $P_X$ can be recovered under any privatization mechanism $M$ by entropic regularization of optimal transport with a suitably chosen cost function given by the negative log-likelihood of the privatization mechanism.
% depends on $M$.  
Note that the fact that we can learn the population distribution $P_X$ from which the original samples have been generated does not imply that we can learn the original samples $X_i$, (i.e., somehow ``denoise'' the observed privatized samples $Y_i$), and indeed the post-processing property of DP ensures that the DP guarantee on the samples $Y_i$ translates to the learned model $G_n$ as well as any new samples generated from this model. 

%We first show that in the population case when $\QQ_Y^n$ is replaced by $P_Y=M\#P_X$, the optimal solution $G^*$ of the \emph{entropic} Optimal Transport GAN is such that $\PP_{G^*(Z)}= P_X$ (assuming $\cG$ is rich enough to generate $P_X$). %The choice of $p$ and the regularization parameter $\lambda>0$  in \eqref{eq:GAN_reg3} depends on the privatization level and mechanism. %This allows one to extend the framework of GANs to $p$-Wasserstein distances due to faster optimizational convergence rates.
%For example, we choose $p=2$ and $\lambda=2\sigma^2$ when $M$ is the Gaussian Mechanism, i.e. each dimension of $X_i$ is corrupted by the addition of i.i.d. Gaussian noise with variance $\sigma^2$. 

%Here, the cost $c$ is chosen to match the privatization mechanism used, for example $c(x,y)\propto\|x-y\|_1$ for the Laplace mechanism and  $c(x,y)\propto\|x-y\|_2^2$ for the Gaussian mechanism. This result shows that the entropic regularization acts as a denoiser for the Gaussian mechanism under the $\OT_{\|\cdot-\cdot\|_2^2}$ distance, and the Laplace mechanism under the $\OT_{\|\cdot-\cdot\|_1}$ distance. See Section~\ref{sec:mainres} for the statement of our result for general privatization mechanisms. We also provide sample complexity results which suggest that the solution of the empirical problem (when $P_Y $ is replaced by $\QQ_Y^n$)  converges to the population solution at the parametric convergence rate $O(1/\sqrt{n})$.

Entropic regularization for Optimal Transport GANs has been of significant interest in the prior literature, albeit for different reasons. Historically, it has been leveraged primarily for its computational benefits, enabling an efficient approximation of the optimal transport problem \cite{cuturi2013sinkhorn}. More recently, \cite{reshetova2021understanding} (also see \cite{mena2019statistical} and \cite{feizi2020understanding}) has shown that it facilitates rapid convergence of GANs and circumvents the curse of dimensionality. In particular,  without regularization the solution of the empirical problem in \eqref{eq:GANemp} converges to the solution of  population problem in \eqref{eq:GAN} as $\Omega(n^{-2/d})$, where $d$ is the dimension of the target distribution ($P_X$), while \cite{reshetova2021understanding} shows that for $p=2$ entropic regularization enables convergence at the parametric rate $O(1/\sqrt{n})$.  In this paper, we prove similar convergence guarantees for the privatized setting for both $p=1$ and $p=2$. In the non-privatized setting of \cite{reshetova2021understanding} entropic regularization of \eqref{eq:GAN} is needed to facilitate convergence albeit introducing undesirable regularization bias that changes the solution (i.e., the generated distribution does not converge to the target distribution $P_X$). In the privatized setting, we show that entropic regularization has the unique advantage of both mitigating the effects of privatization noise and facilitating convergence. {\color{black}Therefore, our framework can be potentially useful even in the unprivatized setting as a way of facilitating convergence without biasing the solution.} %  From the perspective of this literature, our result provides a new application for entropy regularization. We demonstrate that, when applied thoughtfully, entropy regularization can facilitate the learning of the original data distribution even when the data has been subjected to a noise injection  for privacy preservation. This expands our understanding of the potential applications and benefits of entropy regularization, demonstrating its versatility and capacity to enhance outcomes in a privacy-preserving context. 
The contributions of our paper are summarized as follows:
\begin{itemize}
\item \emph{LDP Framework for Optimal Transport GANs}: We propose a novel framework for training GANs from differentially privatized data based on entropic regularization of Optimal Transport. Previous approaches to privatization in GANs exclusively focus on privatizing the training process, for example, by using DP-SGD methods. In contrast, in our framework, privatization occurs exclusively at the data level and hence it is particularly suitable for user-generated data, e.g. federated learning, where each user can locally privatize its data before sending it to the service provider or data collector. The training of the model from privatized samples is indistinguishable from training a non-privatized GAN (with entropic regularization), which enables the immediate use of existing entropic optimal transport libraries.
\item \emph{Sample Complexity Bounds}: We prove convergence guarantees for our LDP framework with entropic optimal transport, including the convergence results for Laplace and Gaussian privatization mechanisms. These results show that entropic regularization uniquely mitigates both the effects of
privatization noise and the curse of dimensionality and provides a clear understanding of the trade-offs involved between privacy, accuracy, and the volume of data. In the non-privatization setting, previous convergence results have been limited to the entropic $2$-Wasserstein distance setting.
\item \emph{Empirical Validation:} We supplement our theoretical contributions with a comprehensive set of experiments designed to validate our claims. These experiments demonstrate the efficacy of our approach in practical scenarios and provide empirical evidence of the superior performance of our method.
\end{itemize}

\subsection{Connection to Rate-Distortion Theory} \label{sec:introRD}
In this section, we illustrate how the main idea of our paper is inherently connected to rate-distortion theory. Consider the special case of \eqref{eq:GAN2} for $p=2$ with the empirical distribution $\QQ_Y^n$ of the privatized samples replaced by
their true distribution $P_Y=M\#P_X$, in which case we can explicitly write it as (see also \eqref{eq:Wasserstein})):\begin{align}\label{eq:intro:Wasserstein}    \min_{G\in\cG}\inf_{\pi\in\Pi(P_{G(Z)},\PP_Y)}\EE_{(G(Z), Y)\sim\pi}\left[\|G(Z)-Y\|^2\right],
	\end{align}
where $\pi\in\Pi(P_{G(Z)},\PP_Y)$ represents the set of all joint distributions on $\mathcal{X}\times\mathcal{Y}$ with marginals $P_{G(Z)}$ and $\PP_Y$. The entropic regularization we advocate in this paper transforms this problem to the following problem (see also \eqref{eq:Wasserstein_reg})
:\begin{align}\label{eq:intro:Wassersteinreg}    \min_{G\in\cG}\inf_{\pi\in\Pi(P_{G(Z)},\PP_Y)}\EE_{(G(Z), Y)\sim\pi}\left[\|G(Z)-Y\|^2\right] +\lambda I_\pi(G(Z), Y),
	\end{align}
where $I_\pi(G(Z), Y)$ is the mutual information between $G(Z)$ and $Y$ as dictated by the joint distribution $\pi$ and $\lambda\in\mathbb{R}$ is the regularization parameter that we can choose. Assuming the set of functions $\cG$ is rich enough to generate any distribution $P_{G(Z)}$ on $\mathcal{X}$ (we relax this condition and make it more precise in Theorem~\ref{thm:denoising}) and relabeling $P_{G(Z)}=P_{\hat{X}}$ for simplicity, we can rewrite \eqref{eq:intro:Wassersteinreg} as
:\begin{align}\label{eq:intro:Wassersteinreg2}    \inf_{P_{\hat{X}|Y}}\EE\left[\|\hat{X}-Y\|^2\right] +\lambda I(\hat{X}, Y).\end{align}
One can recognize this as the Lagrangian form of the following rate-distortion problem under mean-squared error:\begin{align}\label{eq:intro:Wassersteinreg3}    \inf_{P_{\hat{X}|Y}:\EE\left[\|\hat{X}-Y\|^2\right]\leq D}I(\hat{X}, Y), 
\end{align}
where $Y$ can be interpreted as the source variable and $\hat{X}$ as its reconstruction. For general $P_Y$, there is no explicit characterization of the solution of this  problem. Our paper leverages the observation that this problem is easy to solve in one special case: when $Y=X+N$, for arbitrary $X\sim P_{X}$ and $N\sim\mathcal{N}(0,D)$ independent of $X$. In this case, the optimal conditional distribution $P_{\hat{X}|Y}$(or equivalently  the test channel $P_{Y|\hat{X}}$) is such that 
$$Y=\hat{X}+W,\qquad \hat{X}\sim P_X, \qquad W\sim \mathcal{N}(0,D).
$$
This can be observed from the standard characterization of the rate-distortion function for a Gaussian source under mean-squared error; see proof of Theorem 10.3.2 in \cite{coverbook} or see Theorem~\ref{thm:denoising} where we prove a more general result. Note that this implies that the reconstruction $\hat{X}$ of $Y$ has distribution $P_X$ which corresponds to the unprivatized distribution in our setting. Note that this conclusion holds only if the desired compression rate $D$ matches the distribution of the Gaussian component $N\sim\mathcal{N}(0,D)$ of $Y$.  This corresponds to a specific choice of the regularization parameter $\lambda$ in our framework in  \eqref{eq:intro:Wassersteinreg}.
\subsection{Related Work}
Estimation, inference, and learning problems under local differential privacy (LDP) constraints have been of significant interest in the recent literature with emphasis on two canonical tasks: discrete distribution and mean estimation \cite{bassily2017practical,bun2019heavy,chen2021breaking,chen2020breaking,suresh2017distributed,bhowmick2018protection,han2018distributed}. However, insights from these solutions do not extend to training generative models with high-dimensional data under LDP constraints. The understanding of learning problems under LDP constraints is relatively limited, and even less so in the non-interactive setting when the data is accessed only once as in our setting, in which case training can be exponentially harder as shown in \cite{kasiviswanathan2011can,bhowmick2018protection}. % The works have characterized the optimal (order-optimal) LDP mechanisms for both problems as well as sample complexity bounds that reveal the impact of the local privacy constraint on estimation accuracy. However, insights from these solutions do not extend to learning high-dimensional distributions under LDP constraints. In discrete distribution estimation, the alphabet is assumed to be finite, and private mean estimation  leverages the fact that averaging the privatized samples provides an unbiased estimate of the mean. None of these assumptions are applicable in our case. The understanding of learning problems under LDP constraints is relatively limited, and even less so in the non-interactive setting when the data is accessed only once, which can be exponentially harder to train as shown in \cite{kasiviswanathan2011can,bhowmick2018protection}.

The exploration of differentially private learning in generative models has primarily been focused on introducing privacy during the training phase, e.g. by adding noise to the gradients during training \cite{chen2020gs,cao2021don,xie2018differentially,zhang2018differentially}. %While GANs have demonstrated capabilities in synthesizing intricate data, such as high-definition images in a non-private setting \cite{image}, their implementation within a private context presents considerable challenges. 
However, noisy gradients can amplify inherent instability during GANs' training process \cite{arjovsky2017towards,mescheder2018training}. % Even-though, such instabilities can be mitigated with meticulous hyperparameter tuning, this contradicts the essence of privacy that aims to minimize repeated data access \cite{NIPS2013_e6d8545d}. 
These methods can be applied in a federated learning setting by locally privatizing the gradients at each user and transmitting them to the server at every iteration of the learning algorithm \cite{mansbridge2020learning}. However, this leads to a large communication overhead. In contrast, there is only one round of communication in our LDP setting; users send their locally privatized data to the server. The training of the GAN from this privatized data is effectively indistinguishable from the non-private case.

Entropic regularization of optimal transport has been initially proposed as a computationally efficient approximation of optimal transport \cite{cuturi2013sinkhorn, peyre2019computational}. Subsequently, \cite{genevay2019sample,mena2019statistical} have shown statistical convergence benefits of entropic regularization when estimating optimal transport from empirical samples. %\textcolor{black}{, and \cite{RIGOLLET20181228} showed the benefits of projecting with the regularized optimal transport.} 
More recently, \cite{reshetova2021understanding} have shown that these statistical benefits extend to the entropic $W_2$-GAN setting, where both the generated distribution and the target distribution depend on the empirical samples. We extend those results to the privatized setting, showing fast convergence in both the $W_1$ and $W_2$-GAN settings.

\section{Background and Problem Setup}\label{sec:problemform}
To formally state the problem, we first introduce the necessary concepts of privacy.
\subsection{Local Differential Privacy}
%\emph{Local Differential Privacy.} \cite{warner1965randomized,evfimievski2003limiting,kasiviswanathan2011can}  
A local randomizer $\mathcal{A}: \cX\to\cZ$ acting on the data domain $\cX$ satisfies $\epsilon$-LDP for $\epsilon\geq 0$ if for any $S\subseteq \cZ$ and for any pair of inputs $x, x'\in\cX$, it holds that
 \begin{align}
     P(\cA(x)\in S)\leq e^\eps P(\cA(x')\in S)\label{eq:ldp}
 \end{align}
LDP ensures that the input to $\cA$ cannot be accurately inferred from its output. To achieve LDP, one common approach is via the following Laplace mechanism.

\emph{Laplace Mechanism} \cite{dwork2006calibrating}. For any $\eps>0$ and any function $f: \cX\to\RR^k$ such that $\|f(x) - f(x')\|_1\leq \Delta$ for any $x,x'\in\cX,$ the randomized mechanism $\cA(x) = f(x) + (s_1,\ldots,s_k)$ with $s_i\sim \text{Laplace}(0, \Delta/\eps)$ independent of $s_j, j\neq i$ satisfies $\eps$-DP and is called the Laplace Mechanism. We will call $\eps/\Delta$ the noise scale of the mechanism (also called noise multiplier).
Oftentimes, in ML applications, the (pure) LDP constraint may be too stringent, so the following relaxation on pure LDP is often adopted.

\emph{Approximate Local Differential Privacy.} A local randomized algorithm $\mathcal{A}: \cX\to\cZ$ acting on the data domain $\cX$ satisfies $(\epsilon,\delta)$-(approximate) LDP for $\epsilon\geq 0,\delta\in (0,1)$, if for any $S\subseteq \cZ$ and for any pair of inputs $x, x'\in\cX$, it holds that
 \begin{align}
     P(\cA(x)\in S)\leq e^\eps P(\cA(x')\in S)+\delta\label{eq:approx_ldp}
 \end{align}
$(\epsilon,\delta)$-LDP is very similar to pure LDP, but it allows the privacy requirement to be violated with (small) probability $\delta$. One of the most versatile mechanisms to achieve $(\eps,\delta)$-DP is the following Gaussian Mechanism.

\emph{Gaussian Mechanism} \cite{dwork2006our,dwork2014algorithmic} For any $\eps>0$, $\delta\in(0,0.5)$, and any function $f: \cX\to\RR^k$ such that $\|f(x) - f(x')\|_2\leq \Delta$ for any $x,x'\in\cX,$ the randomized mechanism $\cA(x) = f(x) + (s_1,\ldots,s_k)$ with $s_i\sim \cN(0, \sigma^2)$ independent of $s_j, j\neq i$ is called the Gaussian Mechanism and satisfies $(\eps,\delta)$-DP if
\begin{align}
\sigma>\frac{c + \sqrt{c^2+\eps}}{\eps\sqrt2}\Delta,\text{where } c^2 = \ln\frac2{\sqrt{16\delta+1}-1}.\label{eq:gaussian_eps}
\end{align}
Similar to the Laplace mechanism, we will call $\sigma$ the noise scale of the Gaussian mechanism.

% \textcolor{red}{introduce formally LDP, Laplace and Gaussian mechanisms, their corresponding eps, delta guarantees etc.}
\subsection{Optimal Transport GANs}
Optimal Transport GANs minimize the distance between the target and generated distributions. Contrary to the Jensen-Shannon divergence, which was first introduced as a loss function for generative models, and many other popular distances on probability measures (total variation distance, KL-divergences), optimal transport (OT)  is defined through a cost function in the sample space and thus is meaningful for distributions with non-overlapping supports. Moreover, for certain costs, OT defines a distance between distributions and metrizes weak convergence on distributions with finite moments.

\emph{Optimal Transport.} Let  $c: \cU\times \cV \to \RR_+$ be a cost function taking non-negative values and $\cP(\cU)$ be the set of all probability measures with support $\cU\subseteq\mathbb R^d$. Then for $\cU,\cV\subseteq \RR^d$ and $P_U\in\cP(\cU), P_V\in\cP(\cV)$,  two probability measures on $\cU, \cV$ respectively, the Optimal Transport between $P_U$ and $P_V$ is
	\begin{align}\label{eq:OT}
	\OT_c(P_U,P_V) =    \inf_{\pi\in\Pi(P_U,P_V)}\EE_{(U, V)\sim\pi}\left[c(U,V)\right],
	\end{align} 
 where $\Pi(P_U,P_V) = \{\pi\in\cP(\cU\times\cV):\; \int_{\cV}\pi(u, v)dv = P_U(U), \int_{\cU}\pi(u, v)du = P_V(v)\}$ is the set of all couplings of $P_U$ and $P_V$, i.e. all joint probability measures with marginal distributions $P_U$ and $P_V.$  
 
\emph{$p$-Wasserstein distance} When the cost is $c(x, y) = \|x-y\|_p^p$ the optimal transport becomes the $p$-Wasserstein distance between $P_U, P_V$ (raised to power $p$):
	\begin{align}\label{eq:Wasserstein}
	W_p^p(P_U,P_V) =    \inf_{\pi\in\Pi(P_U,P_V)}\EE_{(U, V)\sim\pi}\left[\|U-V\|^p_p\right].
	\end{align} 

\emph{Optimal Transport GAN.}	The main objective of GANs is to find a mapping $G(\cdot)$, called a generator, that comes from a set of functions ${\cG\subseteq\{G:\cZ\to\cX\}}$ (usually modeled as a neural network) and maps a latent random variable $Z\in \cZ$ with some known distribution to a variable $X\in\cX$ with some target probability measure $P_X$ approximated by the empirical distribution $\QQ_X^n$ of $n$ samples $\{X_i\}_{i=1}^n\sim P_X^{\otimes n}$. Using the optimal transport to measure the dissimilarity between the generated $P_{G(Z)}$ and target distribution leads to the following learning problem of GAN:
	\begin{align}\label{eq:GAN_general}
	\min_{G\in\cG}\OT_c\l(\PP_{G(Z)},\QQ_X^n\r).
	\end{align}
 Note that when the cost function is a distance raised to power $p$ as in \eqref{eq:GAN} and \eqref{eq:GANemp}, the GAN is also known as $p$-Wasserstein GAN \cite{arjovsky2017wasserstein,korotin2019wasserstein}.
 
% So, in practical scenarios the dual version for $p=1$ is chosen as an objective \cite{arjovsky2017wasserstein}:
% \begin{align}\label{eq:Wasserstein_dual}
% W(P_{G(Z)},P_Y) = \sup_{f: |f(z)-g(y)|\leq \|z-y\|_1\forall z,y}\EE_{z\sim P_Z} f(G(Z)) - \EE_{Y\sim P_Y}f(Y).
% \end{align} 
% The dual potential $f$ is then approximated by an adversarially trained neural network, which leads to instabilities and bias since the discriminator is not trained to optimality \cite{bousquet2017optimal} and also involves heuristics to ensure the lipschitzness of the discriminator. 
\emph{Entropic Optimal Transport GAN.}
Solving the formulation in \eqref{eq:GAN_general} involves solving for the optimal transport plan $\pi$ — a joint distribution over the real and generated sample spaces, which is a difficult optimization problem with very slow convergence. Adding entropic regularization to the objective makes the problem strongly convex and thus solvable in linear time \cite{peyre2019computational}. 

Formally, the entropy-regularized optimal transport, also known as Sinkhorn distance \cite{cuturi2013sinkhorn}, is defined as 
\begin{align}\label{eq:OT_reg}
S_{c}(P_U,P_V) &= \inf_{\pi\in\Pi(P_U,P_V)}\EE_{(U, V)\sim\pi}\left[c(U,V)\right] + I_{\pi}(U, V),\end{align} 
where $I_\pi(U, V)=\int \log\left(\frac{d\pi(u, v)}{dP_U(u) dP_Y(V)}\right)d\pi(u, v)$ is the mutual information between $U$ and $V$ under the coupling $\pi.$ In case $c(x, y) = \|x-y\|^p_p/\lambda,$ the Sinkhorn distance is proportional to the entropy-regularized Wasserstein distance 
\begin{align}\label{eq:Wasserstein_reg}
\lambda S_c(P_U, P_V) &= W_{p, \lambda}(P_U,P_V)= \inf_{\pi\in\Pi(P_U,P_V)}\EE_{(U, V)\sim\pi}\left[\|U-V\|^p_p\right] + \lambda I_{\pi}(U, V),
\end{align} 
The objective of an entropic optimal transport GAN  is  the entropy-regularized optimal transport between the generated distribution $ G\#P_Z = P_{G(Z)}$ for some latent noise $Z$ and the empirical approximation $Q_X^n$ of the target distribution:
\begin{align}\label{eq:OT_reg_obj}
\min_{G\in\cG}S_c(P_{G(Z)},Q_X^n)
\end{align}
and the objective of an entropic $p$-Wasserstein GAN is
\begin{align}\label{eq:Wasserstein_reg_obj}
\min_{G\in\cG}W_{p, \lambda}(P_{G(Z)},Q_X^n).
% &\quad= \min_{G\in\cG}\inf_{\pi\in\Pi(P_Z,Q_X^n)}\EE_{(Z, X)\sim\pi}\left[\|G(Z)-X\|^p_p\right] + \lambda I_{\pi}(G(Z), X),\nonumber
\end{align} 
It is worth mentioning that both non-regularized and regularized optimal transport formulations admit a dual formulation with optimization over functions of the input random variables. We note that the dual formulation for regularized optimal transport is unconstrained and hence easier to use, while the constraints for the unregularized counterpart are usually harder to enforce (e.g., Lipschitzness~\cite{arjovsky2017wasserstein} or convexity~\cite{korotin2019wasserstein}).

% \textcolor{red}{Formally define W1, W2, the corresponding GAN problems etc.}
\subsection{Optimal Transport GANs with LDP Data}
Let $M:\cX\to\cY$ be a randomized privacy-preserving mechanism: $Y=M(x)\sim P_M(y\mid x).$ For example, $P_{M(X)\mid X}(y\mid x)$ can be the Laplace pdf at $y-x$ for the Laplace mechanism or the Gaussian pdf at $y-x$ for the Gaussian mechanism. Let  $P_Y=M\#P_X$ denote the distribution of $Y$, i.e. the push-forward distribution of $P_X$ through the privatization mechanism $M$. The goal of learning a GAN from privatized samples is to reconstruct $G(Z)\approx X$ in distribution from a sample $S = \{Y\}_{i=1}^n\sim P_Y^{\bigotimes n}$ with empirical distribution  $Q_Y^n = \frac1n \sum_{i=1}^n\delta_{Y_i}.$ 
% \textcolor{red}{Formally the problem we would like to solve, similar to the intro but more formally}

\section{Main Results}\label{sec:mainres}
First, we focus on the population setting where the distribution of the privatized samples $P_Y$ is known and show that by tailoring the cost function for optimal transport to the privatization mechanism $M$, the GAN learning problem with entropic optimal transport can recover the raw data distribution $P_X$.

\begin{thm}
    Let $X\sim P_X$ and $Y = M(X) \sim P_{M(X)\mid X}(\cdot\mid X).$ Assume that  the privatisation mechanism $M$ and the set of generator functions $\cG$ is such that for any $G\in\cG$,  
\begin{align}\label{eq:privatization_condition}
    \text{if}\qquad    D_{KL}(P_{X}\|P_{G(Z)})>0\qquad \text{then}\qquad
 D_{KL}(P_{M(X)}\|P_{M(G(Z))})>0.
    \end{align}
Let $c(x, y) = -\log P_{M(X)\mid X}(y\mid x)$ and
    \begin{equation}\label{eq:thm1}
    G^* = \arg\min_{G\in\cG} S_c(P_{G(Z)}, P_Y). 
    \end{equation}
If $P_X\in \{P_{G(Z)}\mid G\in\cG\}$, i.e. $P_X$ is realizable with the set of  generator functions $\cG$, then $P_{G^*(Z)}=P_X.$
    \label{thm:denoising}
\end{thm}
\begin{proof}
Fix some $G\in\cG$ and assume that $Y$ is a continuous random variable. Denote $\cP = \Pi(P_{G(Z)},P_Y), U = G(Z)$ Then by the definition of mutual information:
\begin{align*}
S_c(P_{G(Z)},P_Y)&= \inf_{\pi\in\cP} \EE_{(U,Y)\sim\pi}[-\log p_M(Y\mid U)] + I_\pi(U, Y)\nonumber\\
&=\inf_{\pi\in\cP}\left\{-\int \log p_M(y\mid u)\pi(u,y) dudy +\int \log \l(\frac{\pi(u,y)}{P_{U}(u)P_Y(y)}\r)\pi(u,y) dudy\right\}
\end{align*}

Notice that the terms on the RHS can be rearranged into the Kullback-Leibler divergence:
\begin{align}
S_c(P_{G(Z)},P_Y)&=\inf_{\pi\in\cP}\left\{-\int \log (P_Y(y))\pi(u,y) dudy+\int \log \l(\frac{\pi_{Y\mid U}(y\mid u)}{p_M(y\mid u)}\r)\pi(u,y) dudy\right\}\nonumber\\
% &=\inf_{\pi\in\cP}\!\l\{\!-\!\int \log (P_Y(y))P_Y(y)dy +
&=\inf_{\pi\in\cP}\l\{h(Y) +
\EE_{U\sim P_{G(Z)}} 
D_{KL}(\pi_{Y\mid U}(\cdot\mid U)\|p_M(\cdot\mid U))\label{eq:Sc_to_KL}
% D_{KL}(\pi\|p_M^G)
\r\}
\end{align}
The right-hand side is minimized when $\pi_{Y\mid U}(y\mid u)=p_M(y\mid u)$ for any $u\in\supp(P_{G(Z)}),\,y\in\supp (P_Y),$ which is a feasible coupling only if $P_{M(G(Z))}=P_Y.$ By the realizability assumption $P_Y = P_{M(X)} = P_{M(G^*(Z))},$ so $P_{M(G(Z))}=P_Y\iff P_{M(G(Z))}=P_{M(G^*(Z)},$ or equivalently $D_{KL}(P_{M(G(Z))}\|P_{M(G^*(Z)}) = 0$. In addition to that, \eqref{eq:privatization_condition} implies that whenever the privatized distributions are equal $D_{KL}(P_{M(G(Z))}\|P_{M(G^*(Z)}) = 0,$ it has to be that the original data distributions are equal too, so we conclude that $D_{KL}(P_{G(Z)}\|P_{G^*(Z)}) = 0,$ or equivalently $P_{G(Z)}=P_{G^*(Z)}=P_X$
% By the assumption of the theorem it implies $P_{G(Z)} = P_{G^*(Z)}=P_X.$
The same set of equalities also holds in the case when $Y$ is discrete with differential entropy changing to entropy.
\end{proof}

% \begin{thm}
%     Let $X\sim P_X$ and $Y = M(X) \sim P_{M(X)\mid X}(\cdot\mid X),$ where the privatisation mechanism $M$ and the set of generator functions $\cG$ is such that for any $G\in\cG$ 
%     \begin{align}\label{eq:privatization_condition}
%         D_{KL}(P_{X}\|P_{G(Z)})>0\implies D_{KL}(P_{M(X)}\|P_{M(G(Z))})>0.
%     \end{align}
%     Set $c(x, y) = -\log P_{M(X)\mid X}(y\mid x)$ and let the optimal (population) generator be
%     \begin{equation}\label{eq:thm1}
%     G^* = \arg\min_{G\in\cG} S_c(P_{G(Z)}, P_Y).    
%     \end{equation}

%     We have:
%     \begin{itemize}
%             \item (i) If $P_X\in \{P_{G(Z)}\mid G\in\cG\},$  then $P_{G^*(Z)}=P_X.$
%         \item(ii) If $P_X\notin \{P_{G(Z)}\mid G\in\cG\},$ then for $M(X) = X+N,$
%         where $N\sim f_N(z)\propto e^{-\|z\|_p^p/(p\sigma^p)}, \,p\geq 1,$ and $\sigma > 0,$ is the noise scale
%     \begin{align*}
%        D_{KL}(P_{G^*(Z)+N}\|P_{X+N})\leq \min_{G\in\cG}
% \begin{cases}
%     W_{2}^2(P_{G(Z)}, P_X)&\text{ if } p=2,\\
%     p2^{p-1}W_p(P_{G(Z)}, P_X)\l(\sigma^p+W_p(P_{G(Z)}, P_X)\r)^{1-1/p}&\text{ if } p\geq 1.
% \end{cases}%\label{eq:W_bound},
%     \end{align*}
%     where $D_{KL}(P\|Q) = \int P\log\frac{dP}{dQ}$ is the KL-divergence.
%     \label{thm:denoising}
%     \end{itemize}
% \end{thm}

% Note that for the additive noise mechanisms $M(X) = X + N,$ with the noise $N$ independent of $X$ and the noise pdf of the form $f_N(x) \propto e^{-\|x\|_p^p/(p\sigma^p)},$ the entropic optimal transport GAN reduces to the entropic $p$-Wasserstein GAN.
The theorem indicates that the optimal solution to the GAN optimization problem \eqref{eq:thm1}  generates the target distribution $P_X$, assuming that the generator class $\cG$ is expressive enough to generate the target distribution.  Note that the cost function in the definition of the entropic optimal transport $S_c(P_{G(Z)}, P_Y)$ in \eqref{eq:thm1} has to be chosen as $c(x, y) = -\log P_{M(X)\mid X}(y\mid x)$ to match the privatization mechanism $M$. Thus provided that there are enough privatized samples, the generator will output the raw target distribution. Assumption \eqref{eq:privatization_condition} ensures that privatizing any generated distribution other than $P_X$ will result in a distribution different from $P_Y.$ This condition is needed to eliminate degenerate cases of privatization mechanisms, for example $M(X) = 0.$ Note that if $P_{M(X)} = P_{M(G(Z))}$ and $P_{G(Z)}\neq P_{X},$ it is not possible to differentiate between them since the learning framework only has access to the privatized distribution. Moreover, we note that the condition is satisfied for any additive noise privatization mechanism provided that the noise characteristic function is non-zero everywhere, which holds for Laplace and Gaussian privatization mechanisms.

 % Moreover, when the true data distribution $P_X$ cannot be exactly generated by any model in $\cG,$ i.e. the approximation error of the class $\cG$ given by $\min_{G\in\cG} W_{p}(P_{G(Z)},P_X)$ is non-zero, part (ii) of the theorem bounds the KL Divergence between the pushforwards of the generated and target distributions. This is sometimes called the smoothed KL divergence between $P_{X}$ and $P_{G^*(Z)}$ \cite{goldfeld2020convergence}.   Part (ii) ensures that if $\eps$ is the approximation error in $p$-Wasserstein distance of the class $\cG$, then $P_{G*(Z)}$ is $\eps$-close to the target distribution $P_X$ in smoothed KL-divergence. The theorem thus justifies using entropic optimal transport as a loss function for learning from privatized data. 
 
 We next show that when the privatization mechanism is given by the popular Laplace or the Gaussian mechanisms, the entropic OT problem reduces to the entropic $W_1$ and $W_2$ GAN problems respectively.
\begin{col}\label{cor1}
    Under the conditions of Theorem \ref{thm:denoising}, if $\sup_{x\in\cX}\|x\|_1\leq \Delta_1,$   $p=1$, and $Y=M(X)$ is the Laplace mechanism with noise scale $\eps/\Delta_1$, then training a GAN with loss $W_{1, \eps/\Delta_1}(P_{G(Z)}, P_Y)$ is $\eps$-LDP, and recovers the target distribution: 
    $P_{G^*(Z)} = P_X.$\label{col:laplace}
\end{col}

\begin{col}\label{cor2}
    Under the conditions of Theorem \ref{thm:denoising}, if $\sup_{x\in\cX}\|x\|_2\leq \Delta_2,$   $p=2$, and $Y=M(X)$ is the Gaussian mechanism with noise scale $\sigma$ defined in \eqref{eq:gaussian_eps}, then training a GAN with loss $W_{2, 2\sigma^2}(P_{G(Z)}, P_Y)$ is $(\eps, \delta)$-LDP, and recovers the target distribution: 
    $P_{G^*(Z)} = P_X.$\label{col:gauss}
\end{col}

%\textcolor{black}{ It is important to note that when the privatization mechanism is an additive noise mechanism, namely $M(X) = X + N$ for some $N\sim f_N,$ then the GAN problem is effectively a deconvolution problem, which has been studied in previous literature. In particular, the seminal work of \cite{RIGOLLET20181228} first showed that projection with respect to the entropic optimal transport is maximum likelihood deconvolution. Corollaries \ref{col:laplace} and \ref{col:gauss} can also be derived using the result of \cite{RIGOLLET20181228}. However, the authors do not draw the connection between LDP and entropic optimal transport and do not propose to use an entropic GAN to perform the deconvolution and recover the distribtuion. Moreover, the results of \cite{RIGOLLET20181228} are only applicable to additive privatization mechanisms and require a more sophisticated proof than that of theorem \ref{thm:denoising}.}

{\color{black} We note that first \cite{RIGOLLET20181228} showed that projection with respect to entropic optimal transport is maximum likelihood deconvolution, and Corollary~\ref{cor1} and \ref{cor2} can be viewed as the population case of \cite{RIGOLLET20181228} when the data distribution comes from a convolution model. While \cite{RIGOLLET20181228} does not draw a
connection between LDP and entropic optimal transport and is not concerned with proposing an entropic GAN  framework for privatized data, their result has a similar flavor to our results in Corollary~\ref{cor1} and \ref{cor2}.  However, we note that proving  that entropic projection is maximum likelihood deconvolution as done in \cite{RIGOLLET20181228} is more involved, while our Theorem\ref{thm:denoising}, which holds for general privatization mechanisms and not only under additive noise ones as in \cite{RIGOLLET20181228}, simply follows from the non-negativity of KL divergence and is inherently related to the characterization of the rate distortion function as discussed in Section~\ref{sec:introRD}}.

% {\color{blue} We  note that \cite{RIGOLLET20181228} showed that projection with respect to entropic optimal transport is maximum likelihood deconvolution. While \cite{RIGOLLET20181228} does not draw a
% connection between LDP and entropic optimal transport and is not concerned with proposing an entropic GAN  framework for privatized data, their result has a similar flavor to our results in Corollary\ref{cor1} and \ref{cor2}. However, we note that proving  that entropic projection is maximum likelihood deconvolution as done in \cite{RIGOLLET20181228} is more involved, while our Theorem\ref{thm:denoising}, which holds for general privatization mechanisms, simply follows from the non-negativity of KL divergence and is inherently related to the characterization of the rate distortion function as discussed in Section~\ref{sec:introRD}}.

The results so far are only applicable to the realizable case $P_X\in\{P_{G(Z)}\mid G\in\cG\},$ namely when the true data distribution $P_X$ can be generated. However, this is not always the case in practice, and the approximation error of the class $\cG$ given by $\min_{G\in\cG} W_{p}^p(P_{G(Z)},P_X)$ can be non-zero. The following lemma can be used in this case. {\color{black} Note that it holds for $p\in\{1, 2\}$ which correspond to the Laplace and Gaussian mechanisms respectively.}

\begin{lemma}
    Let $X\sim P_X$ and $Y = M(X) = X+N,$ where $N\sim f_N(z)\propto e^{-\|z\|_p^p/(p\sigma^p)}, p\in\{1, 2\}$, and
    \begin{equation*}
    G^* = \arg\min_{G\in\cG} W_{p, p\sigma^p}(P_{G(Z)}, P_Y).
    \end{equation*}
If $P_X\notin \{P_{G(Z)}\mid G\in\cG\}:$ 
    \begin{align*}
       D_{KL}(P_{G^*(Z)+N}\|P_{X+N})\leq \min_{G\in\cG} W_{p}^p(P_{G(Z)},P_X)%\label{eq:W_bound},
    \end{align*}
    where $D_{KL}(P\|Q) = \int P\log\frac{dP}{dQ}$ is the KL-divergence.
    \label{lemma:denoising_nonrealizable}
\end{lemma}
 Lemma \ref{lemma:denoising_nonrealizable} bounds the KL Divergence between the pushforwards of the generated and target distributions. This is sometimes called the smoothed KL divergence between $P_{X}$ and $P_{G^*(Z)}$ \cite{goldfeld2020convergence}.   It ensures that if $\eps$ is the approximation error in $p$-Wasserstein distance of the class $\cG$, then $P_{G*(Z)}$ is $\eps$-close to the target distribution $P_X$ in smoothed KL-divergence. We prove the lemma in section \ref{sec:proof_denoising_nonrealizable}
% In this framework, we consider the target data distribution to have compact support. Indeed, for the Laplace mechanism the data distribution has to be supported on an $\ell_1$. The following theorem establishes an excess risk bound for a wide class of privatization mechanisms. It, however, is not directly applicable to the Laplace privatization mechanism. We address this and several other cases in the corollary.
% The result for $p=2$ is

Lemma \ref{lemma:denoising_nonrealizable}, Theorem \ref{thm:denoising}, and Corollaries \ref{col:laplace}, \ref{col:gauss} have been established in the population setting where we work directly with $P_Y$. In practice, $P_Y$ is approximated by the empirical distribution $Q_Y^n$ of its samples $\{Y\}_{i=1}^n\sim P_Y^{\bigotimes n}$. We next investigate how fast the solution of the empirical problem converges to the population solution and first we establish a new result on convergence of the entropic optimal transport that is suited to our framework.

\begin{thm}(Entropic optimal transport GAN excess risk bound) \label{thm:generalization_general}
    Let the target data distribution $\PP_X$ be a probability measure with bounded support, namely $ \supp P_X\subseteq\cX\subset  \RR^d$ and $\sup_{x\in\cX}\|x\|_{\infty} =D < \infty,$ and let the set of admissible generators be $\cG\subseteq \{G: \cZ\to\cX\}.$ Let the cost function $c:\cX\times\cY\to \RR$ be measurable with respect to the product measure: $\EE_{(X, G(Z))\sim P_X\times P_{G(Z)}}c(X, G(Z))<\infty$ for any $G\in\cG$ and non-negative: $c(x, y)\geq 0$ for any $x\in\cX, y\in\RR^d.$ If the exponentiated negative cost function
% density of the privatization mechanism
is a Mercer kernel, that is $k(x, y) = e^{-c(x,y)}$ is
\begin{itemize}
    \item continuous
    \item symmetric: $k(x, y) = k(y, x)$
    \item positive definite: for any number $n$ and any set of points $\{x_i\}_{i=1}^n\subset \cX$ the matrix with entries $K_{ij} = K(x_i,x_j)$ is positive semi-definite
\end{itemize}
then for 
% $c(x,y) = -\log P_{Y\mid X}(y\mid x),$
$$G^* = \arg\min_{G\in\cG} S_c(P_{G(Z)},P_Y),$$
$$G_n = \arg\min_{G\in\cG} S_c(P_{G(Z)},Q_Y^n), $$ where $Q_Y^n$ is the empirical distribution of $\{Y_i\}_{i=1}^n$ -- $n$ i.i.d. samples from $P_Y$ it holds that
	\begin{align*}
	\EE \bigl[S_c(\PP_{G_n(Z)},\PP_Y)- S_c(\PP_{G^*(Z)},\PP_Y)\bigr]&\leq \frac4{\sqrt{n}}\EE\l[\sup_{x\in\cX}e^{2c(x, Y)}\r].
	\end{align*}
\end{thm}
The detailed proof of the theorem is given in section \ref{sec:proof_generalization_general} and is based on two main ideas: one is a simple decomposition of the dual function of entropic optimal transport similar to \cite{stromme2023minimum} and the other one is a Rademacher complexity bound for one of the dual potential, which we obtain through the Mercer's decomposition of the conditional probability distribution. We defer the proof to section \ref{sec:proofs}, while we provide a discussion and give two important corollaries -- for a general privatization mechanism and the Laplace mechanism.

The above theorem is easily adjusted to the privatization framework by plugging in $c(x,y) = -\log P_{M(X)\mid X}(y\mid x),$ which results in the following corollary
\begin{col}
    Let the target distribution $\PP_X$ be a probability measure with bounded support, namely $ \supp P_X\subseteq\cX\subset  \RR^d$ and $\sup_{x\in\cX}\|x\| < \infty,$ and let the set of admissible generators be $\cG\subseteq \{G: \cZ\to\cX\}.$ Additionally, let the distribution function of the privatization mechanism
% density of the privatization mechanism
be a Mercer kernel, that is $k(x, y) =  P_{M(X)\mid X}(y\mid x)$ is
\begin{itemize}
    \item continuous
    \item symmetric: $k(x, y) = k(y, x)$
    \item positive definite: for any number $n$ and any set of points $\{x_i\}_{i=1}^n\subset \cX$ the matrix with entries $K_{ij} = K(x_i,x_j)$ is positive semi-definite.
\end{itemize}
Then for $Y = M(X)\sim p_{M(X)\mid X}$ and 
$c(x,y) = -\log P_{Y\mid X}(y\mid x),$
$$G^* = \arg\min_{G\in\cG} S_c(P_{G(Z)},P_Y),$$
$$G_n = \arg\min_{G\in\cG} S_c(P_{G(Z)},Q_Y^n), $$ where $Q_Y^n$ is the empirical distribution of $\{Y_i\}_{i=1}^n$ -- $n$ i.i.d. samples from $P_Y$ it holds that
	\begin{align}
	\EE_{Y\sim P_Y} \bigl[S_c(\PP_{G_n(Z)},\PP_Y)- S_c(\PP_{G^*(Z)},\PP_Y)\bigr]&\leq \frac4{\sqrt{n}}\EE\l[\sup_{x\in\cX}\frac{1}{P_{M(x)\mid X}(Y\mid x)^2}\r]\label{eq:privatization_generalization}
 % &\leq \frac4{\sqrt{n}}\EE\l[\int\frac{P_{M(x)\mid X}(y\mid u)P_X(u)}{\inf_{x\in\cX}P_{M(x)\mid X}(y\mid x)^2}dudy\r].
	\end{align}
\end{col}
One can now apply the corollary to the Laplace mechanism, whose pdf is a Mercer kernel. However, the direct application of the theorem will not result in a meaningful upper bound since the RHS of \eqref{eq:privatization_generalization} is infinite, but one can still use theorem \ref{thm:generalization_general} after noting that the cost function decomposes into $c(x,y)\propto \|x-y\|_1 =a(y)+\tilde{c}(x,y),$ where the function $\tilde{c}$ is bounded and the term $a(y)$ only depends on $y$ and not $x,$ so it can be factored out of the entropic optimal transport. This leads to the following corollary, where we only require the support of the data distribution to be bounded. 
\begin{col}\label{col:generalization_lap}
    If $Y = M(X)$ is the Laplace mechanism with noise scale $\sigma,$ the support of the data distribution is bounded in $\infty$ norm: $\sup_{x: P_X(x)>0}\|x\|_{\infty}\leq D$ as well as the output of the generator functions: $\forall G\in \cG$ the $\infty$ norm of the output does not exceed $D:\; \sup_{G\in\cG}\|G\|_{\infty}\leq D$ then for 
     $$G^* = \arg\min_{G\in\cG} W_{1, \sigma}(P_{G(Z)},P_Y),$$ $$G_n = \arg\min_{G\in\cG} W_{1, \sigma}(P_{G(Z)},Q_Y^n), $$ where $Q_Y^n$ is the empirical distribution of $n$ i.i.d. samples from $P_Y$ it holds that
	\begin{align*}
	\EE \bigl[W_{1,\sigma}(\PP_{G_n(Z)},\PP_Y)&- W_{1,\sigma}(\PP_{G^*(Z)},\PP_Y)\bigr]\leq \frac{4\sigma e^{4dD/\sigma}}{\sqrt{n}}.
	\end{align*}
\end{col}
\begin{proof}
    For the Laplace mechanism $c(x, y) = d\log(2\sigma)+\|x - y\|_1/\sigma.$ Let $(b(y))_i = \begin{cases}
        y_i&\text{if } |y_i|\leq D\\
        D&\text{if } y_i> D\\
        -D&\text{if } y_i< -D\\
    \end{cases},$
    namely $b(y)$ clips $y$ to the interval $[-D,D]$ coordinate-wise. Then $|x_i - y_i| = |x_i - b(y_i)| + |b(y_i) - y_i|,$ denote $\tilde c(x,y) = c(x, b(y)).$ Since 
    $$c(x, y) = \tilde{c}(x, y)+\|y-b(y)\|_1,$$
    where the term $\|b(y) - y\|_1$ does not depend on $x$ and, therefore, the coupling one can write 
    $$S_c(P_{G(Z)},P_Y) = S_{\tilde c}(P_{G(Z)},P_Y) + \EE_{Y\sim P_Y} \|Y - b(Y)\|_1,$$
    which leads to the following exess risk bound
    \begin{align*}
        \EE \bigl[S_c(\PP_{G_n(Z)},\PP_Y)- S_c(\PP_{G^*(Z)},\PP_Y)\bigr]&=
        \EE \bigl[S_{\tilde c}(\PP_{G_n(Z)},\PP_Y)- S_{\tilde c}(\PP_{G^*(Z)},\PP_Y)\bigr]
    \end{align*}
    Theorem \ref{thm:generalization_general} can now be applied to $S_{\tilde{c}}$ to result in 
    \begin{align*}
        \EE \bigl[S_c(\PP_{G_n(Z)},\PP_Y)- S_c(\PP_{G^*(Z)},\PP_Y)\bigr]&\leq \frac4{\sqrt{n}}\EE\l[\sup_{x\in\cX}e^{2 c(x, b(y))}\r]
        \leq \frac4{\sqrt{n}}\sup_{x\in\cX, y:\|y\|_{\infty}\leq D}e^{2\|x-y\|_1/\sigma}\leq \frac4{\sqrt{n}}e^{4dD/\sigma}
    \end{align*}
    Expressing the Wasserstein distance in terms of optimal transport using \eqref{eq:Wasserstein_reg} leads to 
    \begin{align*}
        \EE \bigl[W_{1,\sigma}(\PP_{G_n(Z)},\PP_Y)- W_{1,\sigma}(\PP_{G^*(Z)},\PP_Y)\bigr]&\leq  \frac{4\sigma}{\sqrt{n}}e^{4dD/\sigma}
    \end{align*}
\end{proof}
Note that to achieve $\epsilon$-differential privacy one needs to choose $\sigma\geq \epsilon/\sup_{x: P_X(x)>0}\|x\|_{1}$ and if, for example, the data is supported on $\cX=\{x\in\RR^d\mid \|x\|_{\infty}\leq D\}$ then to ensure $\epsilon$-LDP on needs to choose $\sigma=\sup_{x\in\cX}\|x\|_1/\epsilon = dD/\epsilon,$ which leads to 
	\begin{align*}
	\EE \bigl[W_{1,\sigma}(\PP_{G_n(Z)},\PP_Y)&- W_{1,\sigma}(\PP_{G^*(Z)},\PP_Y)\bigr]\leq \frac{4dD e^{4\epsilon}}{\epsilon\sqrt{n}}.
	\end{align*}
Note that in this case the curse of dimensionality is not only removed, but the excess risk scales linearly with the dimension. 
 
We next state the generalization result for the Gaussian mechanism which is proven in section \ref{sec:proof_generalization_gauss}. To formally state the sample complexity results, let us first recall some definitions. A distribution $\PP_X$ supported on a $d$-dimensional set $\cX$ is $\sigma^2$  sub-gaussian for  $\sigma\geq0$ if $\EE\exp\l(\|X\|^2/(2d\sigma^2)\r)\leq 2.$
Let $\sigma^2(X) \!\!=\! \min\{\sigma\!\!\geq0\!\l\vert \EE\exp(\|X\|^2/(2d\sigma^2))\leq 2\r.\}$
denote the sub-gaussian parameter of the distribution of $X.$ A set of generators $\cG$ is said to be star-shaped with a center at $0$ if a line segment between $0$ and $G\in\cG$ also lies in $\cG,$ i.e. 
\begin{align}
G\in\cG\; \Rightarrow \alpha G\in\cG, \forall\alpha\in[0,1]. \label{eq:assump_lin}
\end{align}
Note that these conditions are not very restricting. For example, the set of all linear generators, the set of linear functions with a bounded norm or a fixed dimension, the set of all L-Lipschitz functions or neural networks with a relu ($f(x) = \max(0,x)$) activation function at the last layer all satisfy \eqref{eq:assump_lin}.

\begin{thm} (Excess Risk of the Gaussian Mechanism) %(adapted from \cite{reshetova2021understanding})
\label{thm:generalization_gauss}
Let
 $\PP_{Z}$ and $\PP_X$ be sub-gaussian, the support of  $P_X$ be $d$-dimensional, and the generator set $\cG$ consist of $L$-Lipschitz  functions, namely $\|G(Z_1) - G(Z_2)\|\leq L\|Z_1-Z_2\|$ for any $Z_1,Z_2\in\cZ.$ Assume additionally that $\cG$ satisfies \eqref{eq:assump_lin}. If $Y = M(X) = X + N$ is the Gaussian mechanism with noise scale $\sigma$ then for 
 $$G^* = \arg\min_{G\in\cG} W_{2, 2\sigma^2}(P_{G(Z)},P_Y),$$
 $$G_n = \arg\min_{G\in\cG} W_{2, 2\sigma^2}(P_{G(Z)},Q_Y^n), $$ where $Q_Y^n$ is the empirical distribution of $n$ i.i.d. samples from $P_Y$ it holds that
	\begin{align*}
	\EE \bigl[W_{2,2\sigma^2}(\PP_{G_n(Z)},\PP_Y)- W_{2,2\sigma^2}(\PP_{G^*(Z)},\PP_Y)\bigr]
    &\leq C_d\sigma^2 n^{-1/2}\bigl(1+(\tau^2(1+\sigma(X)/\sigma)^2)^{\lceil 5d/4\rceil+3}\bigr),
	\end{align*}
	where $\tau = \max\{L\sigma(Z)/\sigma(X),1\}$ and $C_d$ is a dimension dependent constant.\label{thm:generalization}
 \end{thm}
The theorems show that the excess risk diminishes at the parametric rate (of order $1/\sqrt{n}$), which breaks the curse of dimensionality (convergence of order $n^{-\Omega(1/d)}$), often attributed to GANs.  We also observe that the excess risk is approximately linear in $\sigma^2$, the privatization noise scale, beyond a certain threshold ($\sigma^2>\sigma(X)^2)$). This implies that convergence for larger $\sigma^2$, corresponding to higher privacy, can be achieved by increasing the number of samples $n$.
 
 The above results show that the value of 
 the loss function under the empirical solution $G_n$ converges to the value of the loss function under the population solution $G^*$. However, this result does not directly relate $P_{G_n(Z)}$ to $P_X$. Next, we use Theorem~\ref{thm:generalization_gauss} and Corollary~\ref{col:KL_bound} to  upper bound the
 smoothed KL-divergence between $P_X$ and $P_{G_n(Z)}$.
\begin{col}\label{col:KL_bound}
If the target distribution can be generated, that is $P_X\in \{P_{G(Z)}\mid G\in\cG\},$ then
\begin{itemize}
    \item under the conditions of Corollary \ref{col:generalization_lap} one has
    \begin{align}
        \EE \bigl[D_{KL}(P_Y\|P_{G^{n}(Z)+N})\bigr]
        &\leq  \frac{4}{\sqrt{n}}e^{4dD/\sigma},\nonumber
    \end{align}
    where $N\sim f_N$ is the privatization noise of the Laplace mechanism, $f_N(x)\propto e^{-\|x\|_1/\sigma}$
    \item under the conditions of Theorem \ref{thm:generalization_gauss}  
        \begin{align}
        % \label{eq:col_kl_bound}
            \EE \bigl[D_{KL}(P_Y\|P_{G^{n}(Z)+N}\bigr]
            &\leq C_d n^{-1/2}\bigl(1+(\tau^2(1+\sigma(X)/\sigma)^2)^{\lceil 5d/4\rceil+3}\bigr) \nonumber
        \end{align}
\end{itemize}
\end{col}

Note that the parametric convergence of the smoothed KL-divergence results in the convergence of the smoothed Wasserstein distance \cite{goldfeld2020convergence}, %equation 21,
which is, in turn, a distance metrizing weak convergence similar to $W_p.$
% Note that the smoothed $p$-Wasserstein distance metrizes weak convergence in the space of distributions with bounded $p$-th moments similar to non-smoothed $p$-Wasserstein distance.

We note that known results on the sample complexity of entropic optimal transport %that estimate the complexity of the set of dual potentials 
are either only applicable to $c(x,y) \propto \|x-y\|_2^2$ \cite{mena2019statistical} or require the cost function to be $\infty$-differentiable \cite{genevay2019sample,luise2020generalization}, both of which assumptions do not apply to the Laplace setting. \cite{stromme2023minimum} proves sample complexity bounds for entropic optimal transport but with bounded cost, and most importantly, extending their result to the GAN setting where one of the distributions depends on the sample would require restricting the set of the generators to have a small complexity (VC-dimension or Rademacher complexity for example), which is not common in practice and is hard to compute. Our result, on the other hand, does not depend on the complexity of the set of generators, provided that their output has a bounded norm. \textcolor{black}{We achieve this by bounding the Rademacher complexity of the set of one of the dual potentials. Similar to \cite{genevay2019sample}, we invoke the bound involving the RKHS norms, but instead of using the smoothness of the dual potentials and the RKHS of the Sobolev space (which requires at least $\lceil d/2\rceil$ continuous differentiability), we rely on the fact that in the case of Laplace distribution the privatization mechanism's pdf is a Mercer kernel, which allows us to use its RKHS norms without requiring smoothness. This technique can also be used to provide sample complexity bounds for entropic 1-Wasserstein distance.} Moreover, the properties of $\ell_1$ norm allow us to eliminate the dependence on the tails of the privatized data distribution, which is only sub-exponential and does not concentrate as well as sub-gaussian distributions. The proof of the theorem as well as all other proofs are deferred to section \ref{sec:proofs}.

\section{Experimental Results}
% \subsection{Experimental Setup}
We first describe the approach we used to privatize the data and train the GAN, and then present the experimental results.\\
{\bf Data Privatization}
We conduct experiments for both the Laplace and Gaussian data privatization mechanisms. %Since the privatization happens at the data level and we do not assume any prior information is known except an estimate of $\cX\supset \supp P_X$, 
We set $f(x) = \text{vec}(x),$ where $\text{vec}(x)$ -- is the vectorization of $x$, specifically if $x\in\RR^d$ is a vector then $f(x) = x$ and if it is a matrix $x\in \RR^{d_1\times d_2}$  then $f(x)\in\RR^{d}, d=d_1d_2$ in the definitions of the Laplace and Gaussian privatization mechanisms. 

For the Laplace mechanism and LDP budget $\epsilon$ we set the $\ell_1$ sensitivity to be $\Delta = \sup_{x,x'\in\cX}\|f(x)-f(x')\|_1$ and the noise scale $\sigma = \Delta/\epsilon,$ so $Y = f(X) + Z,$ where $Z_i\sim \text{Laplace}(0, \sigma)$ i.i.d. for each coordinate $i\in\{1,\ldots,d\}.$ Similarly, for the Gaussian mechanism and LDP budget $\epsilon$ we set the $\ell_2$ sensitivity to be $\Delta = \sup_{x\in\cX}\|f(x)-f(x')\|_2$ and the noise scale $\sigma$ thet satisfies \eqref{eq:gaussian_eps} so $Y = f(X) + Z,$ where $Z_i\sim \cN(0, \sigma^2)$ i.i.d. for each coordinate $i\in\{1,\ldots,d\}.$ 
% For the Laplace mechanism, we project the data onto an $\ell_1$ ball, add Laplace noise with scale $\Delta/\eps$ for $\eps$-LDP and add it to the training data, {%\color{blue} 
% where $\Delta$ is the $\ell_1$ sensitivity of the training data and is equal to the radius of the $\ell_1$ ball.
% On the other hand, for the Gaussian mechanism, we project the data onto an $\ell_2$ ball, add Gaussian noise with variance $\sigma^2$ calculated from \eqref{eq:gaussian_eps} for $(\eps,\delta)$-LDP, and add it to the training data {%\color{blue} 
% (similarly with the $\ell_2$ sensitivity $\Delta$ being set to the radius of the $\ell_2$ ball.)}. }

%Then we proceed with training the model.\\
{\bf GAN training}
For training the Sinkhorn GAN we follow the work of \cite{genevay2018learning} by using Sinkhorn-Knopp algorithm \cite{flamary2021pot} to approximate the optimal transport plan $\pi$ in \eqref{eq:Wasserstein_reg_obj} from the mini-batches of size $b$ both for the generated and privatized training data. The algorithm is stated here for completeness, where $\theta$ stands for the parameter of the Generator, i.e. $\cG = \{G_{\theta}:\cZ\to\cX\mid \theta\in \Theta\}$.
\begin{algorithm}[htbp]
    \newcommand{\eqdef}{:=}
	    \caption{\label{alg:EWGAN}Training GAN with $W_{p, p\sigma^p}$}
	 \begin{algorithmic}
	 	\Require $\theta_0$, $\cD=\{y_i\}_{i=1}^n$ (the privatized training data), $b$ (batch size), $L$ (number of Sinkhorn iterations), $\alpha$ (learning rate), $c(\cdot,\cdot)$ (the cost function)
		\Ensure $\theta$
	\State $\theta  \leftarrow \theta_0$
	\For{$t = 1,2,\dots$}
        \State Sample $\{y_i\}_{i=1}^b$ i.i.d. from the dataset $\cD$, $Q_Y^b \eqdef \frac1b\sum_{i=1}^b\delta_{y_i}$
        \State Sample $\{z_i\}_{i=1}^{b} \overset{\text{i.i.d}}{\sim} P_\cZ$, $Q_G^b \eqdef \frac1b\sum_{i=1}^b\delta_{G_\theta(z_i)}$%, $\forall i: x_i\eqdef G_\theta(z_i)$
        \State Calculate the optimal transport plan for $S_c(Q_G^b,Q_Y^b)$ with $L$ Sinkhorn-Knopp steps\\
        $\pi \approx \underset{\pi\in\Pi(Q_G^b,Q_Y^b)}{\arg\min}\mathbb{E}_{({U},Y)\sim\pi}[c({U},Y)] + I_{\pi}({U)}, Y)$
        \State $C_{ij} \leftarrow c(y_i,G_\theta(z_j))$ for $i, j=1,\ldots, b$
		\State $g_t \leftarrow  \nabla_\theta\left\langle \pi, C\right\rangle$
		\State $\theta \leftarrow \theta - \alpha g_t$.
	\EndFor 
	\end{algorithmic}
	\end{algorithm}
   {\bf Dataset and architecture} We train our models on synthetic data as well as MNIST data \cite{lecun1998mnist}, consisting of 60000 grayscale images of handwritten digits. We do not use the labels to mimic a fully unsupervised training scenario. The generator model for MNIST is DCGAN from \cite{radford2015unsupervised} with latent space dimension $100.$ All the losses were used in the primal formulations \eqref{eq:Wasserstein},\eqref{eq:Wasserstein_reg} with optimization over the coupling matrix.

{%\color{blue}
\begin{remark}
    Note that since the DP noise is added to the training data, even if the training algorithm is an iterative process, the final privacy guarantee does not depend on (1) the number of rounds and (2) the specific privacy accountant or composition theorem used.
\end{remark}
}
   
\subsection{Synthetic Data}\label{sec:synthetic_data}
We first test our method on synthetic data. In this experiment, we sample data uniformly from a two-dimensional manifold shaped as a half-circle of radius 1 and we assume the support $\cX$ is known to be the half circle, so that the $\ell_1$ sensitivity is $2$ and the $\ell_2$ sensitivity is $\sqrt{2}$. The $400000$ sampled points are privatized with Laplace or Gaussian noise and then the Entropy-regularized Wasserstein GAN for the corresponding $p$ is trained on the privatized data using algorithm \ref{alg:EWGAN}. We used 2-dimensional latent noise uniform in $[-1,1]^2$ and a small Neural Network with 2 hidden layers and 256 neurons on each hidden layer. We trained it with batch gradient descent using RMSprop optimizer with a learning rate of $10^{-4}.$ The entropy-regularized Wasserstein distance was calculated with geomloss library \cite{peyre2019computational} for the full dataset  in each iteration ($b=n$ in algorithm \ref{alg:EWGAN}) with scaling parameter set to $0.99$. Figure \ref{fig:synthetic_halfcircle} shows the learned manifolds with data privatized with Laplace mechanism $\eps = 5$ and entropic $1$-Wasserstein loss \eqref{eq:Wasserstein_reg_obj} (top) and with data privatized with the Gaussian mechanism $\eps = 5, \delta=10^{-4}$ and entropic $2$-Wasserstein loss \eqref{eq:Wasserstein_reg_obj} (bottom). We note that for both the Laplace and Gaussian mechanisms entropic regularization allows to recover the original domain of the data (columns (a) and (b)), even-though noise in both cases appears to be large enough to completely obfuscate the data domain (column (b)). Without regularization (column (c)), the model generates the privatized distribution and fails to recover the original domain. 
In figures \ref{fig:synthetic_ellipsis_rectangle} we provide the results for ellipsis and rectangular-shaped manifolds, which show similar behavior.
\begin{figure}[htbp]
\newcommand{\w}{0.28}%0.325
\newcommand{\f}{1.2}
\newcommand{\lt}{\!\!\!\!\!\!\!\!\!}
\begin{subfigure}[t]{\w\linewidth}
\vspace{-3ex}
\lt\includegraphics[width=\f\linewidth]{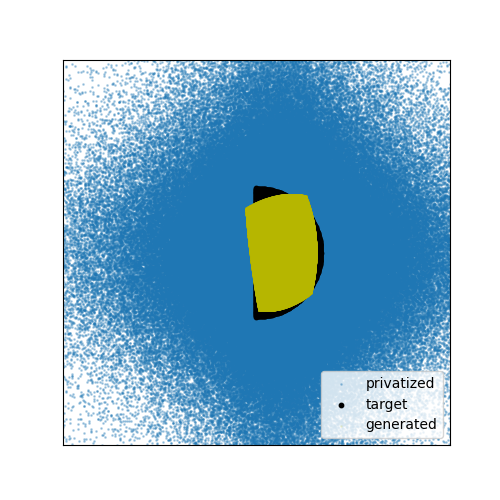}
% \caption{$\eps=50$}
\end{subfigure}\hfill
\begin{subfigure}[t]{\w\linewidth}
\vspace{-3ex}
\lt\includegraphics[width=\f\linewidth]{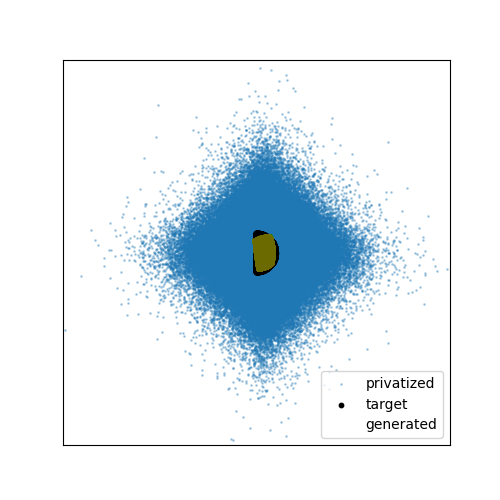}
% \caption{$\eps=40$}
\end{subfigure}\hfill
\begin{subfigure}[t]{\w\linewidth}
\vspace{-3ex}
\lt\includegraphics[width=\f\linewidth]{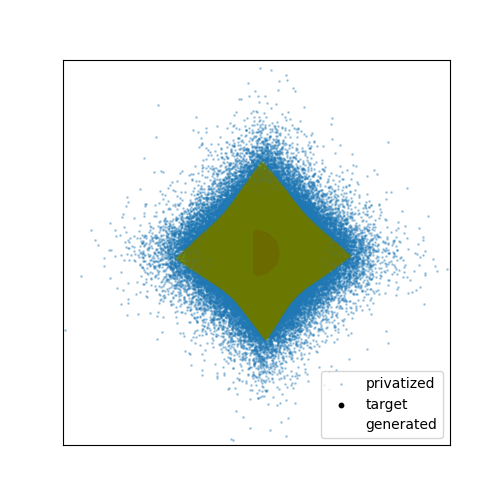}
% \caption{$\eps=40$}
\end{subfigure}\\[-4.2ex]
\begin{subfigure}[t]{\w\linewidth}
\lt\includegraphics[width=\f\linewidth]{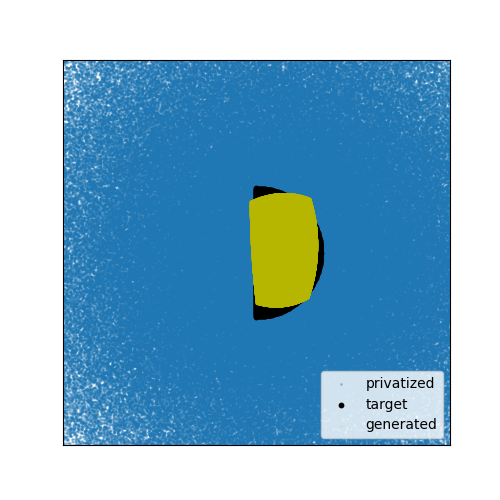}\\[-6ex]
\caption{\footnotesize{entropic~$p$-WGAN~(eq.~\eqref{eq:Wasserstein_reg_obj})}}\label{subfig:entropic}
\end{subfigure}\hfill
\begin{subfigure}[t]{\w\linewidth}
\lt\includegraphics[width=\f\linewidth]{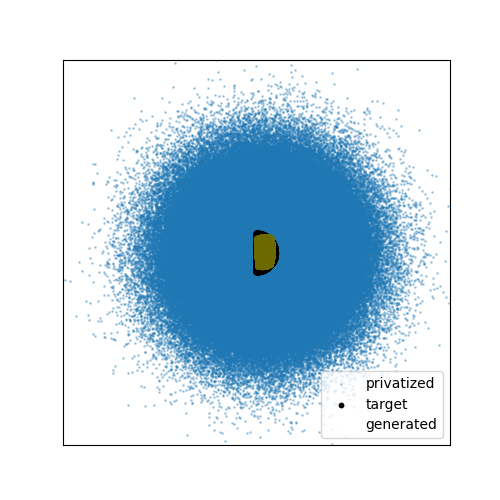}\\[-6ex]    
\caption{\footnotesize{entropic~$p$-WGAN~(eq.~\eqref{eq:Wasserstein_reg_obj}),~enlarged}}
\label{subfig:enlarged}
\end{subfigure}\hfill
\begin{subfigure}[t]{\w\linewidth}
\lt\includegraphics[width=\f\linewidth]{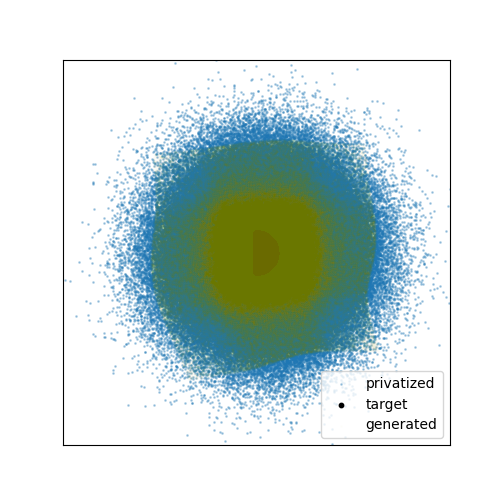}\\[-6ex]
\caption{\footnotesize{unregularized $p$-WGAN~(eq.~\eqref{eq:GAN})}}\label{subfig:unreg_halfcirc}
\end{subfigure}
\caption{Learning data from the half-circle-shaped manifold privatized with Laplace mechanism $\eps = 5$ (top) and Gaussian mechanism $\eps = 5, \delta=10^{-4}$ (bottom). Columns (a) and (b) show the manifold learned with entropic $p$-Wasserstein GAN \eqref{eq:Wasserstein_reg_obj}, and column (c) shows the manifold learned with unregularized $p$-Wasserstein GAN \eqref{eq:GAN_general}. Note $p=1$ for Laplace mechanism and $p=2$ for Gaussian mechanism.}
\label{fig:synthetic_halfcircle}
\end{figure}

We note that the $\eps=5$ local differential privacy guarantee obtained with the Laplace mechanism can be translated to a central privacy guarantee by leveraging  privacy amplification by shuffling. 
Since the distribution of the output of algorithm \ref{alg:EWGAN} does not depend on the order of the samples in the privatized data (due to the random sampling) and because the local privatization mechanism only depends on the data at the client and no auxiliary input, one can assume that the data is shuffled before privatization, which allows to apply \cite[Theorem 3.2]{feldman2023stronger}, resulting in a $(\eps_c=0.353, \delta_c=10^{-6})$ central differential privacy guarantee for the Laplace mechanism. 
\begin{figure}[htbp]
\newcommand{\w}{0.28}%0.325
\newcommand{\f}{1.2}
\newcommand{\lt}{\!\!\!\!\!\!\!\!\!}
\begin{subfigure}[t]{\w\linewidth}
\vspace{-3ex}
\lt\includegraphics[width=\f\linewidth]{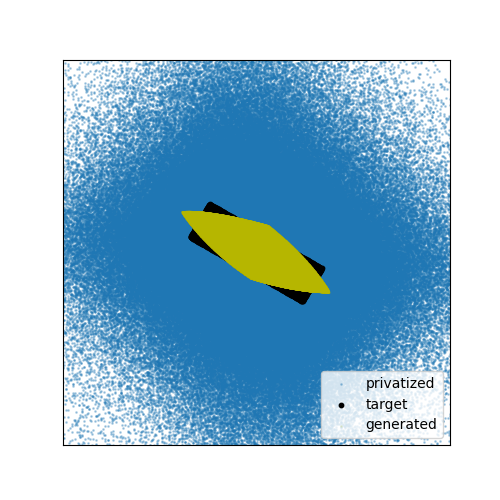}
% \caption{$\eps=50$}
\end{subfigure}\hfill
\begin{subfigure}[t]{\w\linewidth}
\vspace{-3ex}
\lt\includegraphics[width=\f\linewidth]{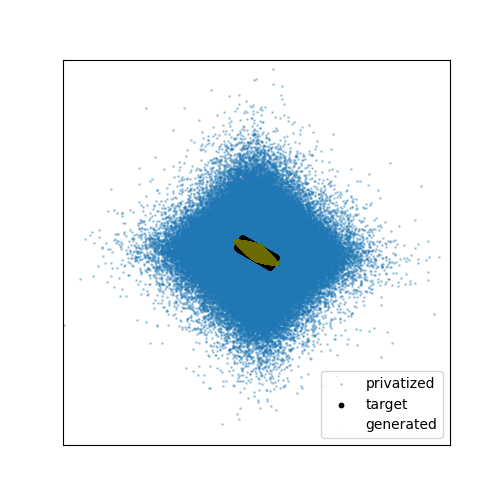}
% \caption{$\eps=40$}
\end{subfigure}\hfill
\begin{subfigure}[t]{\w\linewidth}
\vspace{-3ex}
\lt\includegraphics[width=\f\linewidth]{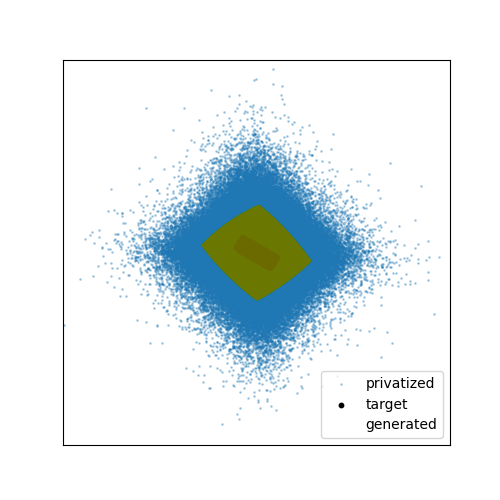}
% \caption{$\eps=40$}
\end{subfigure}\\[-4.2ex]
\begin{subfigure}[t]{\w\linewidth}
\lt\includegraphics[width=\f\linewidth]{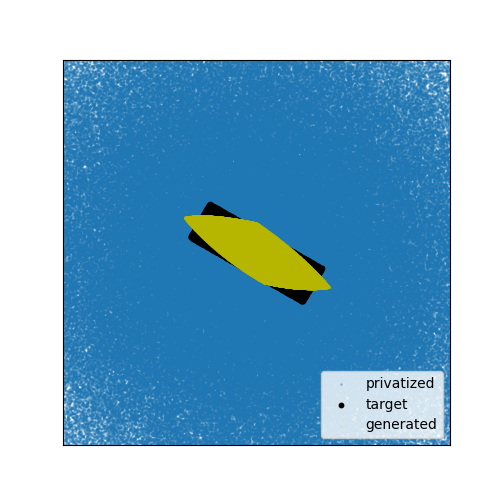}
\end{subfigure}\hfill
\begin{subfigure}[t]{\w\linewidth}
\lt\includegraphics[width=\f\linewidth]{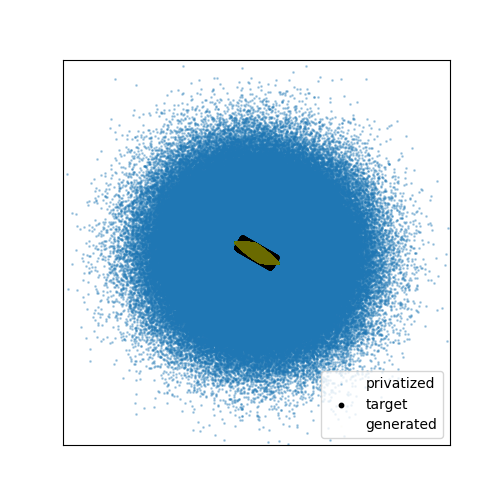}
\end{subfigure}\hfill
\begin{subfigure}[t]{\w\linewidth}
\lt\includegraphics[width=\f\linewidth]{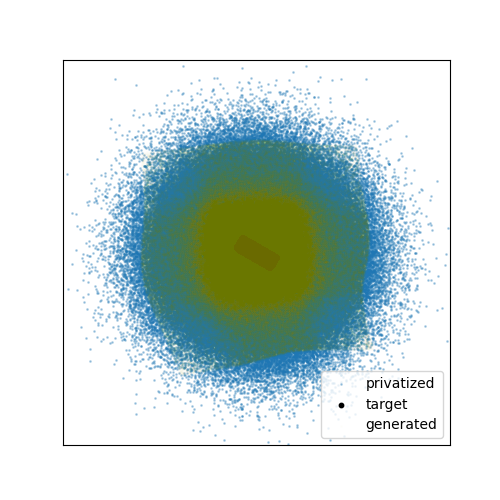}
\end{subfigure}\\[-4.2ex]
\begin{subfigure}[t]{\w\linewidth}
\lt\includegraphics[width=\f\linewidth]{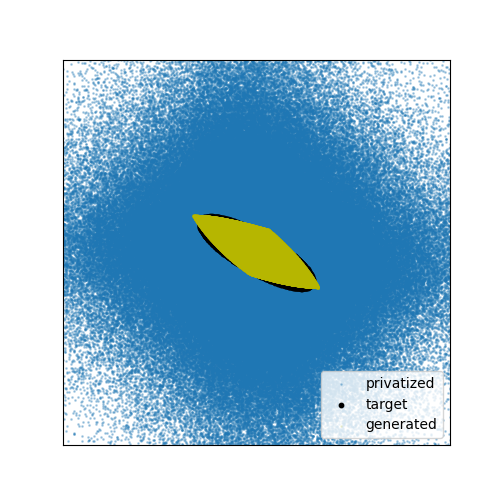}
% \caption{$\eps=50$}
\end{subfigure}\hfill
\begin{subfigure}[t]{\w\linewidth}
\lt\includegraphics[width=\f\linewidth]{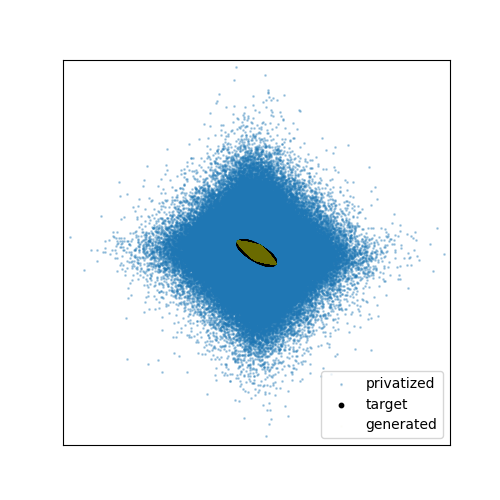}
% \caption{$\eps=40$}
\end{subfigure}\hfill
\begin{subfigure}[t]{\w\linewidth}
\lt\includegraphics[width=\f\linewidth]{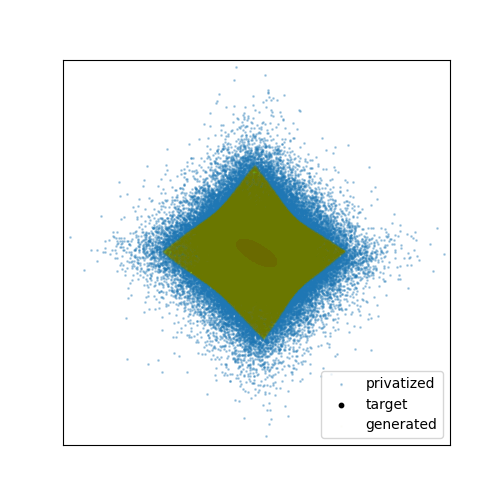}
% \caption{$\eps=40$}
\end{subfigure}\\[-4.2ex]
\begin{subfigure}[t]{\w\linewidth}
\lt\includegraphics[width=\f\linewidth]{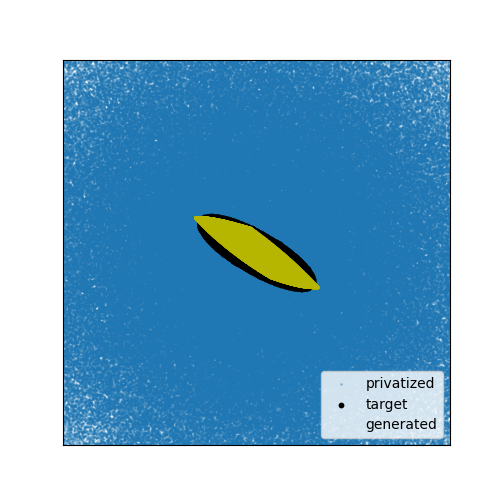}\\[-6ex]\caption{\footnotesize{entropic~$p$-WGAN~(eq.~\eqref{eq:Wasserstein_reg_obj})}}
\end{subfigure}\hfill
\begin{subfigure}[t]{\w\linewidth}
\lt\includegraphics[width=\f\linewidth]{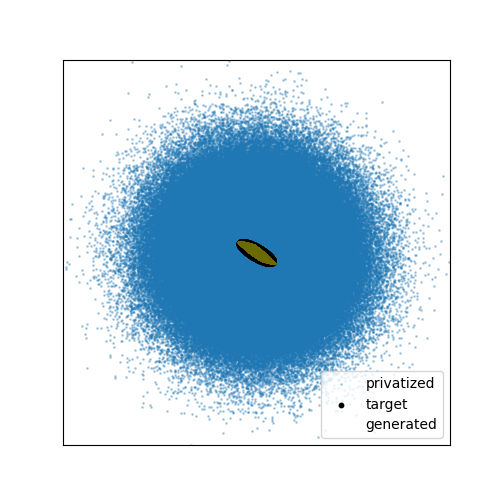}\\[-6ex]    
\caption{\footnotesize{entropic~$p$-WGAN~(eq.~\eqref{eq:Wasserstein_reg_obj}),~enlarged}}
\end{subfigure}\hfill
\begin{subfigure}[t]{\w\linewidth}
\lt\includegraphics[width=\f\linewidth]{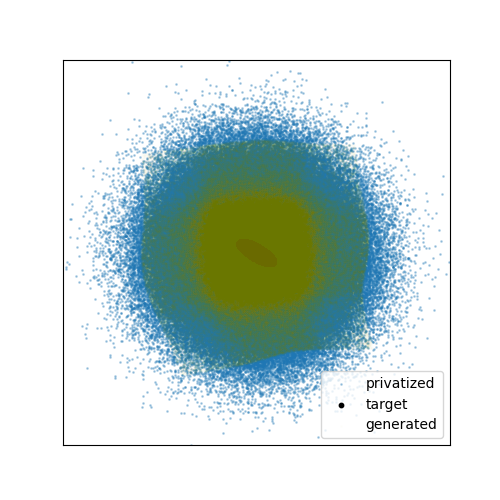}\\[-6ex]\caption{\footnotesize{unregularized $p$-WGAN~(eq.~\eqref{eq:GAN})}}\label{subfig:unreg_ellips}
\end{subfigure}
\caption{Learning data from the rectangular (top 2 rows) and ellipsis-shaped (bottom 2 rows) manifolds privatized with Laplace mechanism $\eps = 5$ (rows 1,3) and Gaussian mechanism $\eps = 5, \delta=10^{-4}$ (rows 2,4). Columns (a) and (b) show the manifold learned with entropic $p$-Wasserstein GAN \eqref{eq:Wasserstein_reg_obj}, and column (c) shows the manifold learned with unregularized $p$-Wasserstein GAN \eqref{eq:GAN_general}. Note $p=1$ for Laplace mechanism and $p=2$ for Gaussian mechanism.}
\label{fig:synthetic_ellipsis_rectangle}
\end{figure}

\subsection{MNIST: Comparison with denoising} \label{sec:MNISTvsdenoiding}
We next provide our experimental results with MNIST \cite{lecun1998mnist} and DCGAN \cite{radford2015unsupervised} generator. 
% For the fair comparison, experiments in this section do not involve projecting onto the $\ell_1/\ell_2$ balls and the maximum radius of the ball where the images fall was used for sensitivity: since each pixel lies in $[-1,1],$ the maximum $\ell_1$ norm of the images is $28^2$ and the maximum $\ell_2$ norm is $28$. 
The pixel values of the images were rescaled to $[-1,1]$ leading to $\Delta = 2\times 28^2$  $\ell_1$ sensitivity and $\Delta=56$ $\ell_2$ sensitivity.

We use $100$-dimensional Uniform $[0,1]$ noise at the input to the generator ($\cP_z = \text{Unif}[0,1]^{100}$). In Figure \ref{fig:wavelet}, we show two raw samples from the MNIST dataset (column (a)) and the corresponding privatized images (column (b)). In column (c), we denoise the privatized images in the second column with wavelet transform \cite{mallat1999wavelet}; the results indicate that the wavelet transform can not be used to recover the images. Here, the wavelet transform parameters for denoising (the wavelet basis, the level and reconstruction thresholds) were optimized to minimize the average distance between the reconstructed and original image under the particular noise instance, thus providing better results than one would expect in a fully privatized setting. 
In  column (d), we used the noise2void \cite{krull2019noise2void} image denoising mechanism with parameters as  given in the paper and trained it on the whole dataset of privatized images, and showed that it also failed to reconstruct the images. These experiments suggest that the privatization noise is large enough to preclude reconstruction of the original images.
Column (e) shows samples generated by GANs trained with $p$-Wasserstein loss (without entropic regularization) that fails to learn from privatized data.  
% In the fifth column, we train a Wasserstein-GAN \cite{gulrajani2017improved} on the wavelet-denoised images, which also fails to provide useful samples. 
Finally, in (f) we provide samples obtained  by our method stated in algorithm \ref{alg:EWGAN} in Figure \ref{fig:wavelet}. Note that while the generated images by the entropic GAN are not perfect for the chosen privacy levels, the results do suggest that the model has learned to generate new images of digits. For training the entropic $p$-WGAN we used  $400$ Sinkhorn-Knopp iterations and the Adam optimizer with learning rate $10^{-4}$ for $100$ epochs. 

\begin{figure}[htbp]
\newcommand{\w}{0.15}
\begin{subfigure}[t]{\w\linewidth}
\includegraphics[width=\linewidth]{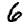}
% \caption{ figure}
\end{subfigure}\hfill
\begin{subfigure}[t]{\w\linewidth}
\includegraphics[width=\linewidth]{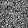}
% \caption{ figure}
\end{subfigure}\hfill
\begin{subfigure}[t]{\w\linewidth}
\includegraphics[width=\linewidth]{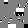}
% \caption{ figure}
\end{subfigure}\hfill
\begin{subfigure}[t]{\w\linewidth}
\includegraphics[width=\linewidth]{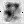}
\end{subfigure}\hfill
\begin{subfigure}[t]{\w\linewidth}
\includegraphics[height=\linewidth]{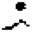}
% \caption{ figure}
\end{subfigure}\hfill
% \begin{subfigure}[t]{\w\linewidth}
% \includegraphics[height=\linewidth]{images/1samples/emd_l2_denoised_3.0eps.png}
% % \caption{ figure}
% \end{subfigure}\hfill
\begin{subfigure}[t]{\w\linewidth}
\includegraphics[height=\linewidth]{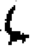}
\end{subfigure}%

\begin{subfigure}[t]{\w\linewidth}
\includegraphics[width=\linewidth]{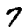}
\subcaption{}
% *{original}
\label{fig:wavelet_raw}
% \caption{ figure}
\end{subfigure}\hfill
\begin{subfigure}[t]{\w\linewidth}
\includegraphics[width=\linewidth]{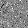}
\subcaption{}
% *{privatized}
\label{fig:wavelet_privatized}
% \caption{ figure}
\end{subfigure}\hfill
\begin{subfigure}[t]{\w\linewidth}
\includegraphics[width=\linewidth]{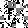}
\subcaption{}
% *{wavelet\\denoising}
% \caption{ figure}
\end{subfigure}\hfill
\begin{subfigure}[t]{\w\linewidth}
\includegraphics[width=\linewidth]{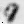}
\subcaption{}
% *{Noise2Void\\denoising}
\label{fig:wavelet_n2v}
\end{subfigure}\hfill
\begin{subfigure}[t]{\w\linewidth}
\includegraphics[height=\linewidth]{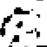}
\subcaption{}
% *{$p$-WGAN~\eqref{eq:GAN}}
\label{fig:wavelet_WGAN}
% \caption{ figure}
\end{subfigure}\hfill
% \begin{subfigure}[t]{\w\linewidth}
% \includegraphics[height=\linewidth]{images/1samples/emd_l1_denoised_4.0eps.png}
% \subcaption*{{\small denoising\\+$p$-WGAN}}
% % \caption{ figure}
% \end{subfigure}\hfill
\begin{subfigure}[t]{\w\linewidth}
\includegraphics[height=\linewidth]{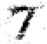}
\subcaption{}
% {Entropic $p$-WGAN~\eqref{eq:Wasserstein_reg_obj}}
\label{fig:wavelet_EWGAN}
\end{subfigure}
\caption{Learning with LDP privatized images: (a) raw images from MNIST dataset, (b) image samples privatized with Gaussian $(\eps, \delta) = (35,10^{-4})$ (top) and  Laplace mechanism $\eps=196$ (bottom), (c) images from the second column denoised with the wavelet transform and (d) the Noise2Void \cite{krull2019noise2void}, (e) random samples from the output distribution of the unregularized $p$-Wasserstein GAN \eqref{eq:GANemp}, and (f) generated samples from the entropy-regularized $p$-Wasserstein GAN trained on the privatized data  (picked to represent the same digit as in column (a))}\label{fig:wavelet}
\end{figure}
%%The results demonstrate that naive denoising with wavelet transform, which is a standard for image denoising, is unable to reconstruct the mnist images privatized with either Gaussian or Laplace noise at the chosen privatization level. Since the parameters for the wavelet transform were chosen based on the non-privatized images, we can think of  wavelet denoising + p-WGAN as only approximately private at the same privacy level. %it as an approximation of the best wavelet reconstruction function and the order of $\eps$ needed for reconstruction. 
%In contrast, the entropic p-WGAN generator learned with the privatized samples was able to learn the distribution far beyond the values of $\eps$ needed for the wavelet transform reconstruction.

% \subsection{Additional samples for comparing wavelet transform, GAN with no regularization and with entropic regularization}
In Figures \ref{tab:wavelet_lap} and \ref{tab:wavelet_gauss} we provide additional 200 uniformly sampled images from the Entropic $p$-WGAN trained on MNIST data privatized with the corresponding mechanism together with the privacy parameter. The setting of the experiment is the same as in section~\ref{sec:emporical_convergence}.

\begin{figure}[htbp]
    \newcommand{\adj}[1]{\raisebox{0.15\textwidth}[\height][\depth]{#1}}
    \newcommand{\vsp}{-1em}
    \centering
    % \vspace{-15em}
    % First row (Laplace Mechanism)
    \begin{minipage}{\linewidth}
        \adj{\rotatebox[origin=c]{90}{$\eps=1568$}}
        \begin{subfigure}[b]{0.3\textwidth}
            \includegraphics[width=\textwidth]{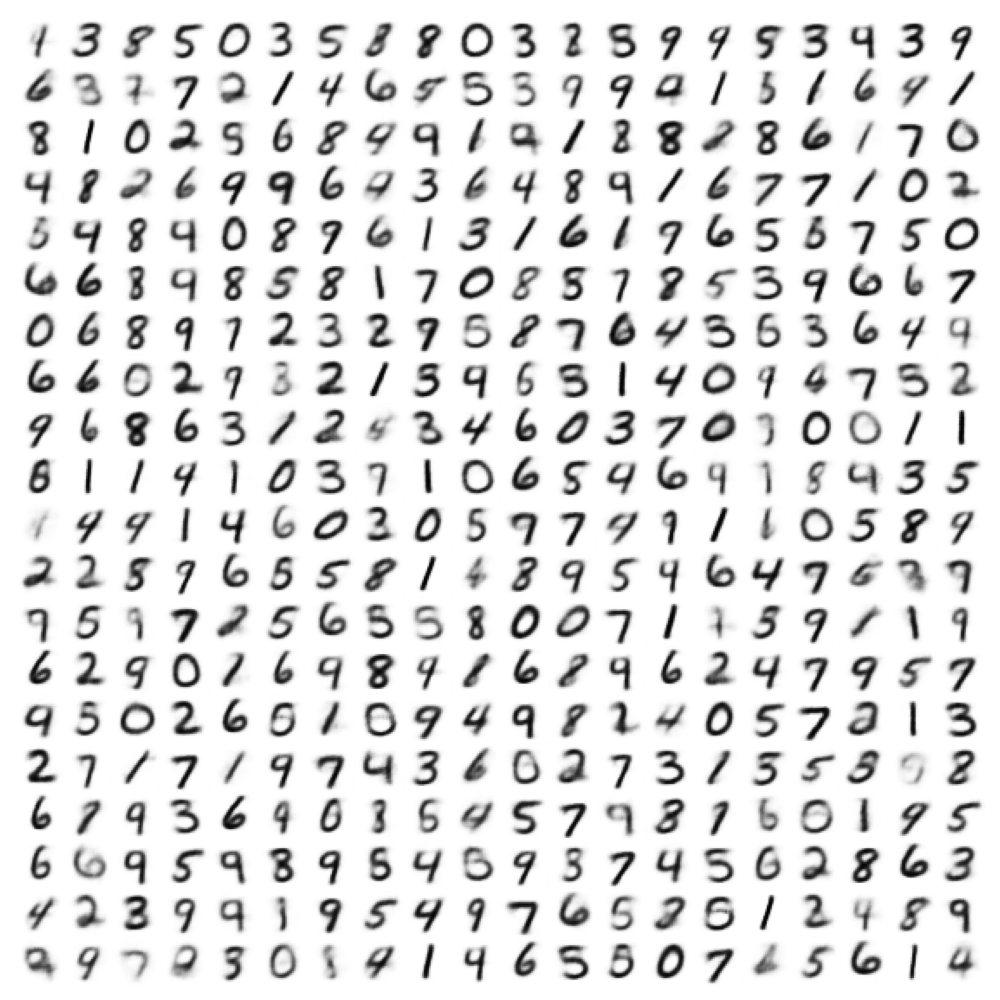}
        \end{subfigure}
        \begin{subfigure}[b]{0.3\textwidth}
            \includegraphics[width=\textwidth]{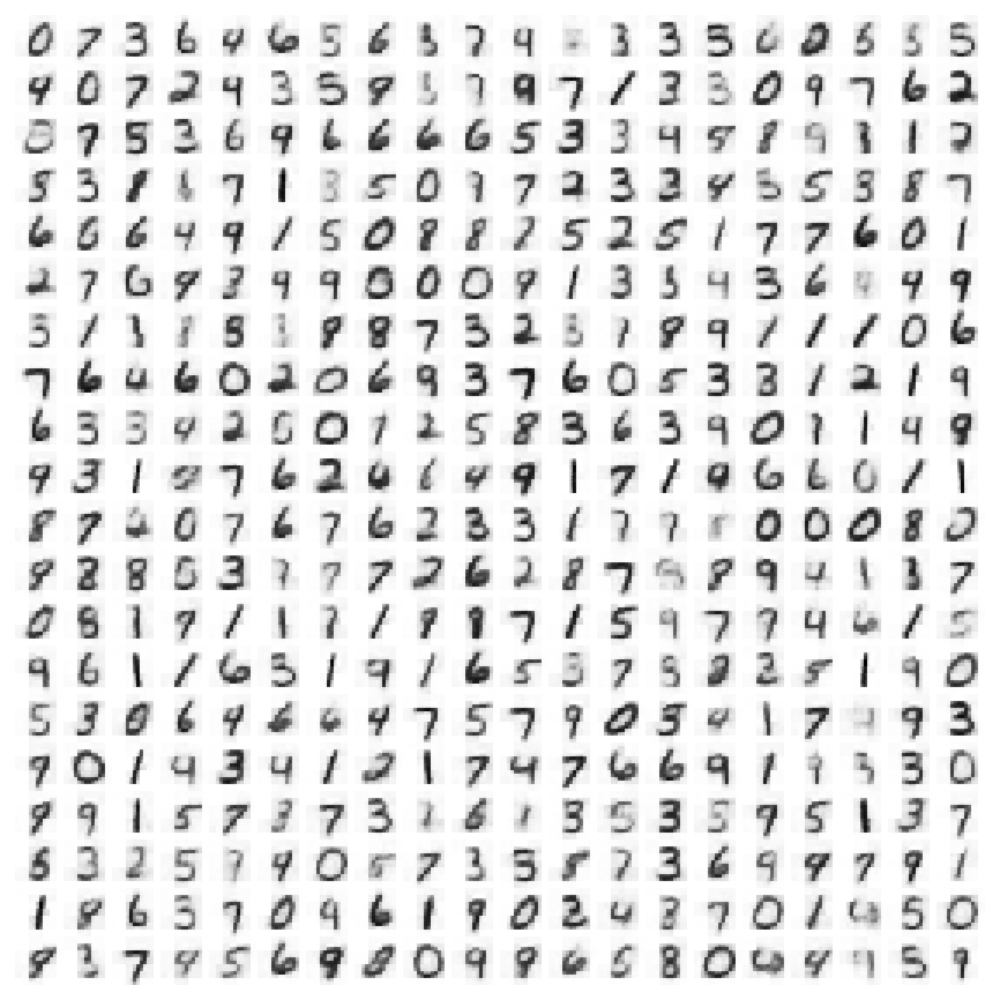}
        \end{subfigure}
        \begin{subfigure}[b]{0.3\textwidth}
            \includegraphics[width=\textwidth]{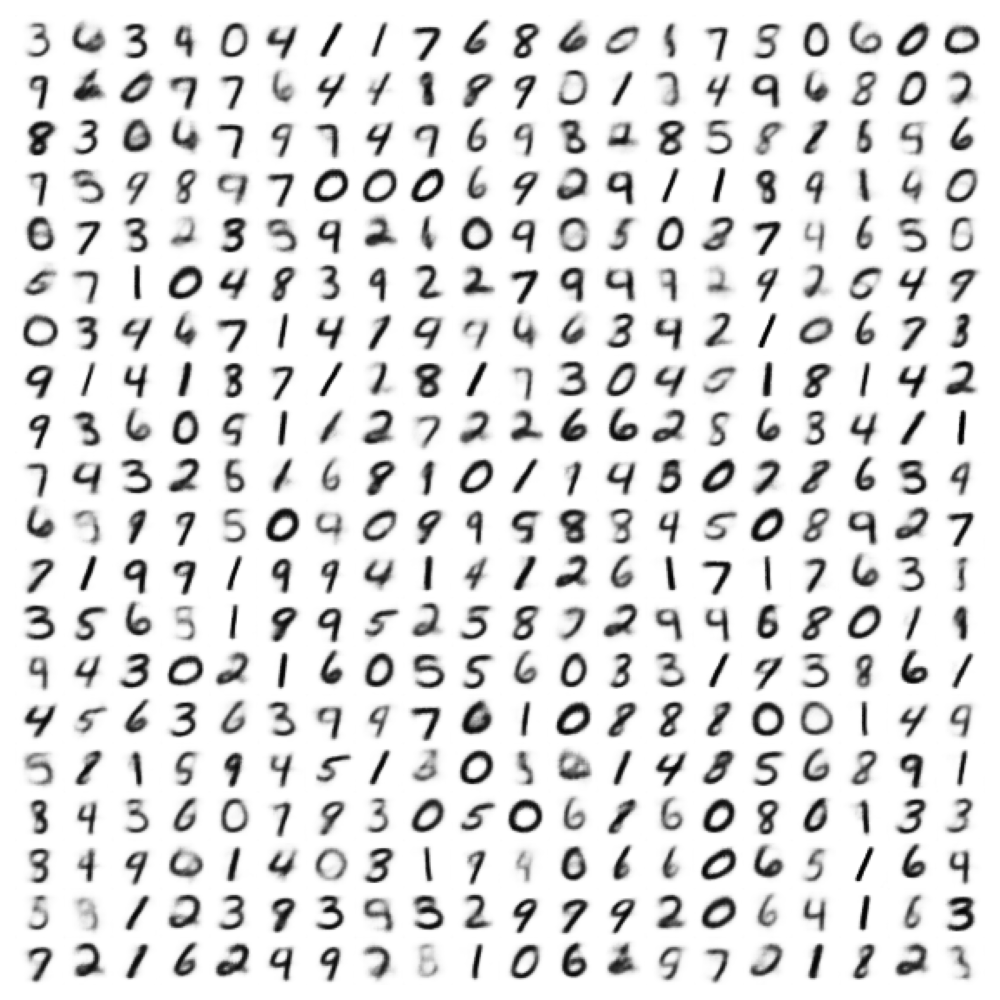}
        \end{subfigure}
    \vspace{\vsp}
    \end{minipage}
    % Second row (Laplace Mechanism)
    \begin{minipage}{\linewidth}
        \adj{\rotatebox[origin=c]{90}{$\eps=196$}}
        \begin{subfigure}[b]{0.3\textwidth}
            \includegraphics[width=\textwidth]{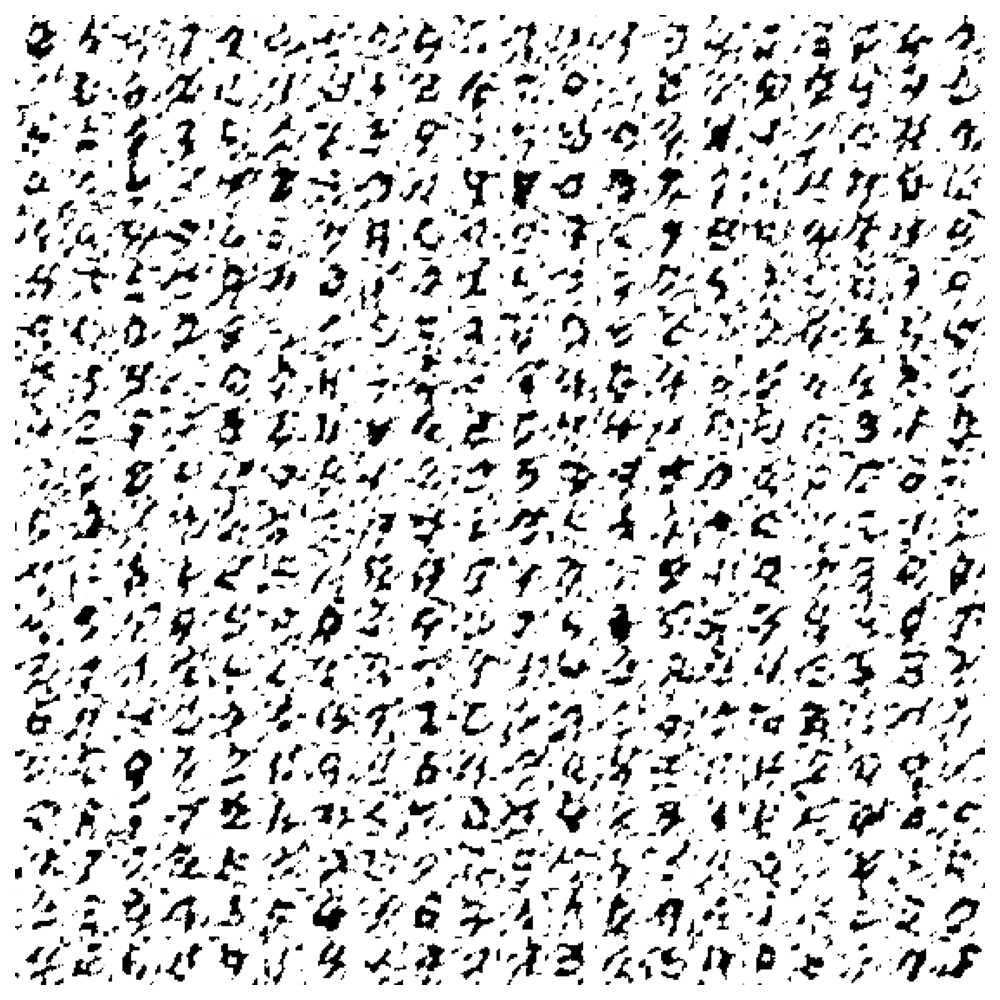}
        \end{subfigure}
        \begin{subfigure}[b]{0.3\textwidth}
            \includegraphics[width=\textwidth]{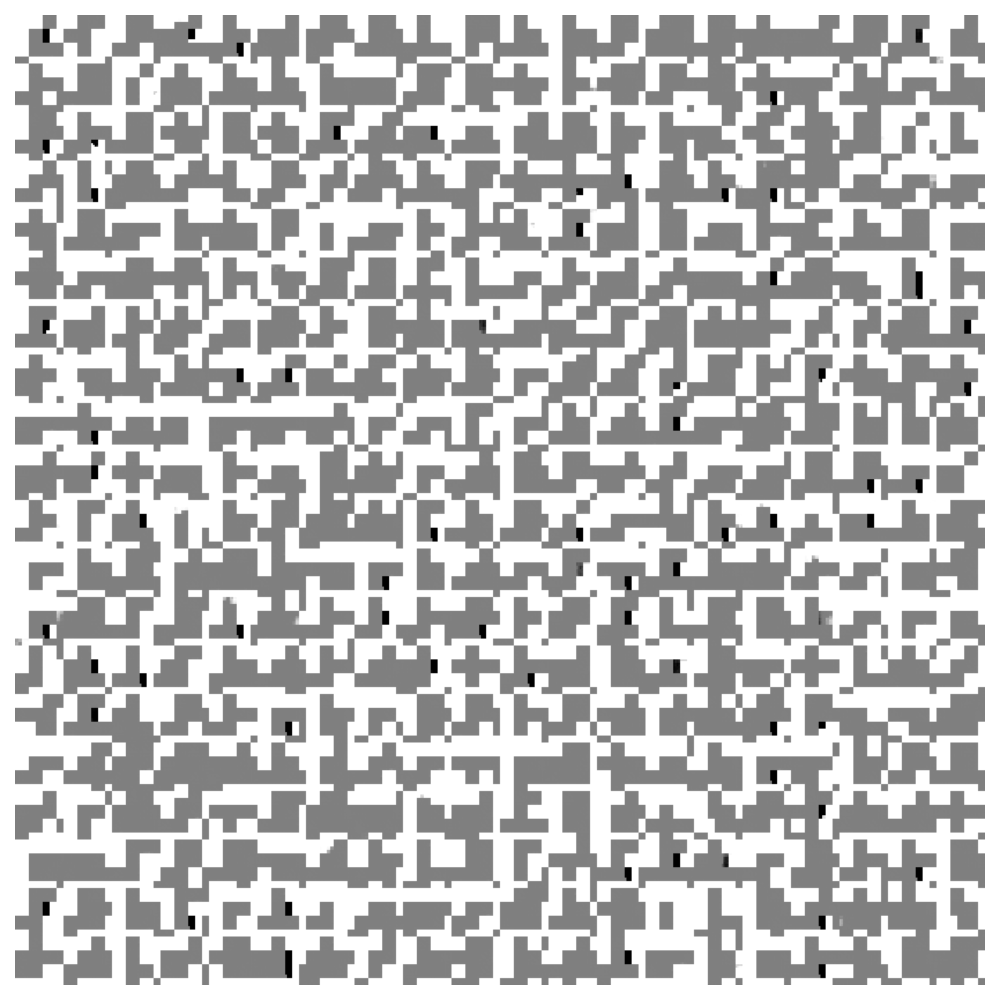}
        \end{subfigure}
        \begin{subfigure}[b]{0.3\textwidth}
            \includegraphics[width=\textwidth]{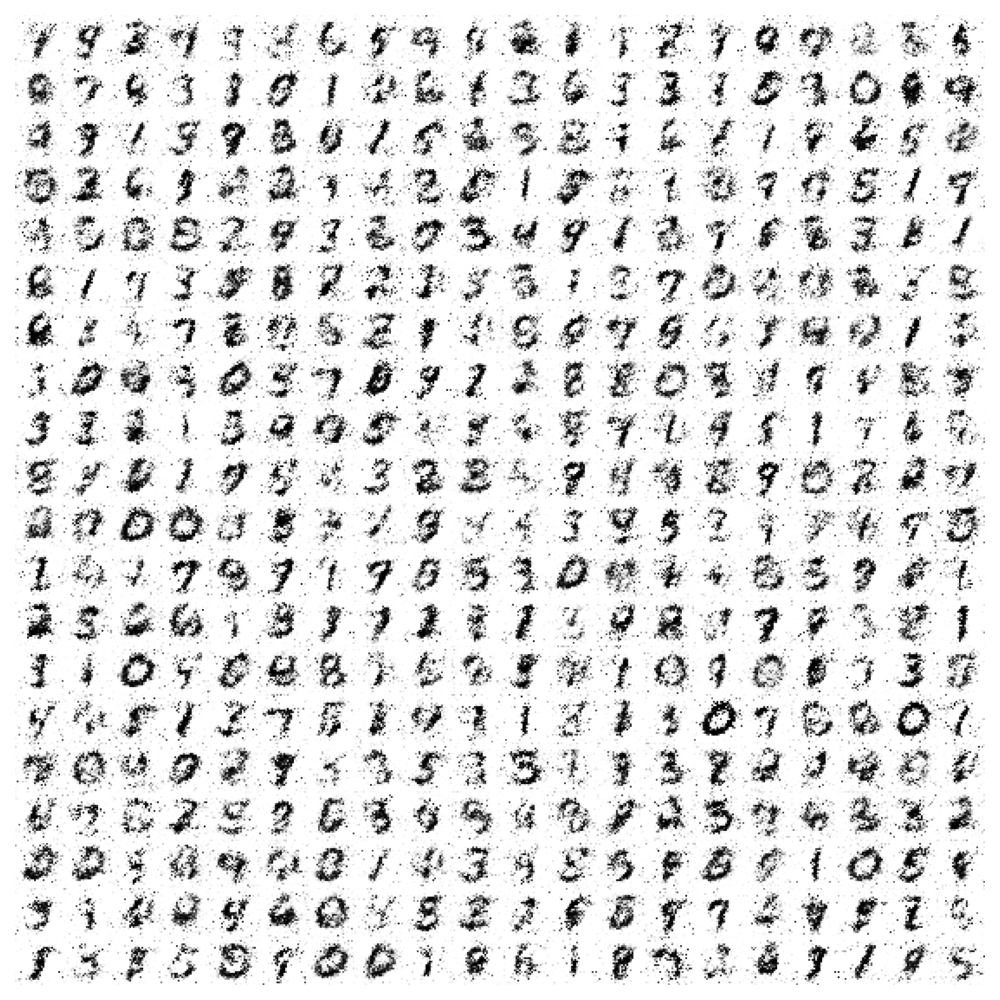}
        \end{subfigure}
    \vspace{\vsp}
    \end{minipage}
    \caption{Comparing Wavelet denoising and Entropic 1-WGAN the Laplace mechanism}\label{tab:wavelet_lap}
\end{figure}

\begin{figure}[htbp]
    \newcommand{\adj}[1]{\raisebox{0.15\textwidth}[\height][\depth]{#1}}
    \newcommand{\vsp}{-1em}
    \centering
    % \vspace{-15em}
    % First row (Laplace Mechanism)
    \begin{minipage}{\linewidth}
    \hspace{-0.5em}
        \adj{\rotatebox[origin=c]{90}{$\eps=1568,\delta=10^{-4}$}}
        \begin{subfigure}[b]{0.3\textwidth}
            \includegraphics[width=\textwidth]{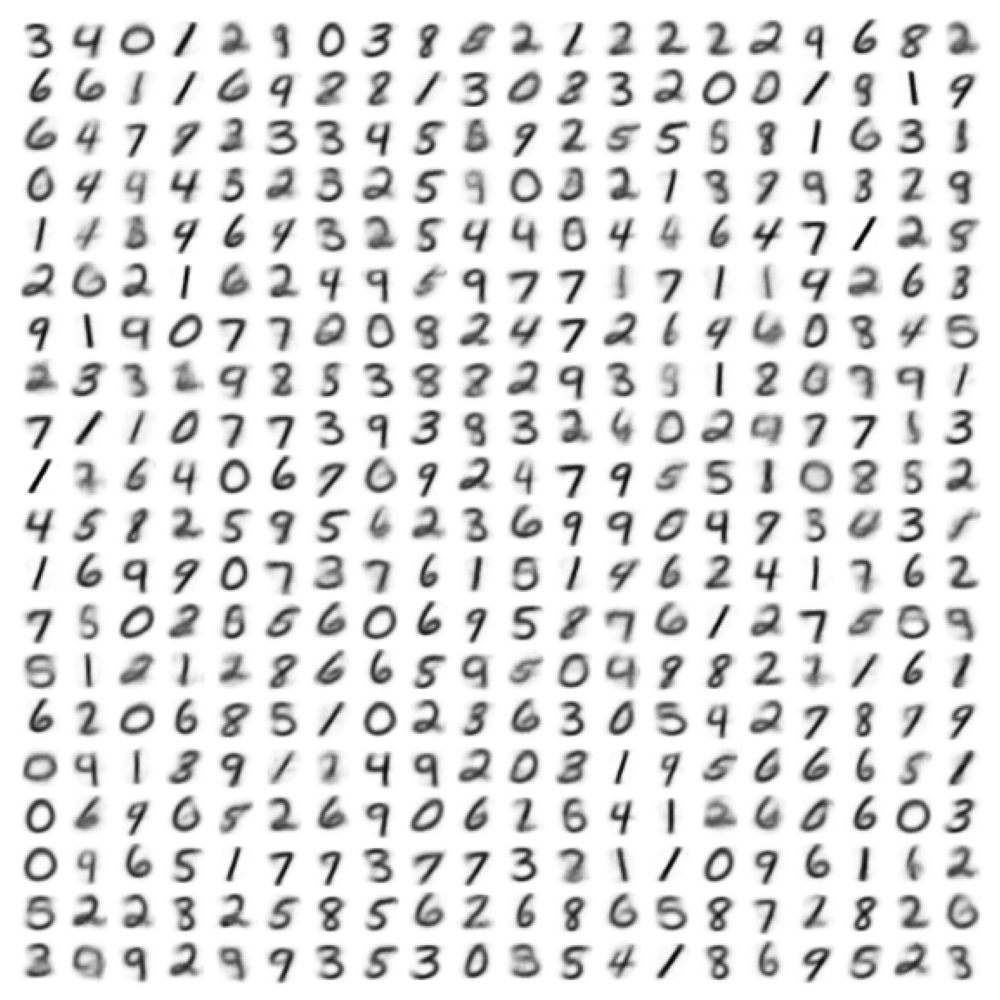}
        \end{subfigure}
        \begin{subfigure}[b]{0.3\textwidth}
            \includegraphics[width=\textwidth]{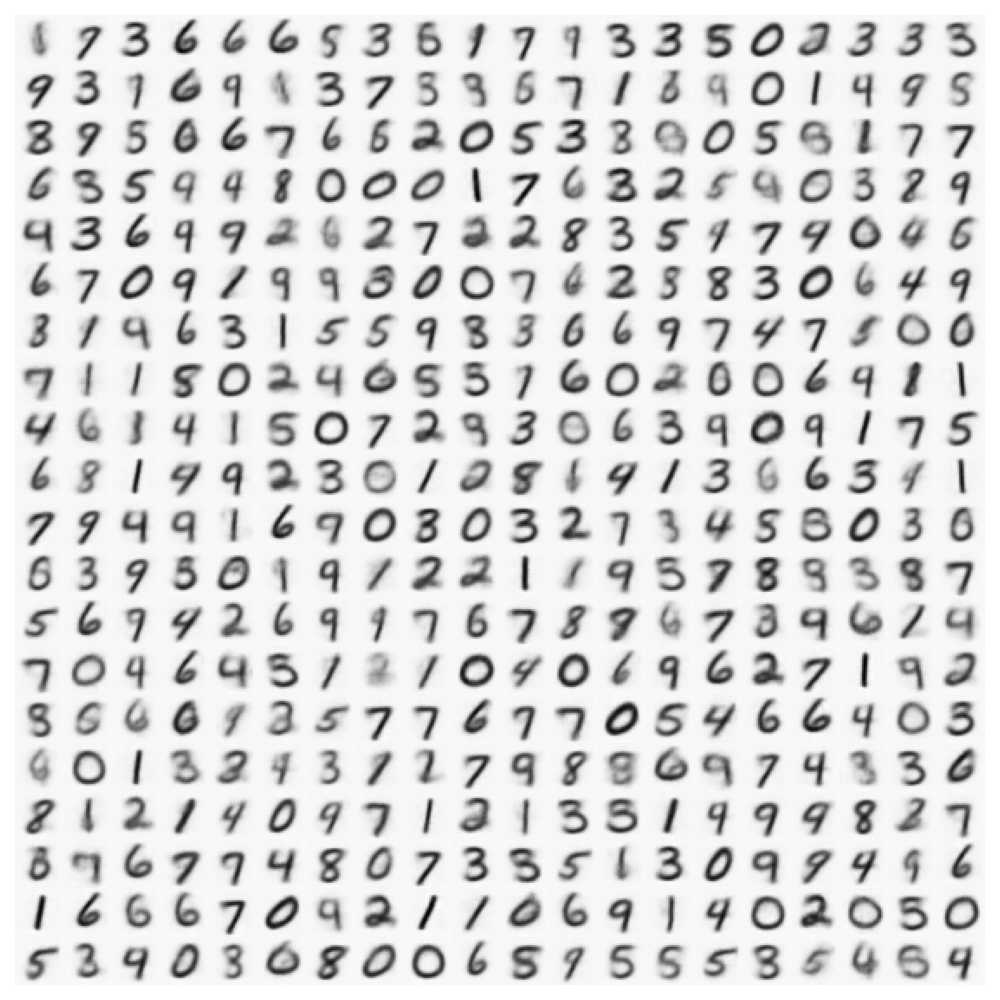}
        \end{subfigure}
        \begin{subfigure}[b]{0.3\textwidth}
            \includegraphics[width=\textwidth]{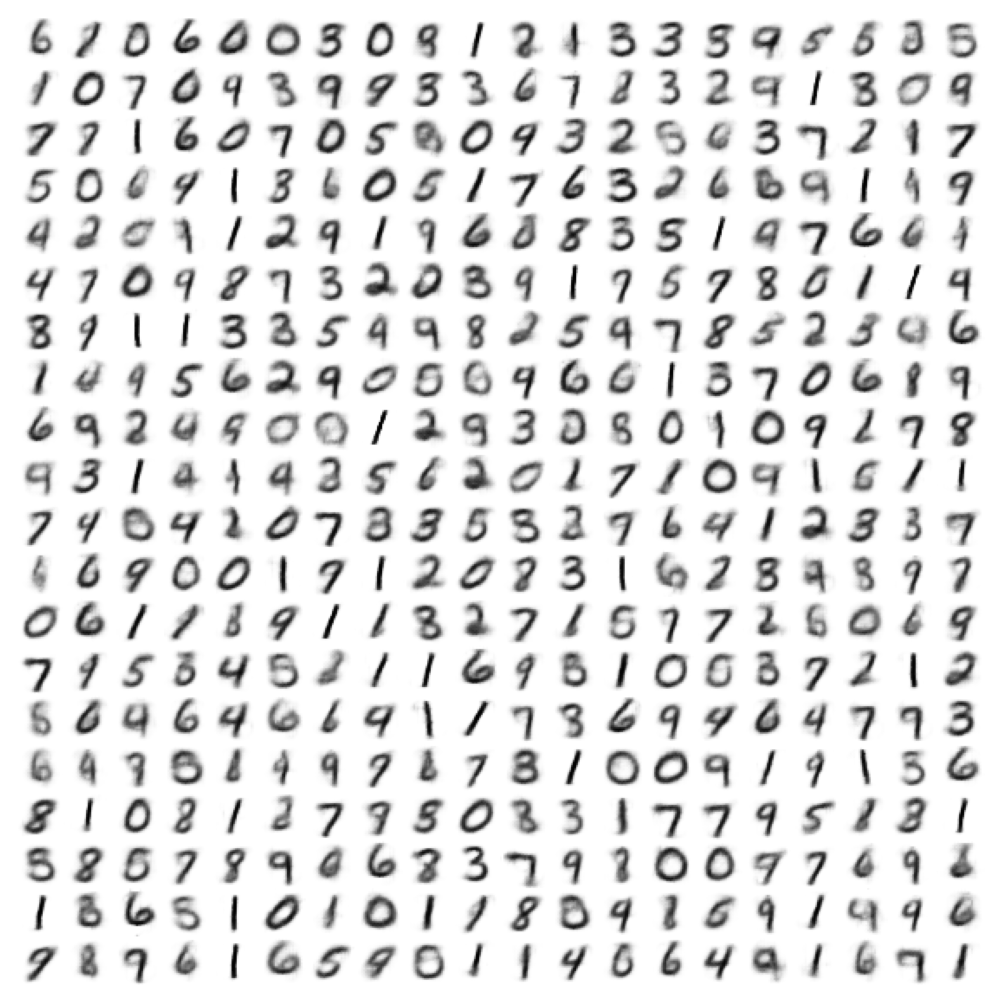}
        \end{subfigure}
    \vspace{-3.5em}
    \end{minipage}
    \begin{minipage}{\linewidth}
    \hspace{-0.5em}
        \raisebox{0.2\textwidth}[\height][\depth]{\rotatebox[origin=c]{90}{$\eps=41,\delta=10^{-4}$}}
        \begin{subfigure}[b]{0.3\textwidth}
            \includegraphics[width=\textwidth]{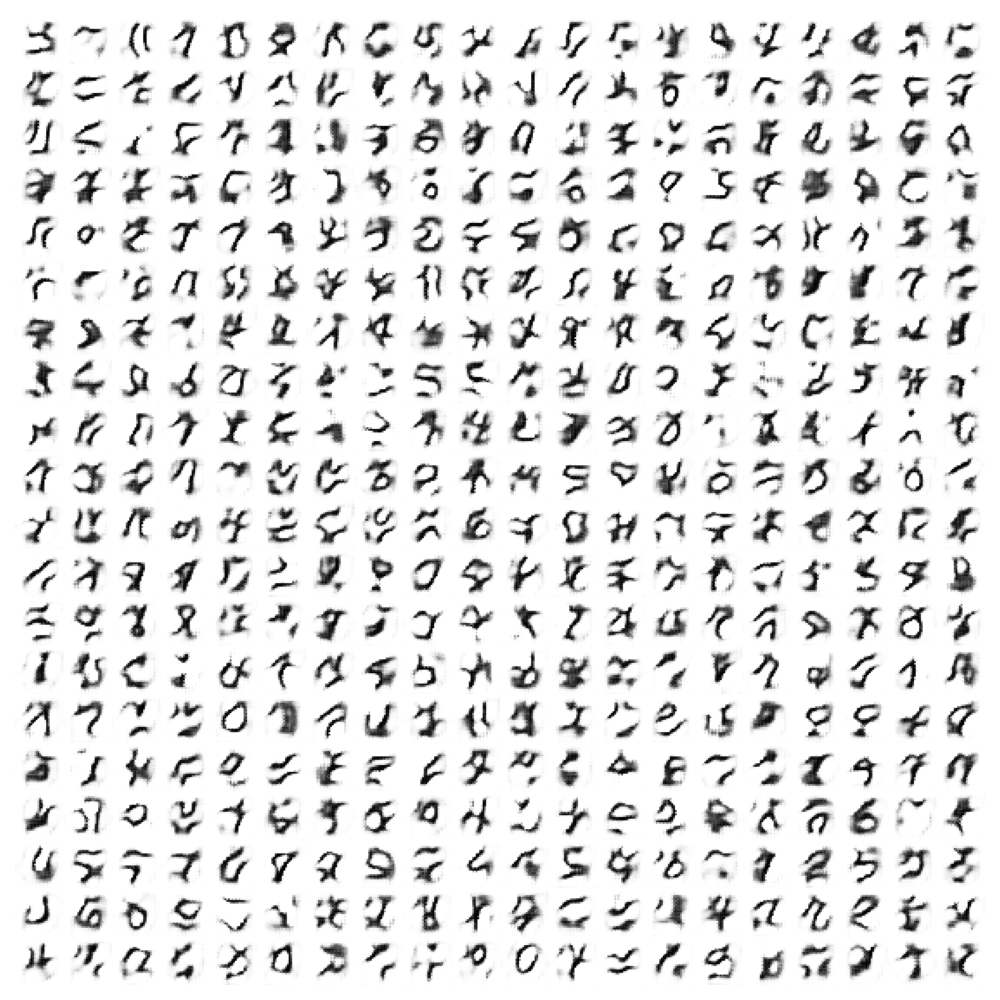}
            \subcaption*{unregularized~$p$-WGAN~\eqref{eq:GAN}}
        \end{subfigure}
        \begin{subfigure}[b]{0.3\textwidth}
            \includegraphics[width=\textwidth]{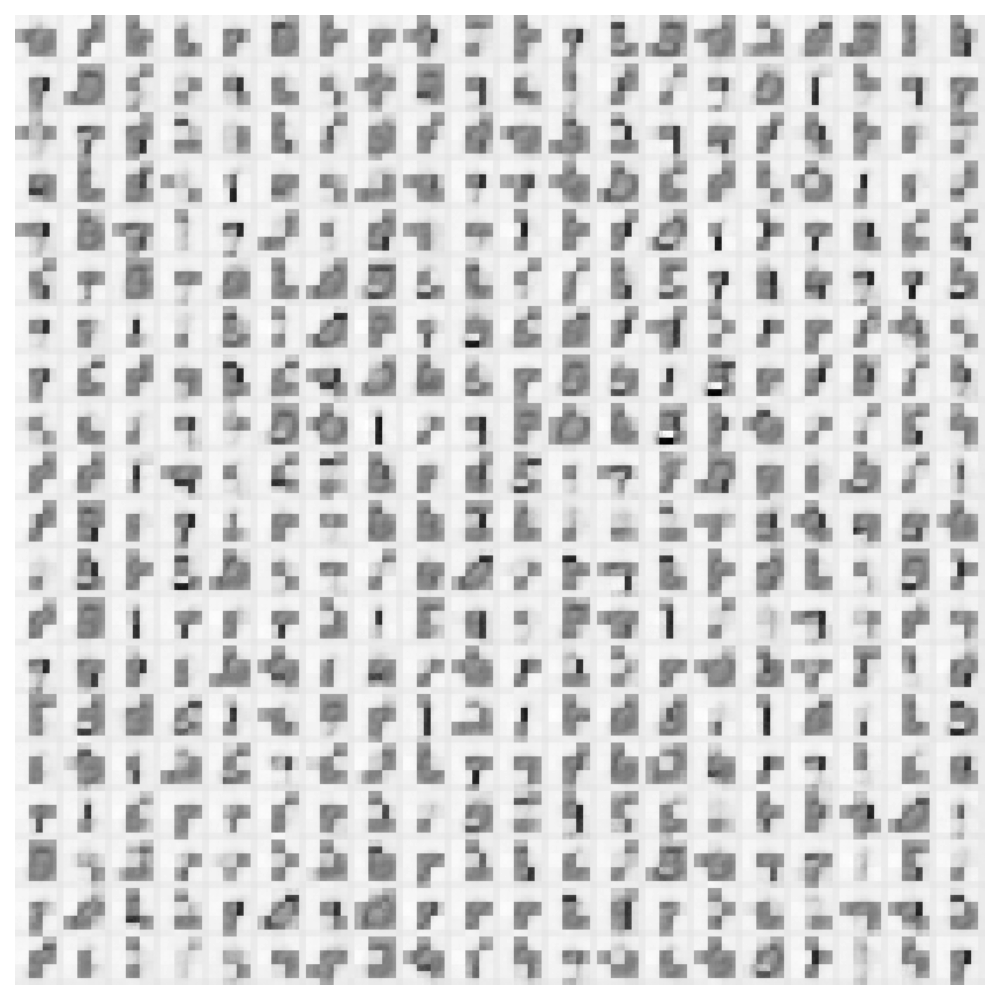}
            \subcaption*{wavelet denoising+$p$-WGAN}
        \end{subfigure}
        \begin{subfigure}[b]{0.3\textwidth}
            \includegraphics[width=\textwidth]{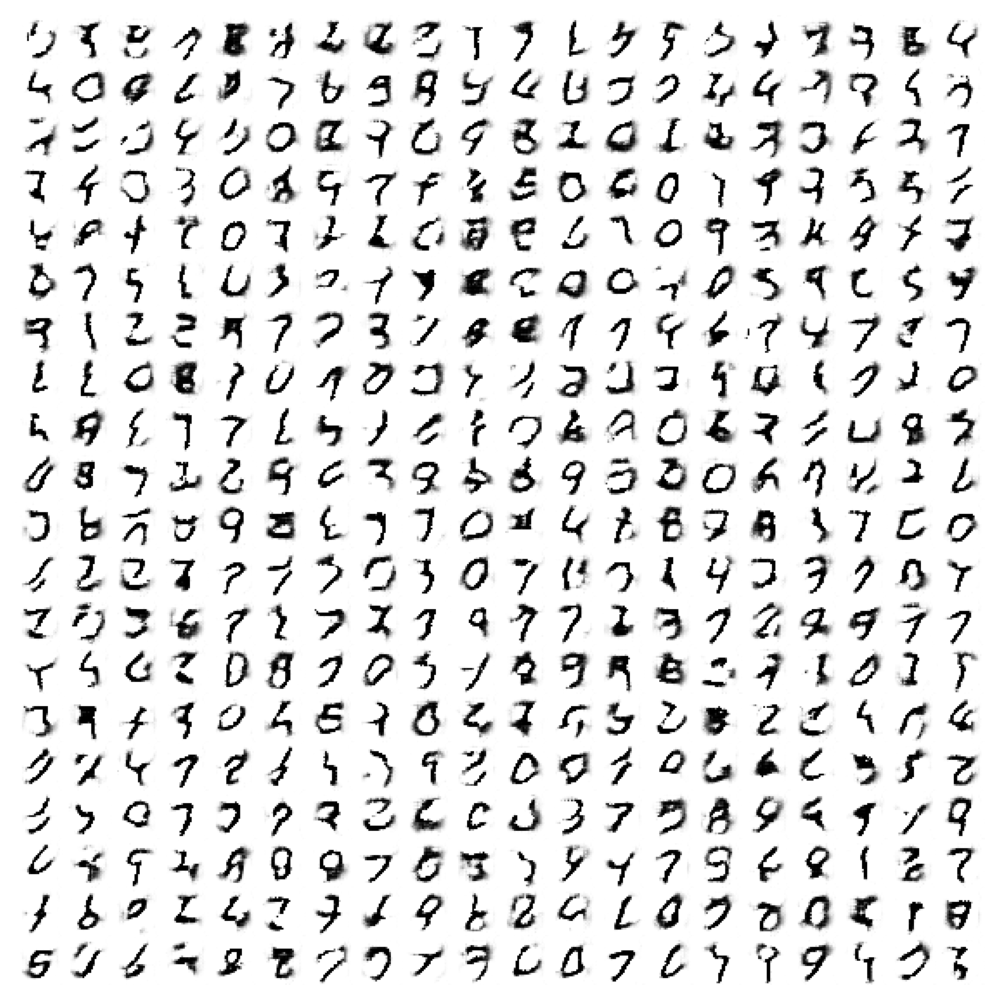}
            \subcaption*{entropic $p$-WGAN~\eqref{eq:Wasserstein_reg_obj}}
        \end{subfigure}
    \vspace{-3em}
    \end{minipage}
    \caption{Comparing Wavelet denoising and Entropic 2-WGAN for the Gaussian mechanism }\label{tab:wavelet_gauss}
\end{figure}

\subsection{MNIST: Empirical convergence}\label{sec:emporical_convergence}
We next empirically check how the performance (measured by the $2$-Wasserstein distance) depends on the privatization level. In our experiment, we set $p=2$ and train GANs on privatized MNIST samples  with 3 different loss functions: the entropy-regularized 2-Wasserstein loss  $W_{2, 2\sigma^2}(P_{G(Z)}, Q_Y^n)$ (where $Q_Y^n$ represents the empirical distribution of the privatized samples), the 2-Wasserstein distance $W_{2}^2(P_{G(Z)}, Q_Y^n)$ and the sinkhorn divergence \cite{feydy2019interpolating} (which is the debiased version of the entropy-regularized 2-Wasserstein distance). We chose a smaller generator model for this experiment -- a 2 hidden layer fully connected neural network with 500 neurons on each hidden layer and a 2-dimensional latent distribution, i.e. $P_Z = \text{Uniform}[0,1]^2, L=200, b=200.$ The images were then flattened into $28^2$-dimensional  vectors.  We report the $2$-Wasserstein distance between the generated and the target distribution $P_X$ in Figure \ref{fig:w2_vs_sigma} (left)  where we approximate the target distribution by the empirical distribution of the non-privatized data. It is not surprising that the Wasserstein distance between the generated and the target distribution is smallest for the the entropy-regularized 2-Wasserstein GAN since our theory suggest that the entropic regularization encourages the generated distribution to converge to the (true) target distribution.  The results also show that the distance grows with the noise scale $\sigma$ for all three of the metrics we considered, however, the slope of our method $W_{2, 2\sigma^2}(P_{G(Z)}, Q_Y^n)$ is the smallest. The growth is to be expected from theorem \ref{thm:generalization}, when the dataset size $n$ is kept constant, increasing the noise scale $\sigma$ (and thus privatization) degrades the performance. We also report the entropy-regularized 2-Wasserstein distance to the privatized dataset since it is minimized when $P_{G(Z)} = P_X$ and can be used as a measure of closeness of $P_{G(Z)}$ to $P_X.$ The plot indicates that the Sinkhorn divergence and the unregularized Wasserstein distance do not implicitly minimize this measure.

\begin{figure}[htbp]
\begin{minipage}[t]{.5\linewidth}
\includegraphics[width=\linewidth]{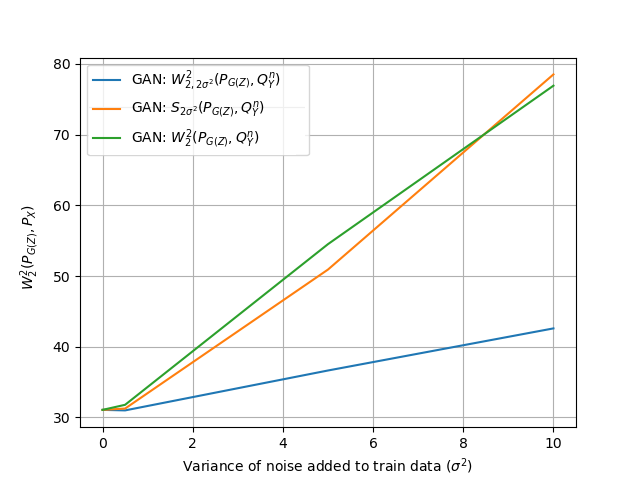}
% \caption{ figure}
\end{minipage}\hfill
\begin{minipage}[b]{.5\linewidth}
\includegraphics[width=\linewidth]{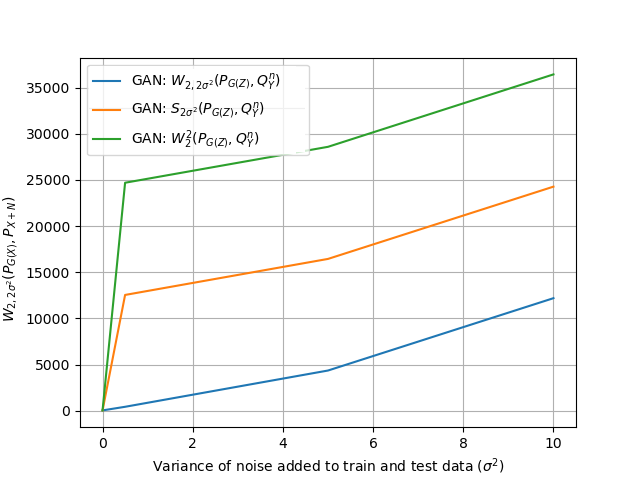}
\end{minipage}
\caption{Dependence of the $2$-Wasserstein distance(left) and the validation error (right) on the noise scale $\sigma$}\label{fig:w2_vs_sigma}.
\end{figure}

\subsection{Influence of train set size on the error}
Here we empirically check how the performance (measured by the $2$-Wasserstein distance) depends on the size of the training set . We use the setting described in section~\ref{sec:emporical_convergence} with $\sigma^2 = 4$ and train the model with the entropy-regularized 2-Wasserstein loss between the  generated distribution and the empirical distribution of the privatized samples $W_{2, 2\sigma^2}(P_{G(Z)}, Q_Y^n).$ We report the $2$-Wasserstein distance between the generated $P_{G(Z)}$ and the target distribution $P_X$ on Figure~\ref{fig:w2_val}, where we approximate the target distribution by the empirical distribution of the non-privatized data that was used for training (curve labeled "train") and that was left out for validation (curve labeled "test"). To compute $W_{2, 2\sigma^2}(P_{G(Z)}, P_Y)$ we use mini-batches of size $200.$ The distance is decreasing on the left plot for both train and test curves, which is expected by theorem~\ref{thm:generalization} to be proportional to $1/\sqrt{n}.$ The closeness of the train and test curves also shows no signs of overfitting, which is most probably happening due to privatization. 
% On the left plot we see that the batch size is a very important parameter for training and choosing a larger batch size improves the generated distribution.
\begin{figure}[htbp]
\centering
% \begin{minipage}[t]{\linewidth}
\includegraphics[width=0.5\linewidth]{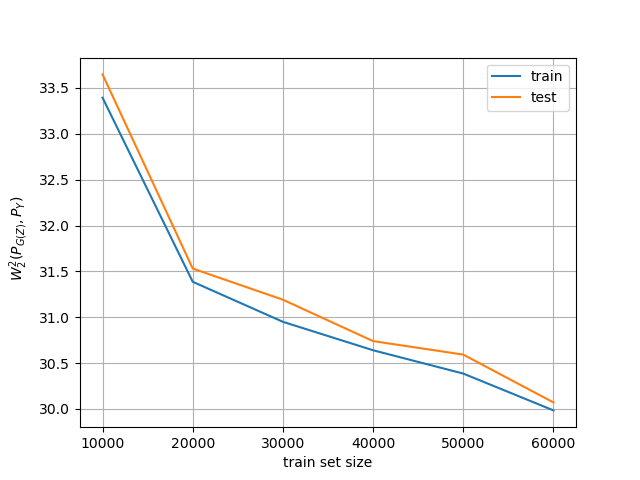}
% \caption{ figure}
% \end{minipage}\hfill
% \begin{minipage}[b]{.5\linewidth}
% \includegraphics[width=\linewidth]{images/batch_size_reg.png}
% \end{minipage}
\caption{Dependence of $W_{2, 2\sigma^2}(P_{G(Z)}, P_Y)$ on the dataset size $n$ 
% (left) and batch size $b$ (right)
}\label{fig:w2_val}
\end{figure}

%compois to be expected given the high dimension of the data and the fact that the model is trained in a non-interactive setting.

   \subsection{MNIST: Higher Privacy samples}\label{sec:higher_privacy}
   In this section, we further investigate the performance of entropic $p$-Wasserstein GAN on locally privatized data. We set the number of sinkhorn steps $L=400$ and the batch size to be $b=400$ and we performed optimization with Adam optimizer \cite{kingma2014adam} and learning rate varied in $\{0.005, 10^{-4}, 5\times 10^{-5}\}.$ We optimized for 150 epochs. For $p=1$ we first took the discrete cosine transform of the images and clipped the coefficients below 0.8 quantile to preserve more information, and then to control the sensitivity, we projected each training image onto an $\ell_1$ ball with radius $140$ (the parameters were chosen based on 1 held-out image in a way that it would not visually distort the image beyond recognition). We also applied DCT transform to the generator output before plugging it into the loss function.
   Similarly, for $p=2$ we projected each training image onto an $\ell_2$ ball with radius $20$, but we did not apply any transforms (since $\ell_2$-norm does not change under multiplication by an orthonormal matrix).
   
\begin{figure}[htbp]
\begin{minipage}[t]{.45\linewidth}
\includegraphics[width=\linewidth]{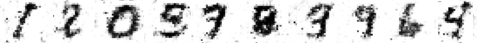}
% \caption{ figure}
\end{minipage}\hfill
\begin{minipage}[b]{.45\linewidth}
\includegraphics[width=\linewidth]{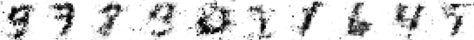}
\end{minipage}
\caption{Entropic 1-WGAN on MNIST trained on data privatized with the Laplace mechanism achieving $\eps$-LDP $\eps = 35$ (left) and $\eps=25$ (right)}\label{fig:laplace_results}
\end{figure}

\begin{figure}[htbp]
\begin{minipage}[t]{0.45\linewidth}
\includegraphics[width=\linewidth]{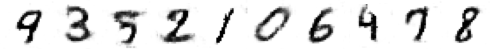}
% \caption{ figure}
\end{minipage}\hfill
\begin{minipage}[b]{0.45\linewidth}
\includegraphics[width=\linewidth]{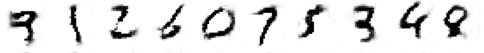}
\end{minipage}
\caption{Entropic 2-WGAN on MNIST trained on data privatized with the Gaussian mechanism achieving $(\eps,\delta)$-LDP with $\delta=10^{-4}$ and $\eps = 30$ (left) and $\eps=25$ (right)}\label{fig:gauss_results}
\end{figure}
% \subsection{}
The results indicate the effectiveness of our model at higher privacy regimes. However,  smaller $\epsilon$ values still produced a lot of noise in the generated samples or eroded the images significantly. This can be potentially mitigated by increasing the number of samples as suggested in Theorem \ref{thm:generalization}; however the relatively small size of MNIST limits the privacy levels that can be achieved. Note that the privacy parameters chosen for the images are rather large, however, our method still performs reasonably despite the exponential blowup in the excess risk. We believe that the dependence of the excess risk on the privacy budget $\epsilon$ can thus be improved and is a direction for future work. 
% We discuss convergence in more detail in the appendix, where we empirically observe that when the dataset size $n$ is kept constant, increasing the noise scale $\sigma$ (higher privatization) degrades the performance as expected from Theorem \ref{thm:generalization}.

We provide additional samples for different privacy levels and report 400 randomly sampled digits on figure \ref{fig:higher_privacy} for the Laplace and Gaussian mechanisms.

\begin{figure}[htbp]
\newcommand{\w}{0.24}
\begin{subfigure}[t]{\linewidth}
\begin{subfigure}[t]{\w\linewidth}
\includegraphics[width=\linewidth]{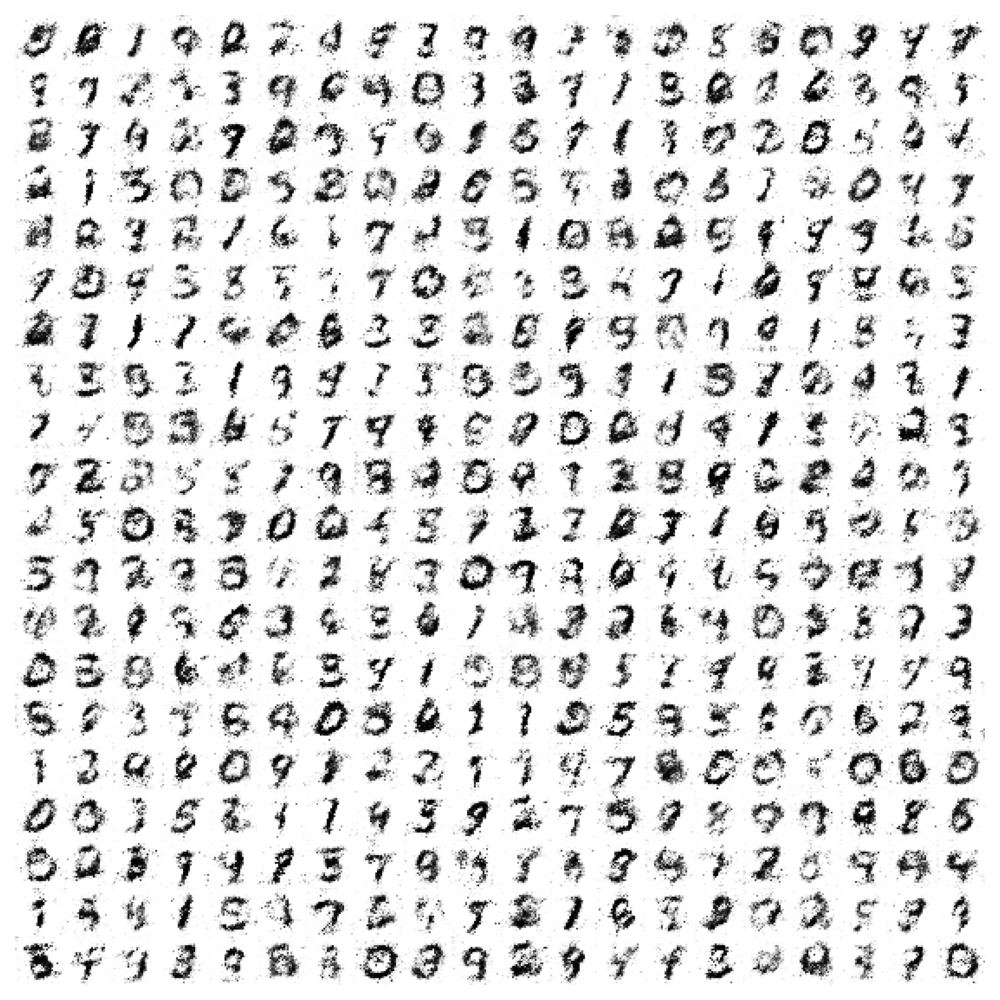}
\caption{$\eps=65.25$}
\end{subfigure}\hfill
\begin{subfigure}[t]{\w\linewidth}
\includegraphics[width=\linewidth]{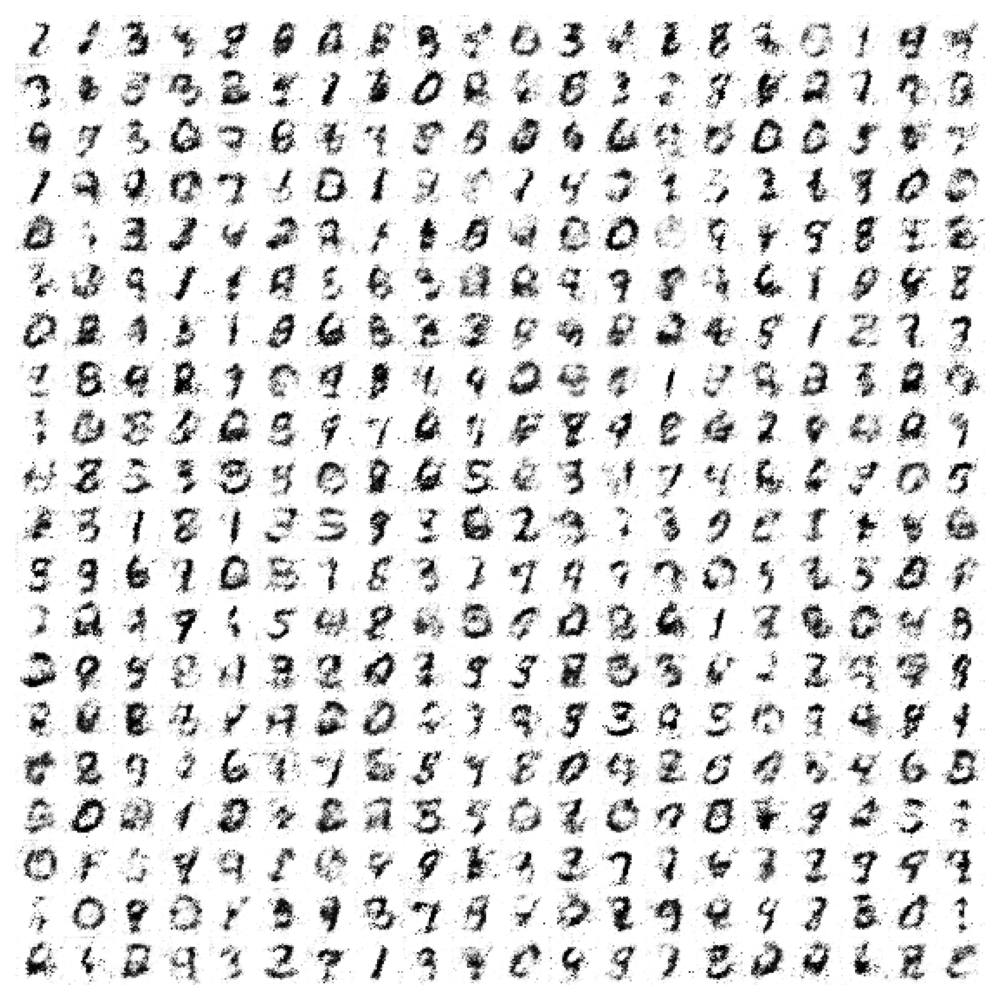}
\caption{$\eps=49$}
\end{subfigure}\hfill
\begin{subfigure}[t]{\w\linewidth}
\includegraphics[width=\linewidth]{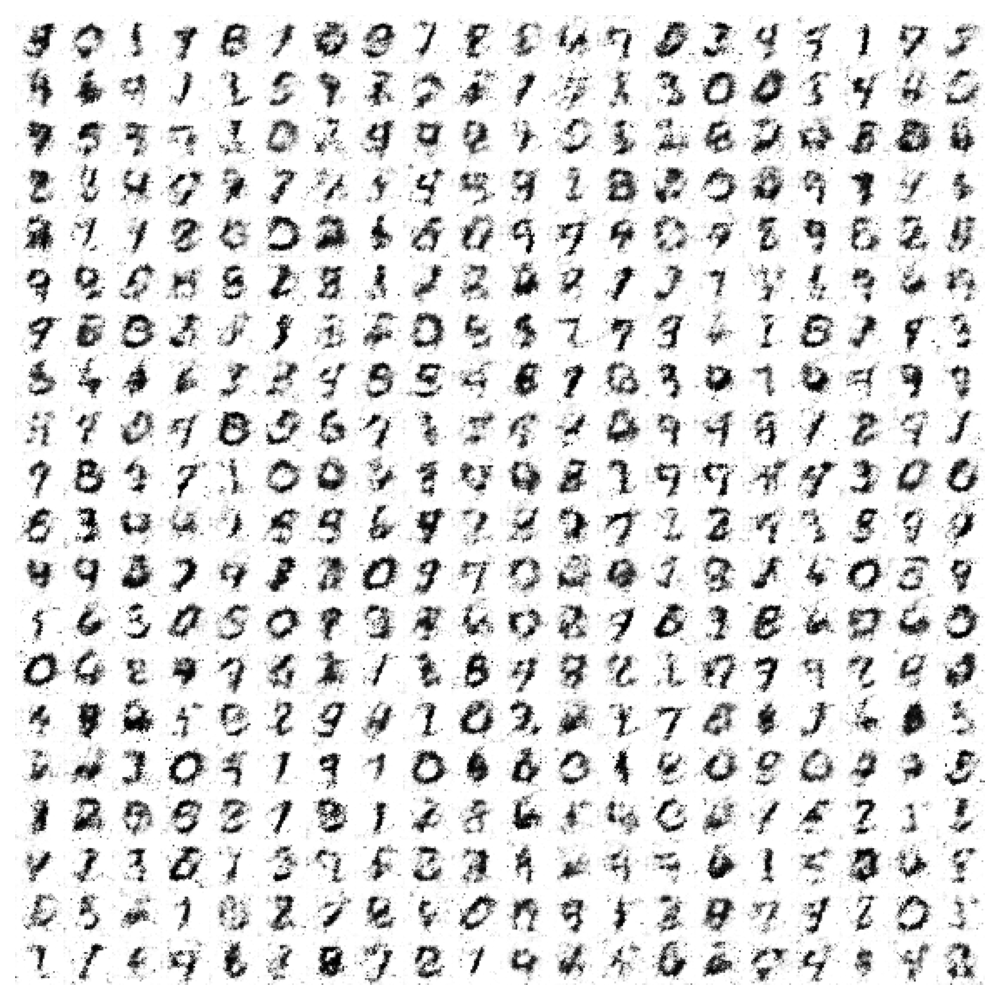}
\caption{$\eps=35$}
\end{subfigure}\hfill
\begin{subfigure}[t]{\w\linewidth}
\includegraphics[width=\linewidth]{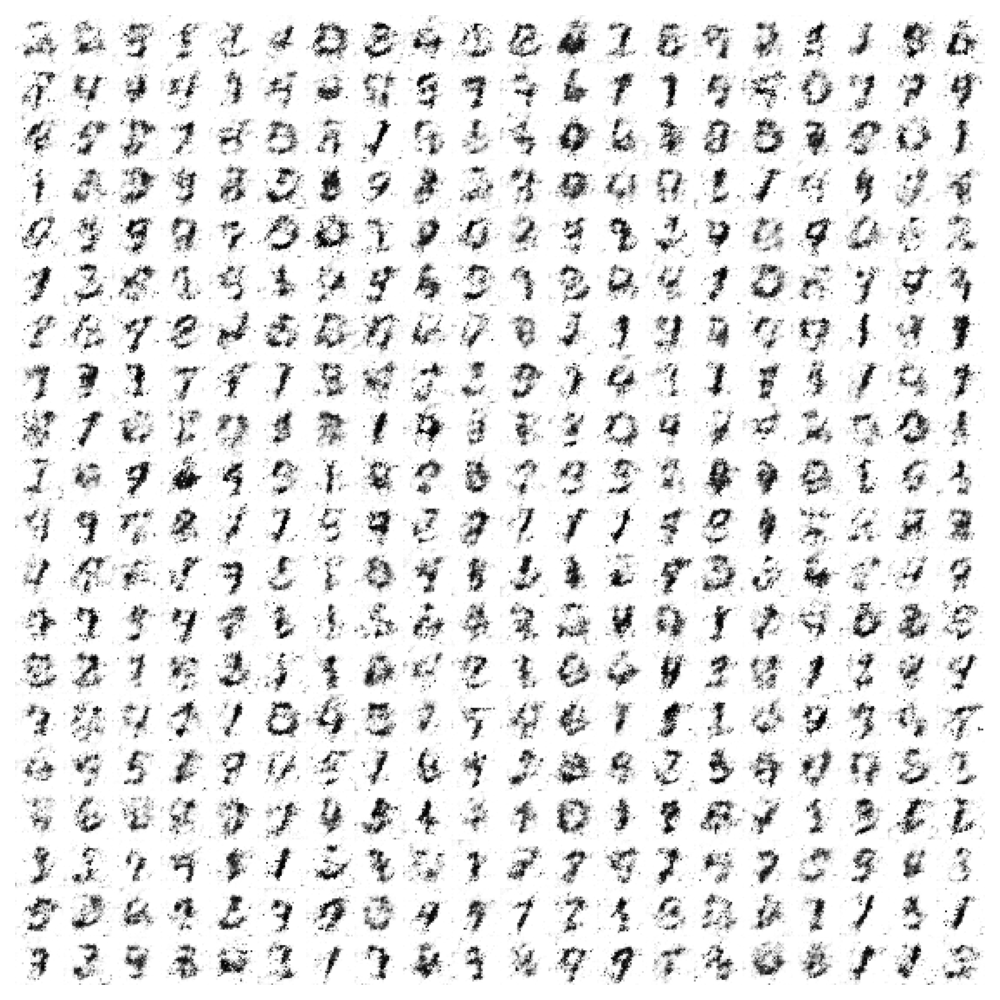}
\caption{$\eps=25$}
\end{subfigure}%
% \caption{Laplace mechanism with different privacy budgets $\eps$ and clipping of the discrete cosine transform, 1-EWGAN}\label{fig:higher_privacy_laplace}
\end{subfigure}
\begin{subfigure}[t]{\linewidth}
% \begin{figure}[htbp]
% \newcommand{\w}{0.24}
\begin{subfigure}[t]{\w\linewidth}
\includegraphics[width=\linewidth]{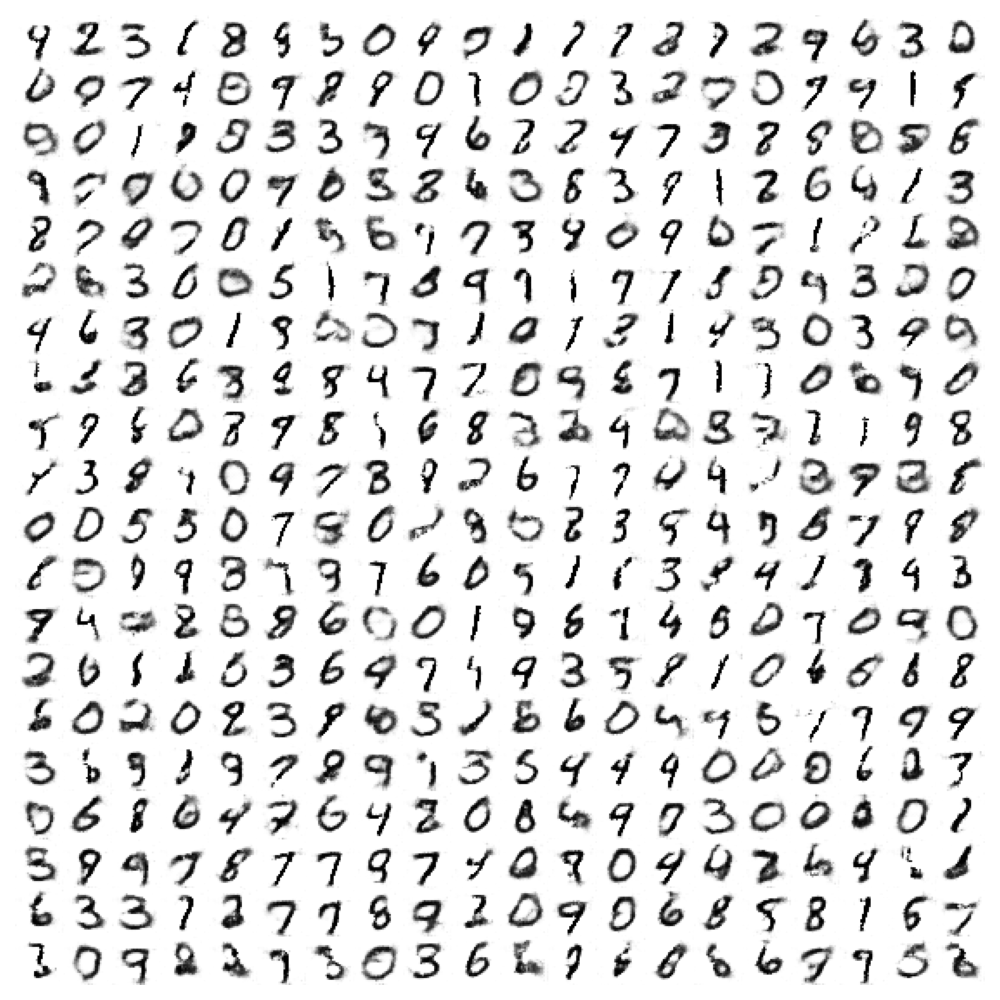}
\caption{$\eps=35$}
\end{subfigure}\hfill
\begin{subfigure}[t]{\w\linewidth}
\includegraphics[width=\linewidth]{images/sinkhorn_log_l2_eps_30.png}
\caption{$\eps=30$}
\end{subfigure}\hfill
\begin{subfigure}[t]{\w\linewidth}
\includegraphics[width=\linewidth]{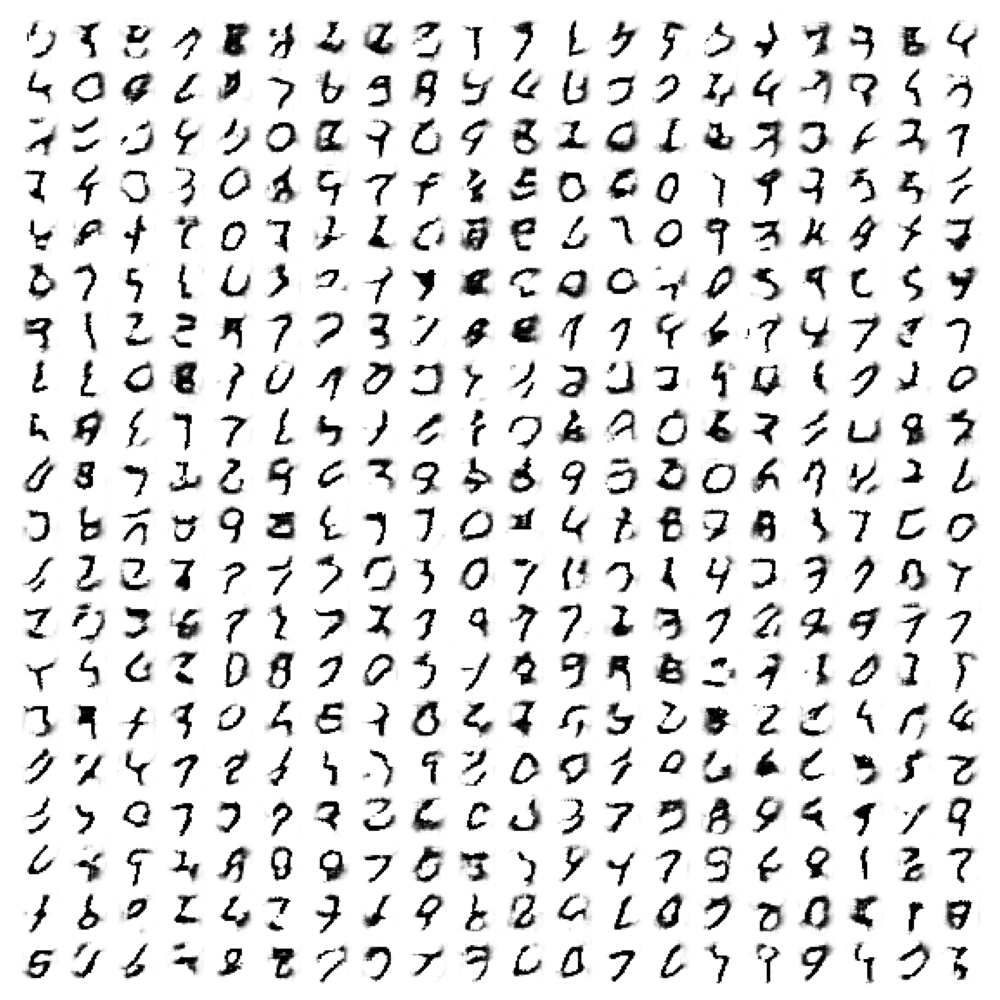}
\caption{$\eps=25$}
\end{subfigure}\hfill
\begin{subfigure}[t]{\w\linewidth}
\includegraphics[width=\linewidth]{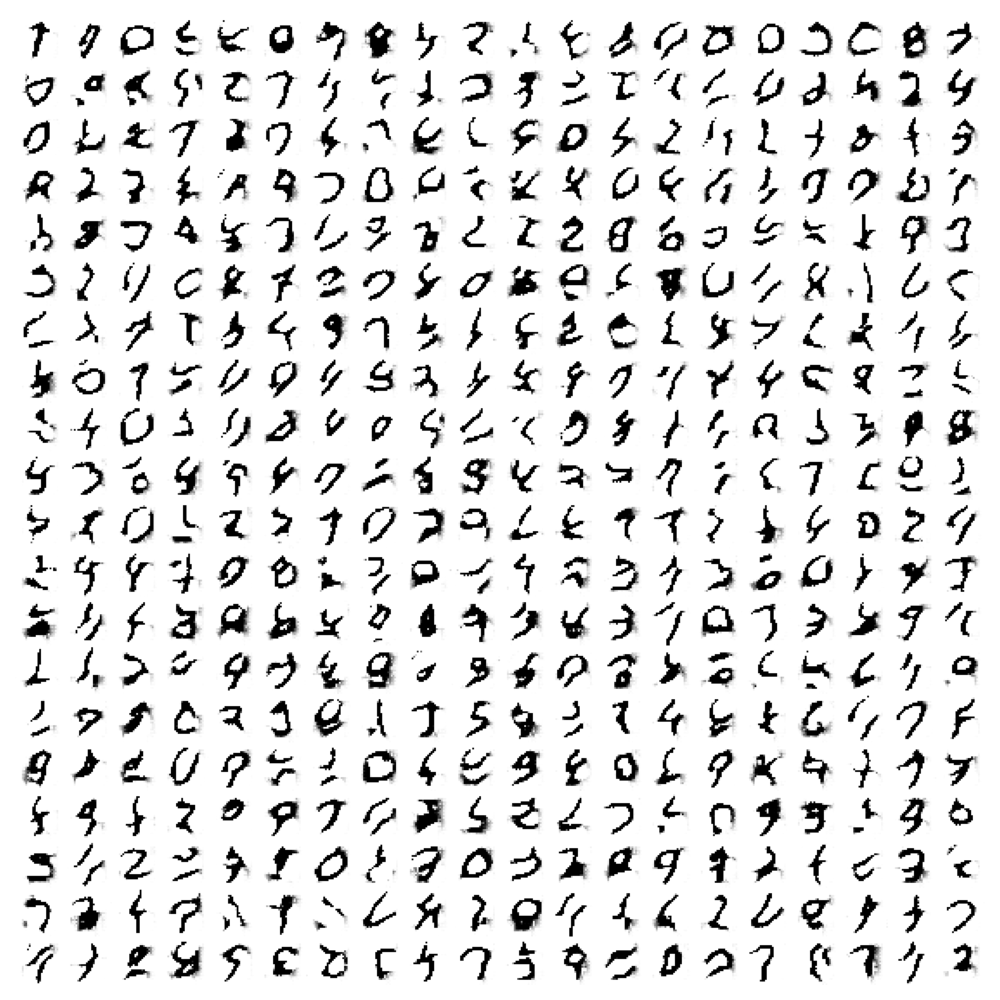}
\caption{$\eps=24$}
\end{subfigure}
% \caption{Gaussian mechanism with different privacy budgets $\eps$ and clipping the euclidean norm of images, 2-EWGAN}\label{fig:higher_privacy_gaussian}
\end{subfigure}
\caption{Laplace mechanism with different privacy budgets $\eps$ and clipping of the discrete cosine transform, 1-EWGAN (top) and Gaussian mechanism with different privacy budgets $\eps$ and clipping the euclidean norm of images, 2-EWGAN (bottom)}\label{fig:higher_privacy}
\end{figure}

\subsection{Comparison with Noisy Wasserstein GAN}
We next provide our experimental results with MNIST and FashionMNIST \cite{xiao2017fashion}, which is a set of $60000$ grayscale images of clothing items of size $28\times 28.$ 

Here we fix $p=1$ and compare entropic $1$-Wasserstein GAN to a noisy $1$-Wasserstein GAN applied to data privatized with the Laplace mechanism.

The noisy $1$-Wasserstein GAN is a GAN trained with the following unregularized loss function:
\begin{equation}
    W_1(P_{M(G(Z))}, Q_Y^n),\label{eq:noisy_1wass}
\end{equation} where $M(\cdot)$ is the privatization mechanism. Note that this loss function also satisfies theorem \ref{thm:denoising}, but suffers from the curse of dimensionality, namely 
$W_1(P_{M(G(Z))}, Q_Y^n) = \Omega(n^{-1/d}).$ However, Wasserstein GANs have been successful in practice \cite{arjovsky2017wasserstein,gulrajani2017improved} and use a different benchmark, so we compare their performance with the performance of the entropic $1$-Wasserstein GAN. 

We again use a DCGAN as a generator, but we removed the batch normalization layers and used a larger batch size of $6000$ as suggested by \cite{bie2023private}. As in section \ref{sec:MNISTvsdenoiding}, experiments in this section do not involve projecting onto the $\ell_1/\ell_2$ balls. The pixel values of the images were rescaled to $[-1,1]$ leading to $\Delta_1 = 2\times 28^2$  $\ell_1$ sensitivity and $\Delta_2=56$ $\ell_2$ sensitivity. For FashionMNIST we choose the noise scale $\sigma = 7,$ which results in $\epsilon = 224$-LDP, and for MNIST we choose the noise scale $\sigma = 3,$ which results in $\epsilon = 97$-LDP. We report the FID measures in and random saples in Figure \ref{fig:noisy_comparison_fashion} for FashionMNIST and Figure \ref{fig:noisy_comparison_fashion} for MNIST.

We also report Frechet Inception Distance (FID) \cite{heusel2017gans} calculated between the generated and the validation set as a quantitative measure of performance. 

\begin{figure}[htbp]
\begin{minipage}{.48\linewidth}
\includegraphics[width=\linewidth]{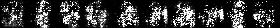}
\subcaption{1-WGAN \eqref{eq:noisy_1wass}, FID = 318}
\end{minipage}\hfill
\begin{minipage}{.48\linewidth}
\includegraphics[width=\linewidth]{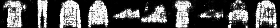}
\subcaption{entropic 1-Wasserstein GAN, \ref{eq:Wasserstein_reg_obj}, FID=208}
\end{minipage}
\caption{Fashion MNIST samples generated with 1-Wasserstein GAN and the addition of noise to the generator (a) and Entropic 1-Wasserstein GAN (b).}\label{fig:noisy_comparison_fashion}
\end{figure}

\begin{figure}[htbp]
\begin{minipage}{.48\linewidth}
\includegraphics[width=\linewidth]{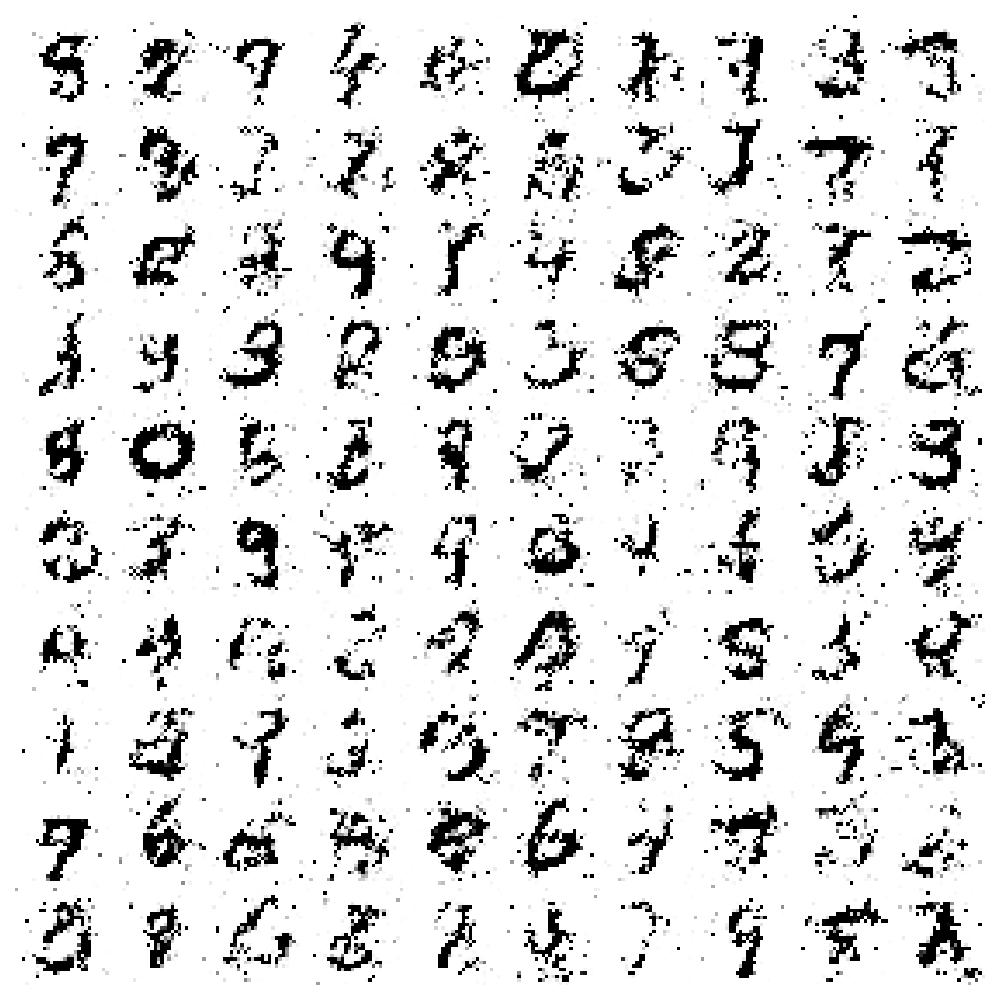}
\subcaption{1-WGAN \eqref{eq:noisy_1wass}, FID = 223}
\end{minipage}\hfill
\begin{minipage}{.48\linewidth}
\includegraphics[width=\linewidth]{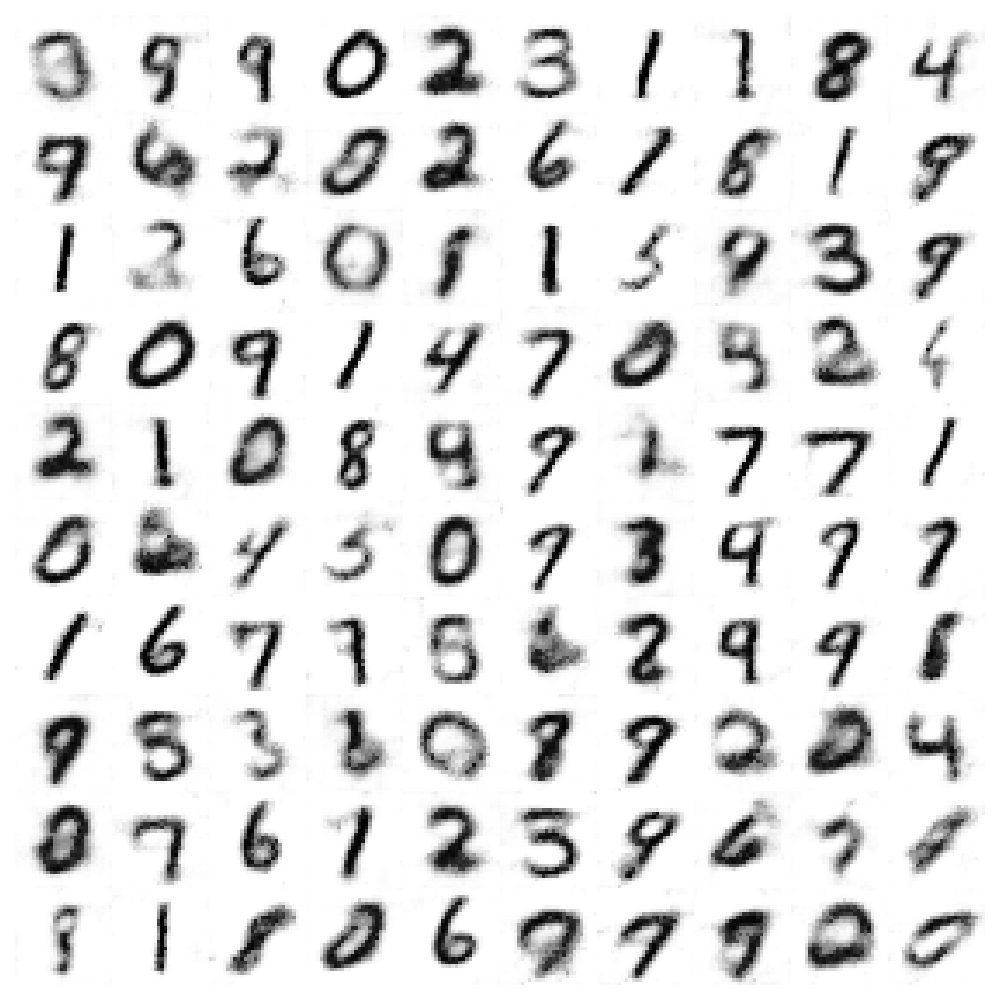}
\subcaption{entropic 1-Wasserstein GAN, \ref{eq:Wasserstein_reg_obj}, FID=28}
\end{minipage}
\caption{MNIST samples generated with 1-Wasserstein GAN and the addition of noise to the generator (a) and Entropic 1-Wasserstein GAN (b).}\label{fig:noisy_comparison}
\end{figure}

The experiments clearly show that the images generated by the entropic Wasserstein GAN still look like clothing items, while the images generated by the unregularized Wasserstein GAN with noise added to the generator look like noise. The closeness of the distributions is also validated by the closeness in FID scores.

Note that the images for the entropic 1-Wasserstein GAN are still corrupted by noise even given the very high privacy budget. This is due to the large $\ell_1$ sensitivity of the data. In the next section we address that problem with clipping the image norms before applying privatization noise.

\subsection{Fashion MNIST: Higher Privacy Samples}
To achieve a smaller privacy budget we clipped the norm of the images to have sensitivity $\Delta_1 = 700$ and $\Delta_1 = 550.$ Adding Laplace noise with $\sigma = 7$ then results in $\epsilon = 100$ and $\epsilon=78,6$ -LDP. The results are presented in figure \ref{fig:fashion_laplace}

\begin{figure}[htbp]
\begin{minipage}{.48\linewidth}
\includegraphics[width=\linewidth]{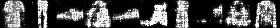}
\subcaption{$\epsilon=100$, FID = 183}
\end{minipage}\hfill
\begin{minipage}{.48\linewidth}
\includegraphics[width=\linewidth]{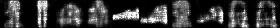}
\subcaption{$\epsilon=78,6$, FID = 214}
\end{minipage}
        \caption{Entropic 1-WGAN trained on FashionMNIST data privatized with the Laplace mechanism} \label{fig:fashion_laplace}
        \end{figure}

   Similarly, for $p=2$ we clipped the $\ell_2$ norm to have sensitivity $\Delta_2 = 40.$ Adding Gaussian noise with $\sigma = 9.17$ or $\sigma = 7.24$ results in $\epsilon = 25$ and $\epsilon = 35$ privacy for $\delta = 10^{-4},$ see figure \ref{fig:fashion_gauss}.

\begin{figure}[htbp]
\begin{minipage}{.48\linewidth}
\includegraphics[width=\linewidth]{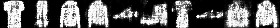}
\subcaption{$\epsilon=35,\delta=10^{-4}$, FID = 207}
\end{minipage}\hfill
\begin{minipage}{.48\linewidth}
\includegraphics[width=\linewidth]{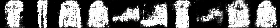}
\subcaption{$\epsilon=25$, FID = 237}
\end{minipage}
        \caption{Entropic 2-WGAN trained on FashionMNIST data privatized with the Gaussian mechanism} \label{fig:fashion_gauss}
        \end{figure}
% \subsection{}
The results indicate the effectiveness of our model at higher privacy regimes. However,  smaller $\epsilon$ values still produced a lot of noise in the generated samples or eroded the images significantly. This can be potentially mitigated by increasing the number of samples as suggested in Theorem \ref{thm:generalization}; however the relatively small size of MNIST limits the privacy levels that can be achieved. 

\section{Detailed Proofs}\label{sec:proofs}
In this section we prove the theorems presented in section \ref{sec:mainres}.

\subsection{Proof of Lemma \ref{lemma:denoising_nonrealizable}}\label{sec:proof_denoising_nonrealizable}
\begin{proof}   We first prove the following lemma, which is used in this proof and the proof of Corollary~\ref{col:KL_bound} (Corollary~\ref{col:KL_bound} is a direct consequence of the lemma and excess risk results -- theorem \ref{thm:generalization_gauss} and Corollary \ref{col:generalization_lap}).

\begin{lemma}\label{lemma:1}
 Let $X\sim P_X$ and $Y = M(X) = X + N,$ where $N = (N_1,\ldots,N_d)\sim f_N$ independent of $X$ 
    and $f_N(x)\propto e^{-\|x\|^p/(p\sigma^p)}$ Then
    \begin{align}
D_{KL}&(P_Y\|P_{G(Z)+N})\leq\frac1{p\sigma^p}\l(W_{2, 2\sigma^2}(P_Y, P_{G(Z)})  - W_{2, 2\sigma^2}(P_Y, P_X)\r),\label{eq:KL_bound}
\end{align}
where $D_{KL}(P\|Q)$ is the KL-divergence ($D_{KL}(P\|Q) = \int P(x)\log\frac{P(x)}{Q(x)}dx$ for continuous $P_X$ and $D_{KL}(P\|Q) = \sum_{x\in\cX} P(x)\log\frac{P(x)}{Q(x)}$ for discrete $X$)
\end{lemma}
\begin{proof}
We start by noting that $c(x,y) = -\log f_N(x-y) = \|x-y\|_p^p/(p\sigma^p) + C,$ where $C$ is a constant. Thus, for two probability measures $\mu,\nu$
\begin{align}
    S_c(\mu,\nu) = \frac1{p\sigma^p}W_{p,p\sigma^p}(\mu,\nu) + C.\label{eq:Sc_to_W}
\end{align}
Additionally, since $P_Y = P_{M(X)},$ the KL-divergence in \eqref{eq:Sc_to_KL} is zero for $P_{G(Z)} = P_X,$ so $S_c(P_X, P_Y) = h(Y).$ 
We can now rewrite the difference on the RHS of \eqref{eq:KL_bound} using \eqref{eq:Sc_to_W} and \eqref{eq:Sc_to_KL}:
\begin{align}
    \frac1{p\sigma^p}\l(W_{p,p\sigma^p}(P_{G(Z)},P_Y) - W_{p,p\sigma^p}(P_{X},P_Y)\r) &=S_c(P_{G(Z)},P_Y) - S_c(P_{X},P_Y)\nonumber\\
    &=\inf_{\pi\in\Pi(P_{G(Z)}, P_Y)}\EE_{U\sim P_{G(Z)}} D_{KL}(\pi_{Y\mid U}(\cdot\mid U)\|p_M(\cdot\mid U)),\label{eq:difference_bound_KL}
\end{align}
where $p_M(\cdot\mid u) = f_N(\cdot-U)$ is the conditional pdf of the privatization mechanism.

We then use the chain rule for KL-divergence, which states that 
% As a final step we use the chain rule for KL-divergence: 
for any two joint distributions $Q^1\ll Q^2$ with marginals $(Q^1_X, Q^1_Y) $ and $(Q^2_X, Q^2_Y)$ correspondingly, it holds that
\begin{align}
    D_{KL}(Q^1\| Q^2) &= D_{KL}(Q^1_X\| Q^2_X) + \EE_{X\sim Q^1_X} D_{KL}(Q^1(\cdot\mid X)\| Q^2(\cdot\mid X))
    \label{eq:kl_chainx}
    % \\
    % &= D_{KL}(Q^1_Y\| Q^2_Y) + \EE_{Y\sim Q^1_Y} D_{KL}(Q^1(\cdot\mid Y)\| Q^2(\cdot\mid Y)).\label{eq:kl_chainy}
\end{align}
Setting $Q^1(u,y) = \pi(u,y)$ and $Q^2 = P_G(u)p_M(y\mid u) = P_{G,M(G)},$ we can rewrite the $D_{KL}$ term in \eqref{eq:difference_bound_KL} 
% using \eqref{eq:kl_chainx} 
as 
\begin{align}
    \frac1{p\sigma^p}\l(W_{p,p\sigma^p}(P_{G(Z)},P_Y) - W_{p,p\sigma^p}(P_{X},P_Y)\r) 
    &=\inf_{\pi\in\Pi(P_{G(Z)}, P_Y)}D_{KL}(\pi\|P_{G(Z),M(G(Z))})\label{eq:interm}.
\end{align}
Finally, \eqref{eq:kl_chainx} also shows that the KL-divergence between two joint distributions dominates the KL-divergence between the corresponding the marginals, namely $D_{KL}(\pi\|P_{G,M(G)})\geq D_{KL}(P_Y\|P_{M(G(Z))}),$ so continuing from \eqref{eq:interm} we get 
\begin{align}
    \frac1{p\sigma^p}\l(W_{p,p\sigma^p}(P_{G(Z)},P_Y) - W_{p,p\sigma^p}(P_{X},P_Y)\r) 
    &\geq D_{KL}(P_Y\|P_{M(G(Z))}) = D_{KL}(P_Y\|P_{G(Z)+N}).\nonumber
\end{align}
\end{proof}

We next prove Lemma \ref{lemma:denoising_nonrealizable}. We first show that 
\begin{align}
W_{p,p\sigma^p}(P_{G(Z)}, P_Y) - W_{p,p\sigma^p}(P_{X}, P_Y)\leq \begin{cases}
    W_{2}^2(P_{G(Z)}, P_X)&\text{ if } p=2,\\
    p2^{p-1}W_p(P_{G(Z)}, P_X)\l(\sigma^p+W_p(P_{G(Z)}, P_X)\r)^{1-1/p}&\text{ if } p\geq 1.
\end{cases}\label{eq:proof1ineq}
\end{align} 

We continue from  \eqref{eq:difference_bound_KL}: fix some coupling  $\pi\in \Pi(P_{G(Z)}, P_X)$ and let the joint distribution of $U=G(Z), X,$ and $Y$ be $(U,X,Y) \sim \gamma(u,x,y) = \pi(u,x)f_N(y-x),$ or equivalently, $U-X-Y$ is a Markov chain with $(U,X)\sim \pi$ and $Y=X+N$ with $N$ independent of $(X,U).$ Note that $\gamma_{UY}(u,y) \in\Pi(P_{G(Z)},P_Y)$ and 
$$\gamma_{Y\mid U}(y\mid u) = \int\gamma(y\mid x)\gamma(x\mid u) dx = \int f_N(y- x)\pi(x\mid u) dx = \EE_{X\sim \pi_{X\mid U=u}}f_N(y-X).$$
Plugging this into \eqref{eq:difference_bound_KL} gives
% , \,(G(Z), X)\sim \pi$ and $Y = X+N$ where $N\sim f_N$ is independent of $(X,G(Z)).$ Then
\begin{align}
    \frac1{p\sigma^p}\l(W_{p,p\sigma^p}(P_{G(Z)},P_Y) - W_{p,p\sigma^p}(P_{X},P_Y)\r) &\leq \EE_{U\sim P_{G(Z)}} D_{KL}(\gamma_{Y\mid U}(\cdot\mid U)\|f_N(\cdot- U))\nonumber\\
    &=\EE_{U\sim P_{G(Z)}} D_{KL}(\EE_{X\sim \pi_{X\mid U}(\cdot\mid U)} f_N(\cdot - X)\|f_N(\cdot-U))\nonumber\\
    &\leq \EE_{U\sim P_{G(Z)}}\EE_{X\sim \pi_{X\mid U}(\cdot\mid U)} D_{KL}(f_N(\cdot - X)\|f_N(\cdot-U))\label{eq:KL_convexity},
\end{align}
where \eqref{eq:KL_convexity} follows from the convexity of KL-divergence and Jensen's inequality. We can now plug in the definition of KL-divergence leading to 
\begin{align}
    \frac1{p\sigma^p}\l(W_{p,p\sigma^p}(P_{G(Z)},P_Y) - W_{p,p\sigma^p}(P_{X},P_Y)\r) &\leq \EE_{(U,X)\sim \pi} D_{KL}(f_N(\cdot - X)\|f_N(\cdot-U))\nonumber\\
    &=\EE_{(U,X)\sim \pi}\int f_N(z - X)\log\frac{f_N(z-X)}{f_N(z-U)}dz\nonumber\\
    &=\frac{1}{p\sigma^p}\EE_{(U,X)\sim \pi,N\sim f_N}\l[\|N+X-U\|_p^p-\|N\|_p^p\r].\label{eq:interm_diff}
\end{align}
In the special case of $p=2$ it follows that 
\begin{align}
    W_{p,p\sigma^p}(P_{G(Z)},P_Y) - W_{p,p\sigma^p}(P_{X},P_Y) &\leq\EE_{(U,X)\sim \pi}\l[\|X-U\|_2^2\r],\nonumber
\end{align}
and taking the infimum over $\pi\in\Pi(P_X, P_{G(Z)})$ leads to \eqref{eq:proof1ineq}. When $p\neq 2, p\geq 1$ we can upper bound the RHS of \eqref{eq:interm_diff} using the convexity of $f(x) = (x)^p$ for $x\geq 0,$ which states that $f(x+\delta)-f(x)\leq f'(x+\delta)\delta$ leads to 
\begin{align*}
    \EE_{(U,X)\sim \pi,N\sim f_N}\l[\|N+X-U\|_p^p - \|N\|_p^p\r]&\leq \EE_{(U,X)\sim \pi,N\sim f_N}\l[\l(\|N\|_p+\|X-U\|_p\r)^p - \|N\|_p^p\r]\\
    &\leq p\EE_{(U,X)\sim \pi,N\sim f_N}\l[\|X-U\|_p\l(\|N\|_p+\|X-U\|_p\r)^{p-1}\r]
\end{align*}
We then use H\"older's inequality $\EE[XY]\leq (\EE|X|^p)^{1/p}(\EE|Y|^{p/(p-1)})^{1 - 1/p}$ to get 
\begin{align*}
    \EE_{(U,X)\sim \pi,N\sim f_N}\l[\|N+X-U\|_p^p - \|N\|_p^p\r]&\leq  p\EE\l[\|X-U\|_p^p\r]^{1/p}\EE\l[\l(\|N\|_p+\|X-U\|_p\r)^{p}\r]^{1-1/p}\\
    &\leq p2^{p-1}\EE\l[\|X-U\|_p^p\r]^{1/p}\EE\l[\|N\|_p^p+\|X-U\|_p^p\r]^{1-1/p}\\
    &= p2^{p-1}\EE\l[\|X-U\|_p^p\r]^{1/p}\l(\sigma^p+\EE\l[\|X-U\|_p^p\r]\r)^{1-1/p}
\end{align*}
We can now take the infimum over the couplings $\pi\in\Pi(P_X, P_{G(Z)})$ and arrive at \eqref{eq:proof1ineq} for $p\neq 2$ case. By Lemma \ref{lemma:1} and \eqref{eq:proof1ineq} choosing $G^* = \arg\min_{G\in\cG}W_{p,p\sigma^p}(P_{G(Z)}, P_Y)$ and $G^W = \arg\min_{G\in\cG}W_{p}(P_{G(Z)}, P_Y)$ we get:
\begin{align}
p\sigma^p D_{KL}(P_{Y}\| P_{G^*(Z)+N})&\leq W_{p,p\sigma^p}(P_{G^*(Z)}, P_Y) - W_{p,p\sigma^p}(P_{X}, P_Y)\nonumber\\
&\leq W_{p,p\sigma^p}(P_{G^W(Z)}, P_Y) - W_{p,p\sigma^p}(P_{X}, P_Y)\nonumber\\
&\leq \min_{G\in\cG}
\begin{cases}
    W_{2}^2(P_{G(Z)}, P_X)&\text{ if } p=2,\\
    p2^{p-1}W_p(P_{G(Z)}, P_X)\l(\sigma^p+W_p(P_{G(Z)}, P_X)\r)^{1-1/p}&\text{ if } p\geq 1.
\end{cases}\nonumber
\end{align} 

% It follows directly from Lemma~\ref{lemma:1} by setting $G = G(Z).$ Combining this with \eqref{eq:proof1ineq} results in
% \begin{align}
%     p\sigma^p D_{KL}(P_{G(Z)+N}, P_{X+N})\leq W_{p,p\sigma^p}(P_{G(Z)}, P_Y) - W_{p,p\sigma^p}(P_{X}, P_Y)\leq  W_{p}^p(P_{G(Z)}, P_X)\nonumber
% \end{align}
% Letting $$G^* = \arg\min_{G\in\cG}W_{p,p\sigma^p}(P_{G(Z)}, P_Y),$$ and taking the minimum on both sides of $$W_{p,p\sigma^p}(P_{G(Z)}, P_Y) - W_{p,p\sigma^p}(P_{X}, P_Y)\leq  W_{p}^p(P_{G(Z)}, P_X)$$ leads to 
% \begin{align}
%     p\sigma^p D_{KL}(P_{G^*(Z)+N}, P_{X+N})\leq W_{p,p\sigma^p}(P_{G^*(Z)}, P_Y) - W_{p,p\sigma^p}(P_{X}, P_Y)\leq \min_{G\in\cG} W_{p}^p(P_{G(Z)}, P_X).\nonumber
% \end{align}
\end{proof}

\subsection{Proof of Theorem \ref{thm:generalization_general}}\label{sec:proof_generalization_general}

\begin{proof} (Theorem \ref{thm:generalization_general})
We will be using the dual formulation of entropic optimal transport, so we begin by providing some related results. We denote the dual objective of $S_c(\mu,\nu)$ for two probability measures $\mu,\nu$ with supports $\supp(\mu),\supp(\nu)\subseteq \RR^d:$
\begin{align*}
    \Phi(f, g;\mu,\nu) = &\EE_{X\sim\mu}f(X) + \EE_{Y\sim\nu}g(Y)- \EE_{(X,Y)\sim\mu\times\nu}\l[e^{f(X) + g(Y)-c(X,Y)}\r]+1.
\end{align*}
Here $f:\supp(\mu)\to\RR$ and $g:\supp(\nu)\to\RR$  are called dual potentials are real-valued functions from the support of $\mu$ and $\nu$ and $f\in L^1(\mu),g\in L^1(\nu),$ where for a probability measure $\mu$ we denote the set of absolutely integratable functions w.r.t. $\mu$ as $L^1(\mu) = \{f:\supp(\mu)\to\RR\mid \EE_{X\sim\mu}[|f(X)|]<\infty\}.$

The dual function yields a lower bound on the optimal transport: for any $f, g\in L^1(\mu)\times L^1(\nu),\;  S_c(\mu, \nu)\geq \Phi(f, g;\mu,\nu).$ Strong duality is guaranteed to hold by corollary 3.1 (case (B)) in \cite{csiszar1975divergence} for any $\mu,\nu$ that satisfy
$\EE_{X,Y\sim \mu\times\nu}[c(X,Y)]<\infty,$  which holds for any combination of $Q_Y^n, P_Y, P_X$ and $P_{G(Z)}$ for any $G\in \cG$ by the assumption. Strong duality means that 
\begin{align}
   S_c(\mu, \nu) &= \max_{f, g\in L^1(\mu)\times L^1(\nu)} \Phi(f,g;\mu,\nu).\label{eq:dual}
\end{align}

The optimality conditions for the dual problem $\max_{f, g\in L^1(\mu)\times L^1(\nu)} \Phi(f,g;\mu,\nu)$ yield for any $x,y\in \RR^d$ (only the values of the dual potentials for $x\in\supp(\mu)$ and $y\in\supp(\nu)$ affect the problem value, but we extend them to $x,y\in\RR^d,$ this is known as the canonical extension, see \cite{pooladian2021entropic})
\begin{align}
    f(x) &= -\log\EE_{Y\sim \mu}\l[e^{g(Y)-c(x,Y)}\r]\label{eq:f_optimality}\\
    g(y) &= -\log\EE_{X\sim \nu}\l[e^{f(X)-c(X,y)}\r]\label{eq:g_optimality}.
\end{align}
Note that the dual potentials are defined up to an additive constant, that is $f(x)+c, g(x)-c$ is also a pair of optimal dual potentials. The optimality conditions yield $\Phi(f,g;\mu,\nu) = \EE_{X\sim\mu}f(X) + \EE_{Y\sim\nu}g(Y) = S_c(\mu,\nu),$ so we assume the optimal potentials are chosen to have 
\begin{align}
    \EE_{X\sim\mu}f(X)=\EE_{Y\sim\nu}g(Y) = \frac12S_c(\mu,\nu)\geq0.\label{eq:fg_expectation}
\end{align}
With the above definitions, we can proceed to bound the excess risk:
\begin{align}
    S_c(P_{G_n(Z)}, P_Y) -S_c(P_{G^*(Z)}, P_Y) &=\underbrace{S_c(P_{G_n(Z)}, P_Y) -S_c(P_{G_n(Z)}, Q_Y^n)}_A
    +\underbrace{S_c(P_{G_n(Z)}, Q_Y^n) - S_c(P_{G^*(Z)}, P_Y)}_B\label{eq:ot_G}
\end{align}
We start by bounding the first term on the RHS in \eqref{eq:ot_G} using the dual formulation: denoting the optimal dual potentials for the population optimal transport with generator $G_n$
$$f_n, g_n = \arg\max_{f, g\in L^1(P_{G_n(Z)})\times L^1(P_Y)}\Phi(f, g;P_{G_n(Z)},P_Y).$$
Then rewriting \eqref{eq:ot_G}  in the dual formulation gives

\begin{align}
    A&=\Phi(f_n, g_n;P_{G_n(Z)}, P_Y) - \max_{f, g\in L^1(P_{G_n(Z)})\times L^1(Q_Y^n)} \Phi(f, g;P_{G_n(Z)}, )\nonumber\\
    &\leq\Phi(f_n, g_n;P_{G_n(Z)}, P_Y) - \Phi(f_n, g_n;P_{G_n(Z)}, Q_Y^n)\nonumber,
\end{align}
We can next plug in the definition \eqref{eq:dual} for $\Phi$ to get
\begin{align}
    A\leq\EE_{Y\sim P_Y}g_n(Y) - \frac1n\sum_{i=1}^ng_n(Y_i) &- \EE_{Y\sim P_Y}\EE_{X\sim P_{G_n(Z)}}\l[e^{f_n(X)+ g_n(Y)-c(X,Y)}\r]\label{eq:bound_gn1}\\
    & + \frac1n\sum_{i=1}^n\EE_{X\sim P_{G_n(Z)}}\l[e^{f_n(X)+ g_n(Y_i)-c(X,Y_i)}\r]\label{eq:bound_gn2}.
\end{align}
Recall the optimality condition \eqref{eq:g_optimality} for $g_n,$ which asserts that for any $y\in \RR^d:\EE_{X\sim P_{G_n(Z)}}\l[e^{f_n(X) + g_n(Y)-c(X,y)}\r]=1,$
so the last summands in \eqref{eq:bound_gn1} and \eqref{eq:bound_gn2} are equal to -1 and 1 respectively and cancel out, leaving 
\begin{align}
    A&\leq\EE_{Y\sim P_Y}g_n(Y) - \frac1n\sum_{i=1}^ng_n(Y_i)\label{eq:A_bound}.
\end{align}
Bounding $B$ in \eqref{eq:ot_G} is simpler than bounding $A$ because by the optimality of $G_n:$
$$B\leq S_c(P_{G^*(Z)}, Q_Y^n) - S_c(P_{G^*(Z)}, P_Y),$$
and now standard results for the sample complexity of entropic optimal transport like Theorem 2 from \cite{stromme2023minimum} can be used to bound it since $G^*$ does not depend on the sample. However, the known results will require additional assumptions on the cost function/privatization mechanism, which we would like to avoid, so we proceed by bounding $B$ in a fashion similar to $A.$ Denote 
$$\hat f^*, \hat g^* = \arg\max_{f, g\in L^1(P_{G^*(Z)})\times L^1(Q_Y^n)}\Phi(f, g;P_{G_n(Z)},Q_Y^n),$$ the optimality of these potentials and strong duality results in the following bound similar to the one for $A:$
\begin{align}
    B&\leq\Phi(\hat f^*, \hat g^*;P_{G^*(Z)}, Q_Y^n) - \Phi(\hat f^*, \hat g^*;P_{G_n(Z)}, P_Y)\leq \frac1n\sum_{i=1}^n \hat g^*(Y_i)- \EE_{Y\sim P_Y}\hat g^*(Y)\label{eq:B_bound}.
\end{align}
Bounding the expected values of $A$ and $B$ can now be done in the same way as bounding the excess risk in classic learning problems, which can be achieved through Rademacher complexity.
\begin{definition}(Rademacher Complexity)
For a family of functions $\cF$ and a fixed sample $S = \{Y_i\}_{i=1}^n$ the empirical Rademacher complexity of $\cF$ with respect to the sample $S$ is defined as:
$$\hat \Rad_S(\cF) = \EE_\sigma\l[\sup_{f\in\cF}\frac1n\sum_{i=1}^n\sigma_if(Y_i)\r],$$
where the expectation is taken with respect to $\sigma = (\sigma_1,\ldots,\sigma_n)$ with $\sigma_i$ being independent uniform random variables taking values in $\{\pm1\}.$

For any integer $n\geq1$ the Rademacher complexity of $\cF$ is the expectation of the empirical Rademacher complexity over all samples of size $n:$
$$\Rad_n(\cF) = \EE_{S\sim P_Y^{\otimes n}}\EE_\sigma\l[\sup_{f\in\cF}\frac1n\sum_{i=1}^n\sigma_if(Y_i)\r].$$
\end{definition}
Rademacher complexity is one of the key tools to bound suprema of empirical processing, with the following lemma connecting the two, which appears in \cite{mohri2018foundations} in the proof of theorem 3.3, equation 3.13 (we removed the unused assumptions):
\begin{lemma}
    For a set of functions $\cF$ mapping $\cX$ to $\RR$ and a sample $S = \{Y_i\}\sim P_Y^{\otimes n}:$
    \begin{align}
        \EE_S\l[\sup_{f\in\cF}\l(\EE[f(Y)] - \frac1n\sum_{i=1}^n f(Y_i)\r)\r]\leq 2\Rad_n(\cF)
    \end{align}
 \end{lemma}
Fix $S = \{Y_i\}_{i=1}^n,$ applying this lemma to $A$ in \eqref{eq:A_bound} gives
\begin{align}
    \EE[A]&\leq\EE\l[\EE_{Y\sim P_Y}g_n(Y) - \frac1n\sum_{i=1}^ng_n(Y_i)\r]\leq \EE\l[\sup_{g\in\cH_n}\l(\EE_{Y\sim P_Y}g(Y) - \frac1n\sum_{i=1}^ng(Y_i)\r)\r]\leq 2\Rad_n(\cH_n),
\end{align}
 where $\cH_n$ is the set of dual potentials $g$ for all the admissible generators $g\in\cG$, that is $\cH_n = \{g:\RR^d\to\RR\mid \exists f, \exists G\in \cG: \Phi(f, g;P_{G(Z)}, P_Y) = S_c(P_{G(Z)}, P_Y)\}.$ Note that here we again assume that the optimal dual potentials are extended to $\RR^d$ and \eqref{eq:fg_expectation} holds.
 Similarly for $B$ in \eqref{eq:B_bound}:
\begin{align}
    \EE[B]&\leq\EE\l[\frac1n\sum_{i=1}^n\hat g(Y_i)-\EE_{Y\sim P_Y}\hat g(Y)\r]\leq 2\Rad_n(\hat \cH),
\end{align}
 where $\hat\cH$ is the set of dual potentials $g$ for all the admissible generators $g\in\cG$ for the empirical problem, that is $\hat \cH = \{g:\RR^d\to\RR\mid \exists f, \exists G\in \cG: \Phi(f, g;P_{G(Z)}, Q_Y^n) = S_c(P_{G(Z)}, Q_Y^n)\}.$ The excess risk is finally bounded by
\begin{align}
    \EE\bigl[S_c(P_{G_n(Z)}, P_Y) -& S_c(P_{G^*(Z)}, P_Y)\bigr]\leq 2\Rad_n(\cH_n) + 2\Rad_n(\hat\cH)
    \label{eq:rademacher_bound}
\end{align}
The rest of the proof bounds the Rademacher complexities using the properties of the dual potentials. We start by bounding the optimal dual potential from $\cH_n,$ the bound for  $\hat\cH$ is identical.

Let $f, g$ be the maximizes of $\Phi(f, g;P_{G(Z)}, P_Y).$ The optimality conditions
\eqref{eq:f_optimality},
\eqref{eq:g_optimality} and our convention \eqref{eq:fg_expectation} together with Jensen's inequality for $-\log(x)$ yield for any $x, y\in \RR^d:$
\begin{align}
    f(x) &= -\log\EE_{Y\sim P_Y}\l[e^{g(Y)-c(x,Y)}\r]\leq \EE_{Y\sim P_Y}\l[c(x,Y) - g(Y)\r]\leq \EE_{Y\sim P_Y}\l[c(x,Y)\r]\label{eq:f_upper}\\
    g(y) &= -\log\EE_{X\sim P_{G(Z)}}\l[e^{f(X)-c(X,y)}\r]\leq \EE_{X\sim P_{G(Z)}}\l[c(X,y) - f(X)\r]\leq \EE_{X\sim P_{G(Z)}}\l[c(X,y)\r]\label{eq:g_upper}
\end{align}
% Plugging the above bounds back into the optimality conditions leads to the lower bound
% \begin{align*}
% g(y) &= -\log\EE_{X\sim P_{G(Z)}}\l[e^{f(X)-\|X-y\|_1}\r]\geq -\log\EE_{X\sim P_{G(Z)}}\l[e^{\EE_{Y\sim P_Y}\l[\|X-Y\|_1\r]-\|X-y\|_1}\r].\\
% f(x) &\geq -\log\EE_{X\sim P_{G(Z)}}\l[e^{f(X)-\|X-y\|_1}\r]\geq -\log\EE_{X\sim P_{G(Z)}}\l[e^{\EE_{Y\sim P_Y}\l[\|X-Y\|_1\r]-\|X-y\|_1}\r].
% \end{align*}

Note that for any function $h(y)$ it holds that $\hat\Rad_n(\cH_n) = \hat\Rad_n(\cH_n \oplus h(y)),$ namely, adding or substracting a specific function from all the functions in a set does not change the Rademacher complexity. We will choose $h(y) = \sup_{x\in\cX}c(x,y).$

Now let $u(y) = e^{h(y) - g(y)}.$ By the upper bound on $g(y)$ \eqref{eq:g_upper}:
$$u(y)\geq e^{\sup_{x\in\cX}c(x,y) -\EE_{X\sim P_{G(Z)}}\l[c(X,y)\r]}\geq 1,$$ 
so the function $f(x) = -\log x$ is 1-Lipschitz on the range of $u.$ By Talagrand's lemma (Lemma 5.7 in \cite{mohri2018foundations}), composition with a 1-Lipschitz function cannot increase the Rademacher complexity of a function set. Denoting $\cU_n = \{u(y) = e^{\sup_{x\in\cX}c(x,y)-g(y)}\mid g\in\cH_n\}$ we arrive at 
$$\hat\Rad_n(\cH_n) = \hat\Rad_n((-\log)\circ \cU\oplus h(y))\leq \hat \Rad_S(\cU).$$

To further bound the Rademacher complexity of $\cU$ we use the following result for positive definite symmetric kernels that is a direct consequence of Mercer's theorem \cite{mercer1909xvi}.
\begin{theorem} \label{thm:rkhs_bound}(Theorems 6.8, 6.12 in \cite{mohri2018foundations}) Let $\cZ\subset \RR^d$ and $K: \cZ \times \cZ \rightarrow \RR$ be a positive definite symmetric kernel. Then there exists a Hilbert space $\mathbb{H}$ and a mapping $\Phi:\cZ\to\mathbb{H}$ such that:
$$
\forall x, x^{\prime} \in \mathcal{X}, \quad K\left(x, x^{\prime}\right)=\left\langle\Phi(x), \Phi\left(x^{\prime}\right)\right\rangle .
$$
% Furthermore, $\mathbb{H}$ has the following property known as the reproducing property:
% $$
% \forall h \in \mathbb{H}, \forall x \in \mathcal{X}, \quad h(x)=\langle h, K(x, \cdot)\rangle .
% $$
$\mathbb{H}$ is called a reproducing kernel Hilbert space (RKHS) associated to $K$. and let $\Phi: \cZ \to \mathbb{H}$ be a feature mapping associated to $k$. Let $S=\{z_i\}_{i=1}^n\subset \cZ$ be a sample of size $n$, and let $\cH=\left\{x \mapsto\langle\mathbf{w}, \Phi(x)\rangle:\|\mathbf{w}\|_{\mathbb{H}} \leq \Lambda\right\}$ for some $\Lambda \geq 0$. Then
$$
\widehat{\Re}_S(\mathcal{H}) \leq \frac{\Lambda \sqrt{\sum_{i=1}^n K(z_i, z_i)}}{n}
$$
\end{theorem}
First, we check the conditions of the theorem: fix some $u\in\cU,$ then there exists a $G\in\cG$ and a pair $f,g$ of optimal dual potentials maximizing $\Phi(f, g;P_{G(Z), P_Y})$ such that
\begin{align}
    u(y) &=e^{\sup_{x\in\cX}c(x,y) - g(y)} =  e^{\sup_{x\in\cX}c(x,y)}\EE_{X\sim P_{G(Z)}}\l[e^{f(X)-c(X,y)}\r]\nonumber,
\end{align}
To simplify the notation denote $v(y) = e^{\sup_{x\in\cX}c(x,y)}$ and note that 
$K(x, y) = v(x)v(y)e^{-c(x,y)}$ is a positive definite symmetric kernel (as a product of two kernels), so applying theorem \ref{thm:rkhs_bound} to kernel $K$ leads to 
\begin{align*}
    u(y) &= \EE_{X\sim P_{G(Z)}}\l[K(X, y)\frac{e^{f(X)}}{v(X)}\r]=\l\langle\EE_{X\sim P_{G(Z)}}\l[\frac{e^{f(X)}}{v(X)}\Phi(X)\r], \Phi(y)\r\rangle = \langle w, \Phi(y)\rangle.
\end{align*}
So $u(y)$ is indeed a linear function in RKHS and its associated norm is 
\begin{align*}
    \|w\|_H^2 &= \l\langle\EE_{X\sim P_{G(Z)}}\l[\frac{e^{f(X)}}{v(X)}\Phi(X)\r], \EE_{X\sim P_{G(Z)}}\l[\frac{e^{f(X)}}{v(X)}\Phi(X)\r]\r\rangle
    % &=\l\langle\EE_{X\sim P_{G(Z)}}e^{f(X)}/v(X)\Phi(X), \EE_{X'\sim P_{G(Z)}}e^{f(X')}/v(X')\Phi(X')\r\rangle\\
    =\EE_{X,X'\sim P_{G(Z)}^{2}}\frac{e^{f(X)+f(X')}}{v(X)v(X')}K(X,X') \\
    &= \EE_{X,X'\sim P_{G(Z)}^{2}}e^{f(X')+f(X)-c(X,X')}\leq \EE_{X,X'\sim P_{G(Z)}^{2}}e^{\EE_{Y\sim P_Y}\l[c(X',Y) + c(X,Y)-c(X,X')\r]}\leq e^{2\sup_{x\in\cX}\EE_{Y\sim P_Y}\l[c(x,Y)\r]}, \\
    % &\leq e^{\sup_{x,x'\in\cX}\EE_{Y\sim P_Y}\l[c(x',Y) + c(x,Y)-c(x,x')\r]}\leq e^{-\inf_{x,x'\in\cX}c(x,x')}e^{2\sup_{x\in\cX}\EE_{Y\sim P_Y}\l[c(x,Y)\r]},
\end{align*}
where we used the upper bound on $f(x)$ given in \eqref{eq:f_upper}.
Combining this with theorem \ref{thm:rkhs_bound} leads to 
\begin{align}
    \hat \Rad_S(\cH_n)&\leq \hat \Rad_S(\cU)\leq 
    % e^{-\inf_{x,x'\in\cX}c(x,x')/2}
    e^{\sup_{x\in\cX}\EE_{Y\sim P_Y}\l[c(x,Y)\r]}
    \frac{\sqrt{\sum_{i=1}^n{v(Y_i)^2}}}{{n}}
    =
    % e^{-\inf_{x,x'\in\cX}c(x,x')/2}
    e^{\sup_{x\in\cX}\EE_{Y\sim P_Y}\l[c(x,Y)\r]}\frac{\sqrt{\sum_{i=1}^n{e^{2\sup_{x\in\cX}c(x,Y_i)}}}}{n}
    % \leq\frac{e^{4Dd}}{\sqrt{n}}
    % \label{eq:Hn_bound}
    \nonumber
\end{align}
To get the bound for the Rademacher complexity we take the expectation of both sides and apply the Jensen's inequality, which leads to  
\begin{align}
    \Rad_n(\cH_n)&=\EE[\hat \Rad_S(\cH_n)]\leq
    e^{\sup_{x\in\cX}\EE_{Y\sim P_Y}\l[c(x,Y)\r]}\sqrt{\frac{\EE{e^{2\sup_{x\in\cX}c(x,Y)}}}{n}}\leq \frac{\EE{e^{2\sup_{x\in\cX}c(x,Y)}}}{\sqrt{n}}
    % \leq\frac{e^{4Dd}}{\sqrt{n}}
    \label{eq:Hn_bound}
\end{align}
% By Jensen's inequality
% \begin{align}
%     \Rad_n(\cH_n)&\leq  e^{\sup_{x\in\cX}\EE_{Y\sim P_Y}\l[\|x-Y\|_1\r]}\sqrt{\frac{\EE_{Y\sim P_Y}e^{2\sup_{x\in\cX}\|x-Y\|_1}}{n}}\nonumber\\
%     &\leq  \frac{\EE_{Y\sim P_Y}e^{2\sup_{x\in\cX}\|x-Y\|_1}}{\sqrt{n}}.
% \end{align}
The derivation for $\hat \Rad_S(\hat\cH)$ follows the same lines with the only change being the use of $Q_Y^n$ instead of $P_Y$ leading to 
$$\hat \Rad_S(\hat\cH)\leq 
% e^{-\inf_{x,x'\in\cX}c(x,x')/2}
e^{\sup_{x\in\cX}\frac1n\sum_{i=1}^nc(x,Y_i)}\frac{\sqrt{\sum_{i=1}^n{e^{2\sup_{x\in\cX}c(x,Y_i)}}}}{n}.$$ 
Taking the expectation of both sides and applying the Cauchy-Schwartz inequality leads to
\begin{align*}
    \Rad_n(\hat\cH)&=\EE[\hat \Rad_S(\hat\cH)]\leq
    \sqrt{\EE\l[\sup_{x\in\cX}e^{\frac2n\sum_{i=1}^nc(x,Y_i)}\r]}\sqrt{\frac{\EE{e^{2\sup_{x\in\cX}c(x,Y)}}}{n}}\leq \frac{\EE{e^{2\sup_{x\in\cX}c(x,Y)}}}{\sqrt{n}}
    % \leq\frac{e^{4Dd}}{\sqrt{n}}
\end{align*}

which combined with \eqref{eq:Hn_bound} and \eqref{eq:rademacher_bound} gives
% leading to 
% \begin{align}
%     \hat \Rad_S(\hat\cH)&\leq  e^{\sup_{x\in\cX}\frac1n\sum_{i=1^n}c(x,Y_i)}\frac{\sqrt{\sum_{i=1}^ne^{2\sup_{x\in\cX}c(x,Y_i)}}}{{n}}.
% \end{align}
% By Cauchy-Schwartz:
% \begin{align}
%     \Rad_n(\hat\cH)&\leq  \sqrt{\frac{\EE\l[e^{\sup_{x\in\cX}\frac2n\sum_{i=1^n}c(x,Y_i)}\r]\EE e^{2\sup_{x\in\cX}\|x-Y\|_1}}{n}}\nonumber\\
%     &\leq  \sqrt{\frac{\EE\l[e^{\frac2n\sup_{x\in\cX}\|x-Y\|_1}\r]^n\EE e^{2\sup_{x\in\cX}\|x-Y\|_1}}{n}}\nonumber\\
%     &\leq  \frac{\EE_{Y\sim P_Y}e^{2\sup_{x\in\cX}\|x-Y\|_1}}{\sqrt{n}}.
% \end{align}
% Finally we get that 
\begin{align*}
    \EE[S_c(P_{G_n(Z)}, P_Y) - S_c(P_{G^*(Z)}, P_Y)]\leq 
    % e^{-\inf_{x,x'\in\cX}c(x,x')/2}
    4\frac{\EE{e^{2\sup_{x\in\cX}c(x,Y)}}}{\sqrt{n}}
    % \EE\l[\frac{\sqrt{\sum_{i=1}^n{e^{2\sup_{x\in\cX}c(x,Y_i)}}}}{n}\l(e^{\sup_{x\in\cX}\frac1n\sum_{i=1}^nc(x,Y_i)} + e^{\sup_{x\in\cX}\EE[c(x,Y)]}\r)\r]
\end{align*}
% Finally, we note that $W_{1,\sigma}(P_{G_n(Z)}, P_Y) = \sigma S_c(P_{G(Z)/\sigma}, P_{Y/\sigma}),$ so finally we get
% \begin{align*}
%     \EE[W_{1,\sigma}(P_{G_n(Z)}, P_Y) - W_{1,\sigma}(P_{G^*(Z)}, P_Y)]\leq \frac{4\sigma e^{4dD/\sigma}}{\sqrt{n}}
% \end{align*}
% Mercer's theorem in the form of Theorem 6.2 in \cite{mohri2018foundations}.
% \begin{theorem}[]
% For a compact set $\cZ\subset \RR^d$ and $k:\cZ\times \cZ\to \RR$ if $k$ is continous, symmetric and positive definite then $k(z,z') = \sum_{i=1}^\infty \lambda_i\phi_i(z), \phi_i(z'),$ where $\lambda_i>0$ and the convergence is uniform.
\end{proof}
\subsection{Proof of Theorem~\ref{thm:generalization_gauss}}\label{sec:proof_generalization_gauss}
Theorem~\ref{thm:generalization_gauss} follows from the following theorem proved in \cite[Theorem 6]{reshetova2021understanding}.

\begin{thm}{\cite[Theorem 6]{reshetova2021understanding}}

\label{thm:generalization_lip}
	Let $\PP_Z$ and $\PP_Y$ be sub-Gaussian and the set of generators $\cG$ consist of $L$-Lipschitz  functions, i.e. $\|G(Z_1) - G(Z_2)\|\leq L\|Z_1-Z_2\|$ for any $Z_1,Z_2$ in the support of $P_Z$ and let $\cG$ satisfy \eqref{eq:assump_lin}. Then the generalization error for entropic GAN with $p=2$ \eqref{eq:Wasserstein_reg_obj} can be both upper bounded as
	\begin{align}
	&\EE \bigl[W_{2,\lambda}^2(\PP_{G^n(Z)},\PP_Y)- W_{2,\lambda}^2(\PP_{G^*(Z)},\PP_Y)\bigr]		
	\leq \label{eq:conv}
	C_d\lambda n^{-1/2}\bigl(1+(2\tau^2/\lambda)^{\lceil 5d/4 \rceil+3}\bigr)
	\end{align}
	with $\tau^2 = \max\{L^2\sigma^2(Z),\sigma^2(Y)\}$.
\end{thm}

% \textcolor{red}{The proof of Theorem~2 follows from the above theorem by taking/plugging [can you please add a few sentences here.]}

We present the proof of Theorem~\ref{thm:generalization_gauss} below:

\begin{proof}

%We can now use \cite[Theorem 6]{reshetova2021understanding}, stated below:

In our case $\lambda = 2\sigma^2$ and $Y = X+N$ with $N\sim\cN(0,\sigma^2I),$ so $\sigma(Y) \leq \sigma(X) + \sigma(N),$ where $\sigma(N) = \sigma^2.$ Thus, plugging it into the theorem we get
\begin{align}
	\EE \bigl[W_{2,\lambda}^2(\PP_{G^n(Z)},\PP_Y)- W_{2,\lambda}^2&(\PP_{G^*(Z)},\PP_Y)\bigr]\nonumber\\
    &\leq \label{eq:conv}
	C_d\sigma^2n^{-1/2}\l(1+\l(\frac{\max\{L\sigma(Z), \sigma(X) + \sigma\}}{\sigma}\r)^{2\lceil \frac{5d}4 \rceil+6}\r)\nonumber
	\end{align}

 Letting $\tau = \max\{L\sigma(Z)/\sigma(X),1\}$ we get $(\sigma(X) + \sigma)\tau\geq \max\{L\sigma(Z), \sigma(X) + \sigma\},$ which leads to \begin{align}
	\EE \bigl[W_{2,\lambda}^2(\PP_{G^n(Z)},\PP_Y)- W_{2,\lambda}^2&(\PP_{G^*(Z)},\PP_Y)\bigr]\nonumber\\
    &\leq 
	C_d\sigma^2n^{-1/2}\l(1+\l(\tau^2\l(1 + \sigma(X)/\sigma\r)^2\r)^{\lceil \frac{5d}4 \rceil+3}\r).\nonumber
	\end{align}
\end{proof}

% We can now prove Corollary~3.
% \begin{proof}
% Denoting $f_N$ the pdf of $\cN(0, \sigma^2I)$ and plugging it into lemma \ref{lemma:1} with $P_G = P_{G^n(Z)}$ leads to
% \begin{align}
% D_{KL}(P_Y\|P_{G^n(Z)+N})&\leq\frac1{2\sigma^2}\l(W_{2, 2\sigma^2}(P_Y, P_{G^n(Z)})  - W_{2, 2\sigma^2}(P_Y, P_{X})\r).\nonumber
% \end{align}
% Taking the expectation of both sides and applying Theorem~2 proves the claim.
% \end{proof}
\section{Discussion and conclusion}
We have proposed and analyzed a new framework for locally differentially private training of GANs. Our analysis indicates that the addition of mutual information to the objective of the optimal transport GAN can act as a deconvolution operator provided the right choice of the cost function.  The method not only recovers the original distribution in the population setting but also converges at a parametric rate and can be easily combined with non-privatized training methods as a black box in practice since the modifications do not influence the training process. We believe understanding how to train ML models from privatized data and improving the privacy/utility trade-offs is of paramount importance for the future of privacy-preserving machine learning.
\bibliography{main}
\end{document}